\documentclass[wcp]{jmlr}
\usepackage{mathtools}
\usepackage{xcolor}
\usepackage{longtable}% for long tables
\usepackage{booktabs}

\usepackage{float, comment}
\newtheorem{assumption}[theorem]{Assumption}

\usepackage{lineno}
\makeatletter
\let\Ginclude@graphics\@org@Ginclude@graphics 
\makeatother

\title[Faster Convergence of RSGD with Increasing BS]{Faster Convergence of Riemannian Stochastic Gradient Descent with Increasing Batch Size}

\author{\Name{Kanata Oowada} \Email{kanata7527@gmail.com}\\
\Name{Hideaki Iiduka} \Email{iiduka@cs.meiji.ac.jp}\\
\addr Meiji University, Japan}

\begin{document}

\maketitle
\thispagestyle{plain} % for ArXiv

\begin{abstract}
We theoretically analyzed the convergence behavior of Riemannian stochastic gradient descent (RSGD) and found that using an increasing batch size leads to faster convergence than using a constant batch size, not only with a constant learning rate but also with a decaying learning rate, such as cosine annealing decay and polynomial decay. The convergence rate improves from $O(T^{-1}+C)$ with a constant batch size to $O(T^{-1})$ with an increasing batch size, where $T$ denotes the total number of iterations and $C$ is a constant. Using principal component analysis and low-rank matrix completion, we investigated, both theoretically and numerically, how an increasing batch size affects computational time as quantified by stochastic first-order oracle (SFO) complexity. An increasing batch size was found to reduce the SFO complexity of RSGD. Furthermore, an increasing batch size was found to offer the advantages of both small and large constant batch sizes.
\end{abstract}

\begin{keywords}
Batch Size; Stochastic First-order Oracle Complexity; Learning Rate; Riemannian Optimization; Riemannian Stochastic Gradient Descent
\end{keywords}

\section{Introduction}
\begin{table*}[ht]\label{table:exi_rslt}
\small
\begin{tabular}{cccc}
\hline
Study/Theorem &Batch Size & Learning Rate & Convergence Analysis\\
\hline
\citet{Bonnabel2013Stoc} & -- & s.a. LR & $\underset{t\to \infty}{\lim}\mathbb{E}[\mathrm{grad}f(x_t)]= 0 \text{ a.s.}$\\
\hline
\citet{Zhang2016Firs} & -- & $\eta_t=O((L_r + \sqrt{t})^{-1})$ & $\mathbb{E}[f(x_T) - f^{\star}]=O(\sqrt{T}(C + T)^{-1})$\\
\hline
\citet{Tripuraneni2018Aversto} & Constant & $\eta_t = O(t^{-\alpha})$ & $\mathbb{E}\|\Delta_T\|= O(T^{-\frac{1}{2}})$\\
\hline
\citet{Hosseini2020Analter} & Constant & $\eta = O(T^{-\frac{1}{2}})$ & $\|G_T\|^2=O(T^{-\frac{1}{2}})$\\
\hline
\citet{Durmus2021OnRie} & Constant & Constant & $\mathbb{E}[f(x_T)-f^{\star}]=O(C_1^T+C_2)$\\
\hline
\citet{Sakai2024ConvHada} & Constant & $\eta_t =O((t+1)^{-\frac{1}{2}})$ & $\|G_T\|^2=O(T^{-\frac{1}{2}}\log T)$\\
\hline
\citet{Sakai2025Agen} & Increasing & Constant & $\|G_T\|^2=O(T^{-1})$\\
\hline
Theorem \ref{thm:c_d} & Constant & Constant and Decay & $\|G_T\|^2=O(T^{-1}+C)$ \\
\hline
Theorem \ref{thm:i_d} & Increasing & Constant and Decay & $\|G_T\|^2=O(T^{-1})$\\
\hline
Theorem \ref{thm:i_w} & Increasing & Warm-up & $\|G_T\|^2=O((T-T_w)^{-1})$\\
\hline
Theorem \ref{thm:c_w} & Constant & Warm-up & $\|G_T\|^2=O((T-T_w)^{-1} +C)$\\
\hline
\end{tabular}
\centering
\caption{Comparison of RSGD convergence analyses, where $\alpha\in(\frac{1}{2},1), C_1\in(0,1), C_2$ and $C$ are constants, and $(x_t)$ is a sequence generated by RSGD. The usual stochastic approximation LR (s.a. LR) is defined as $\sum_{t=0}^{\infty}\eta_t=\infty, \sum_{t=0}^{\infty}\eta_t^2<\infty$. $\mathbb{E}[\|\Delta_T\|]\coloneq\mathbb{E}[\|R_{x^{\star}}^{-1}(x_T)\|]$. $T$ is the total number of iterations, and $T_w$ is the number of warm-up iterations (see Section \ref{i_w}). $\|G_T\|^2\coloneq\min_{t\in\{0,\cdots, T-1\}}\mathbb{E}[\|\mathrm{grad}f(x_t)\|^2_{x_t}]$. ‘Decay’ includes LR schedules of the form $\eta_t=O(t^{-\frac{1}{2}})$. More detailed explanations of the convergence criteria are provided in Appendix \ref{apdix:explain_conv_criteria}.}
\end{table*}
Stochastic gradient descent (SGD) \citep{Robbins1951Astoc} is a basic algorithm, widely used in machine learning \citep{Liu2021Swintr, He2016Deepres, Krizhevsky2012ImageNe}. Riemannian stochastic gradient descent (RSGD) was introduced in \citet{Bonnabel2013Stoc}. Riemannian optimization \citep{Absil2008opti, Boumal2023Anintr}, which addresses RSGD and its variants, has attracted attention due to its potential application in various tasks, including many machine learning tasks. For example, it has been used in principal components analysis (PCA) \citep{Breloy2021Majo, Liu2020Simp, Roy2018Geom}, low-rank matrix completion \citep{Vandereycken2013Lowr, Boumal2015Lowr, Kasai2016Lowra}, convolutional neural networks \citep{Wang2020Orthogonal, Huang2017deep}, and graph neural networks \citep{Zhu2020Graph, Chami2019Hype, Liu2019Hype}. It has also been used in applications of optimal transportation theory \citep{Lin2020Proje, Weber2023Riem}. It is well established that the performance of Euclidean SGD strongly depends on both the batch size (BS) and learning rate (LR) settings \citep{Goyal2017Accur, Smith2018dontdecay, Zhang2019which, Lin2020Dontu}. In the context of Riemannian SGD, \citet{Ji2023Adap, Bonnabel2013Stoc, Kasai2019Riem, Kasai2018Riemrec} addressed the use of a constant BS, and \citet{Sakai2025Agen, Han2022Imp} addressed the use of an adaptive BS, although the latter did not consider RSGD specifically. In parallel, various decaying LR strategies, including cosine annealing \citep{Loshchilov2017Wup} and polynomial decay \citep{Chieh2018Polylr}, have been devised, and their effectiveness has been demonstrated for the Euclidean case. Motivated by these previous studies, we developed a novel convergence analysis of RSGD, proved that the convergence rate is improved with an increasing BS, and showed that the stochastic first-order oracle (SFO) complexity is reduced with an increasing BS. Our numerical results support our theoretical results.

\subsection{Previous Results and Our Contributions}
\textbf{Constraint:} \citet{Bonnabel2013Stoc, Zhang2016Firs, Durmus2021OnRie, Sakai2024ConvHada} considered Hadamard manifolds or their submanifolds. \cite{Tripuraneni2018Aversto} treated Riemannian manifolds that satisfy the condition required to ensure a unique geodesic connecting any two points, including Hadamard manifolds. \cite{Sakai2025Agen} addressed embedded submanifolds of Euclidean space. In contrast, we considered general Riemannian manifolds that encompass all of the above conditions.

\noindent
\textbf{Learning rate:} \cite{Bonnabel2013Stoc} considered s.a. LRs, 
%LRs satisfying $\sum_{t=0}^{\infty}\eta_t=\infty, \sum_{t=0}^{\infty}\eta^2_t<\infty$
and \cite{Zhang2016Firs, Tripuraneni2018Aversto, Sakai2024ConvHada} considered diminishing LRs of the form $\eta_t = \frac{1}{t^a}$. \cite{Hosseini2020Analter, Durmus2021OnRie} considered constant LRs. \cite{Sakai2025Agen} used both constant and diminishing LRs. We considered constant, diminishing, cosine annealing, and polynomial decay LRs, as well as their warm-up versions.

\noindent
\textbf{Batch size:} Because \cite{Bonnabel2013Stoc} considered expected risk minimization using expectation instead of sample mean—as is customary in empirical risk minimization—the concept of BS is not generally applicable in the expected-risk-minimization setting. \cite{Zhang2016Firs, Tripuraneni2018Aversto, Hosseini2020Analter, Durmus2021OnRie, Sakai2024ConvHada} considered a constant BS. \cite{Sakai2025Agen} and this study considered an increasing BS. Furthermore, we numerically and theoretically compared using a constant BS with using an increasing BS on the basis of SFO complexity. Our contributions are as follows: (i) we theoretically showed that using an increasing BS yields a more favorable SFO complexity of $O(\epsilon^{-2})$, whereas using a constant BS equal to the critical BS yields a rate of $O(\epsilon^{-4})$; (ii) we numerically observed that an increasing BS yields both a better optimal solution—in terms of attaining a smaller gradient norm—and a shorter computational time than either a small or a large constant BS (see Section \ref{sec:SFO} for more details).

\noindent
\textbf{Objective function:} \cite{Bonnabel2013Stoc} considered three times continuously differentiable functions, with both the gradient and Hessian uniformly bounded. \cite{Tripuraneni2018Aversto} treated functions that are twice continuously differential, subject to additional conditions on the Hessian. Other studies considered only once continuously differentiable functions. \cite{Zhang2016Firs} addressed convex and smooth functions, while \cite{Durmus2021OnRie} considered strongly convex and smooth functions. \cite{Hosseini2020Analter, Sakai2024ConvHada, Sakai2025Agen} treated more general classes of functions, specifically nonconvex functions with bounded gradients. We treated nonconvex functions without bounded gradients.

\noindent
\textbf{Convergence rate:} A comparison of the previous and current RSGD convergence analyses is presented in Table \ref{table:exi_rslt}; here we outline the key results. \citet{Tripuraneni2018Aversto} obtained a convergence rate of $O(T^{-\frac{1}{2}})$ (with a criterion other than the gradient norm) using an LR of $\eta=O(T^{-\frac{1}{2}})$; however, this rate decays to zero as $T\to\infty$, making it unsuitable for practical use. \cite{Durmus2021OnRie} attained exponential convergence under strong convexity and smoothness assumptions. \cite{Sakai2025Agen} attained a rate of $O(T^{-1})$ for several adaptive methods, including RSGD. Our work differs from \cite{Sakai2025Agen} in four key aspects : (i) we provide both theoretical and numerical evidence—based on SFO complexity—that using an increasing BS leads to better performance than using a constant BS; (ii) in addition to a constant LR and a diminishing LR, we addressed cosine annealing, polynomial decay, and warm-up LRs; (iii) \cite{Sakai2024ConvHada} focused on embedded submanifolds of Euclidean space unlike us; (iv) we do not assume bounded gradient norms, thereby expanding the range of applicable scenarios.

\noindent
% By Mr. Fisher: If your proposed optimization problem is an important part of this work, it should be mentioned as well.
\textbf{Contributions:} 
\begin{itemize}
\item We developed a convergence analysis of RSGD that incorporates an increasing BS, a cosine annealing LR, and a polynomial decay LR, under assumptions more general than those used in prior work. Our analysis shows that using an increasing BS improves the convergence rate of RSGD from $O(T^{-1}+C)$ to $O(T^{-1})$.
\item We numerically and theoretically demonstrated—using PCA and low-rank matrix completion (LRMC) tasks—the advantage of an increasing BS in reducing computational time, as quantified by SFO complexity. Specifically, we observed that an increasing BS yields a better optimal solution (compared with a small constant BS) in terms of attaining a smaller gradient norm and a shorter computation time (compared with a large constant BS). These findings are consistent with our theoretical analysis, which shows that an increasing BS reduces the SFO complexity for achieving $\|G_T\|^2\leq\epsilon^2$ from $O(\epsilon^{-4})$—with a constant BS set equal to the critical BS—to $O(\epsilon^{-2})$ in the experimental setting.
\end{itemize}

\section{Preliminaries}
\label{sec:preli}
Let $\mathcal{M}$ be a Riemannian manifold and $T_x\mathcal{M}$ denote the tangent space at $x\in\mathcal{M}$. On $T_x\mathcal{M}$, the inner product of $\mathcal{M}$ is denoted by $\langle\cdot, \cdot\rangle_x$ and induces the norm $\|\cdot\|_x$. For a smooth map $f\colon\mathcal{M}\rightarrow\mathbb{R}$, we can define gradient $\mathrm{grad} f$ as a unique vector field that satisfies $\forall v\in T_x\mathcal{M}\colon Df(x)[v] = \langle \mathrm{grad}f(x), v\rangle_x$. A Riemannian manifold is generally not equipped with vector addition, whereas in Euclidean space iterative methods are updated by addition. Iterative updates are instead performed using alternative operations—such as retraction.
%(better) In the absence of vector addition on a general Riemannian manifold, iterative updates are instead performed using alternative operations—such as retraction.
\begin{definition}[Retraction] 
Let $0_x$ denote the zero element of $T_x\mathcal{M}$. A map $R\colon T\mathcal{M} \ni (x,v) \mapsto R_x(v)\in\mathcal{M}$ is called a retraction on $\mathcal{M}$ if it satisfies the two following conditions for all $x\in \mathcal{M}$. (I) $R_x(0_x) = x$; (II) with the canonical identification $T_{0_x}T_x\mathcal{M}\simeq T_x\mathcal{M}$, $DR_x(0_x) = \mathrm{id}_{T_x\mathcal{M}}$, where $\mathrm{id}_{T_x\mathcal{M}}\colon T_x\mathcal{M}\to T_x\mathcal{M}$ denotes the identity map.
\end{definition}
Iterative methods on Riemannian manifold are defined by $x_{t+1}=R_{x_t}(\eta_t d_t)$ generally. Exponential maps are often used as retractions. The following assumption plays a central role in Lemma \ref{lem:undl_anal}.
\begin{assumption}
[Retraction Smoothness]\label{asm:rtr_smooth} Let $f\colon\mathcal{M}\rightarrow\mathbb{R}$ be a smooth map. Then there exists $L_r>0$, such that $^{\forall}x\in\mathcal{M}, ^{\forall}v\in T_x\mathcal{M}$, $f(R_x(v)) \leq f(x) + \langle \mathrm{grad} f(x), v\rangle_x + \frac{L_r}{2} \|v\|_{x}^2$.
\end{assumption}
In the Euclidean space setting, $L$-smoothness implies a property similar to retraction smoothness. The property corresponding to $L$-smoothness in Euclidean space is defined for $f\colon\mathcal{M}\rightarrow\mathbb{R}$ as $^{\exists}L > 0, ^{\forall}x, y \in \mathcal{M}\colon \|\mathrm{grad}f(x) - \Gamma^x_y \mathrm{grad}f(y)\|_{x}\leq L_r d(x,y)$, where $\Gamma$ is the parallel transport from $T_y\mathcal{M}$ to $T_x\mathcal{M}$, and $d (\cdot, \cdot)$ is the Riemannian distance. This condition is sufficient for Assumption \ref{asm:rtr_smooth} with $R\coloneq \mathrm{Exp}$ (see \citet[Corollary 10.54]{Boumal2023Anintr}). This case is frequently used (e.g., \citet{Zhang2016Firs, Criscitiello2023Curva, Kim2022neste, Liu2017Acceler}). Other sufficient conditions for Assumption \ref{asm:rtr_smooth} were identified by \citet[Lemma 3.5]{Kasai2018Riemrec} and \citet[Proposition 3.2]{Sakai2025Agen}. Now, we consider the empirical risk minimization problem such that 
\begin{align*}
\text{minimize}\ f(x) \coloneq \frac{1}{N}\sum_{j=1}^N f_j(x)\quad \text{subject to}\ x\in\mathcal{M},
\end{align*}
where each $f_j\colon\mathcal{M}\rightarrow\mathbb{R}$ is smooth and lower bounded. Therefore, $f$ is also smooth and lower bounded. This assumption is often made for both Euclidean space and Riemannian space. The lower boundedness of $f$ is essential for analyses using optimization theory because unbounded $f$ may not have optimizers. We denote an optimal value of $f$ by $f^{\star}$ and let $N$ denote the size of the dataset. In many machine learning tasks, either the dimension of model parameters $x$ or the size of dataset $N$ is large. Hence, we use the following minibatch gradient to efficiently approximate the gradient of $f$: $\mathrm{grad} f_B(x) \coloneq \frac{1}{b}\sum_{j=1}^b \mathrm{grad} f_{\xi_j}(x)$, where $B$ represents a minibatch with size $b$, and $(\xi_j)_{j=1}^b$ is a sequence of $\{1,\cdots,N\}$-valued i.i.d. random variables distributed as $\Pi$. That is, the gradient of $f_B (x)=\frac{1}{b}\sum_{j=1}^b f_{\xi_j}(x)$ is used instead of $\mathrm{grad}f(x)$. The following assumption is reasonable given this situation.
\begin{assumption}
[Bounded Variance Estimator]\label{asm:sto_gra} The stochastic gradient given by a distribution $\Pi$ is an unbiased estimator of the full gradient and has bounded variance: (I) $\forall x \in \mathcal{M}\colon\mathbb{E}_{\xi \sim \Pi}\left[\mathrm{grad} f_{\xi}(x)\right] = \mathrm{grad} f(x)$; (II) $\exists \sigma>0, \forall x \in \mathcal{M}\colon \mathbb{V}_{\xi \sim \Pi}(\mathrm{grad} f_{\xi}(x)) \leq \sigma^2$.
\end{assumption}
For example, if $f$ is lower bounded, satisfies Assumption \ref{asm:rtr_smooth}, and $\Pi$ is taken to be the uniform distribution over $\{1,\ldots,N\}$, then Assumption \ref{asm:sto_gra} holds. RSGD is defined by $x_{t+1} = R_{x_t}(-\eta_t \mathrm{grad} f_{B_t}(x_t))$, and is a generalization of Euclidean SGD. Because $(\xi_{i,t})_{i,t}$ are i.i.d. samples, $\boldsymbol{\xi}_{T}\coloneq(\xi_{1,T},\cdots,\xi_{b_T,T})^{\top}$ is independent of $(x_t)_{t=0}^{T}$. Note that the BS $(b_t)_t$ may vary at each iteration. The detailed definitions of $\mathbb{E}_{\xi \sim \Pi}, \mathbb{V}_{\xi \sim \Pi},$ and $\mathbb{E}$ can be found in Appendix \ref{apdix:proofs}.

\section{Convergence Analysis}
\label{secti:conv_anal}
We begin by presenting a lemma that will be used in the convergence analysis of RSGD. Lemma \ref{lem:undl_anal} plays a central role is the following convergence analyses; its proof is provided in Appendix \ref{prf:undl_anal}.

\begin{lemma}
[Underlying Analysis]\label{lem:undl_anal} Let $(x_t)_t$ be a sequence generated by RSGD and let $\eta_{\max}>0$. Consider a positive-valued sequence $(\eta_t)_t$ such that $\eta_t \in [0,\eta_{\max}] \subset [0, \frac{2}{L_r})$. Then, under Assumptions \ref{asm:rtr_smooth} and \ref{asm:sto_gra}, we obtain
\[
\fcolorbox{black}{gray!20}{$\displaystyle{
 \min_{t\in\{0, \cdots, T-1\}}\mathbb{E}[\|\mathrm{grad}f(x_t)\|_{x_t}^2]
 \leq \frac{2(f(x_0) - f^{\star})}{2 - L_r \eta_{\max}} \frac{1}{\sum_{t=0}^{T-1} \eta_t} + \frac{L_r \sigma^2}{2 - L_r \eta_{\max}} \frac{\sum_{t=0}^{T-1} \eta_t^2 b_t^{-1}}{\sum_{t=0}^{T-1} \eta_t}.
 }$}
\]
\end{lemma}

\subsection{Case (i): Constant BS; Constant or Decaying LR}
\label{c_d}
In this case, we consider a BS $(b_t)_t$ and an LR $(\eta_t)_t$ such that $b_t = b$ and $\eta_{t+1}\leq\eta_t$. In particular, we present the following examples with constant or decaying LRs.
\begin{comment}
\begin{minipage}{0.48\linewidth}
\begin{equation}\label{eq:const_lr}
\text{Constant LR:} \quad \eta_t = \eta_{\max},
\end{equation}
\end{minipage}\hfill
\begin{minipage}{0.48\linewidth}
\begin{equation}\label{eq:dim_lr}
\text{Diminishing LR:} \quad \eta_t = \frac{\eta_{\max}}{\sqrt{t+1}},
\end{equation}
\begin{equation}\label{eq:cosan_lr}
\text{Cosine Annealing LR:} \ \eta_t \coloneq \eta_{\min} + \frac{\eta_{\max} - \eta_{\min}}{2} \left(1 + \cos \frac{t}{T}\pi \right),
\end{equation}
\end{minipage}
\begin{minipage}
\begin{equation}\label{eq:poly_dec_lr}
\text{Polynomial Decay LR:} \ \eta_t \coloneq \eta_{\min} + (\eta_{\max} - \eta_{\min})\left(1 - \frac{t}{T}\right)^p,
\end{equation}
\end{minipage}

% or

\noindent
\begin{tabularx}{\textwidth}{XX}
\centering
\begin{equation}\label{eq:const_lr}
\text{Constant LR:} \quad \eta_t = \eta_{\max},
\end{equation}
&
\begin{equation}\label{eq:dim_lr}
\text{Diminishing LR:} \quad \eta_t = \frac{\eta_{\max}}{\sqrt{t+1}},
\end{equation}
\end{tabularx}
\begin{align}
&\text{Cosine Annealing LR:} \ \eta_t \coloneq \eta_{\min} + \frac{\eta_{\max} - \eta_{\min}}{2} \left(1 + \cos \frac{t}{T}\pi \right),\label{eq:cosan_lr}\\
&\text{Polynomial Decay LR:}\ \eta_t \coloneq \eta_{\min} + (\eta_{\max} - \eta_{\min})\left(1 - \frac{t}{T}\right)^p,\label{eq:poly_dec_lr}
\end{align}
\end{comment}
\small
\begin{align}
&\text{Constant LR:} \quad \eta_t = \eta_{\max}, \label{eq:const_lr}\\
&\text{Diminishing LR:} \quad \eta_t = \frac{\eta_{\max}}{\sqrt{t+1}}, \label{eq:dim_lr}\\
&\text{Cosine Annealing LR:}
\quad \eta_t \coloneq \eta_{\min} + \frac{\eta_{\max} - \eta_{\min}}{2} \left(1 + \cos \frac{t}{T}\pi \right), \label{eq:cosan_lr}\\ 
&\text{Polynomial Decay LR:}
\quad \eta_t \coloneq \eta_{\min} + (\eta_{\max} - \eta_{\min})\left(1 - \frac{t}{T}\right)^p,\label{eq:poly_dec_lr}
\end{align}
\normalsize
where $\eta_{\max}$ and $\eta_{\min}$ are positive values satisfying $0 \leq \eta_{\min} \leq \eta_{\max}<\frac{2}{Lr}$. Note that $\eta_{\max}$ (resp. $\eta_{\min}$) becomes the maximum (resp. minimum) value of $\eta_t$; namely, $\eta_t \in [\eta_{\min},\eta_{\max}] \subset [0,\frac{2}{L_r})$. 
\begin{theorem}
\label{thm:c_d}
We consider LRs \eqref{eq:const_lr}, \eqref{eq:dim_lr}, \eqref{eq:cosan_lr}, and \eqref{eq:poly_dec_lr} and a constant BS $b_t = b>0$ under the assumptions of Lemma \ref{lem:undl_anal}. Then, we obtain
\begin{align*}
&\text{Diminishing LR \eqref{eq:dim_lr}}:\quad \min_{t\in\{0, \cdots, T-1\}}\mathbb{E}[\|\mathrm{grad}f(x_t)\|_{x_t}^2]\leq \frac{Q_1 + Q_2\sigma^2 b^{-1}\log T}{\sqrt{T}}=O\left(\frac{\log T}{\sqrt{T}}\right),\\
&\text{Otherwise \eqref{eq:const_lr}, \eqref{eq:cosan_lr}, \eqref{eq:poly_dec_lr}}:\quad \min_{t\in\{0, \cdots, T-1\}}\mathbb{E}[\|\mathrm{grad}f(x_t)\|_{x_t}^2]\leq \frac{\tilde{Q}_1}{T} + \frac{\tilde{Q}_2\sigma^2}{b}=O\left(\frac{1}{T}+\frac{1}{b}\right),
\end{align*}
where $Q_1, Q_2, \tilde{Q}_1$, and $\tilde{Q}_2$ are constants that do not depend on $T$. 
\end{theorem}

\subsection{Case (ii): Increasing BS; Constant or Decaying LR}
\label{i_d}
In this case, we consider a BS $(b_t)_t$ and an LR $(\eta_t)_t$ such that $b_t \leq b_{t+1}$ and $\eta_{t+1}\leq\eta_t$. We use the same examples as in \textbf{Case (i)} [$\eqref{eq:const_lr},\cdots, \eqref{eq:poly_dec_lr}$], again with constant or decaying LRs, and a BS that increases every $K\in\mathbb{N}$ steps, where $\mathbb{N}$ is the set of positive integers. We let $T$ denote the total number of iterations and define $M\coloneq \lfloor\frac{T}{K}\rfloor$ to represent the number of times the BS is increased. The BS, which takes the form of $\gamma^m b_0$ or $(am + b_0)^p$ every $K$ steps, is an example of an increasing BS. We can formalize the resulting BSs: for every $m\in\{0, \ldots, M-1\}, t \in S_m \coloneq[mK, (m+1)K)\cap\mathbb{N}_0$, 
\begin{align}
&\text{Exponential Growth BS: }b_t \coloneq b_0 \gamma^{m\left\lceil\frac{t}{(m+1)K}\right\rceil},\label{eq:exp_g_bs}\\
&\text{Polynomial Growth BS: }b_t \coloneq \left(am\left\lceil\frac{t}{(m+1)K}\right\rceil + b_0\right)^c,\label{eq:pol_g_bs}
\end{align}
where $\gamma, c > 1, a>0$, and $\mathbb{N}_0\coloneq \mathbb{N}\cup\{0\}$.
\begin{theorem}
\label{thm:i_d}
We consider BSs \eqref{eq:exp_g_bs} and \eqref{eq:pol_g_bs} together with LRs \eqref{eq:const_lr}, \eqref{eq:dim_lr}, \eqref{eq:cosan_lr} and \eqref{eq:poly_dec_lr} under the assumptions stated in Lemma \ref{lem:undl_anal}. Then, the following results hold for both constant and increasing BSs.
\begin{align*}
&\text{Diminishing LR \eqref{eq:dim_lr}}:\quad\min_{t\in\{0, \cdots, T-1\}}\mathbb{E}[\|\mathrm{grad}f(x_t)\|_{x_t}^2]\leq \frac{Q_1 + Q_2\sigma^2 b_0^{-1}}{\sqrt{T}}=O\left(\frac{1}{\sqrt{T}}\right),\\
&\text{Otherwise \eqref{eq:const_lr}, \eqref{eq:cosan_lr}, \eqref{eq:poly_dec_lr}}:\quad \min_{t\in\{0, \cdots, T-1\}}\mathbb{E}[\|\mathrm{grad}f(x_t)\|_{x_t}^2]\leq \frac{\tilde{Q}_1 + \tilde{Q}_2\sigma^2b_0^{-1}}{T}=O\left(\frac{1}{T}\right),
\end{align*}
where $Q_1, Q_2, \tilde{Q}_1$, and $\tilde{Q}_2$ are constants that do not depend on $T$. 
\end{theorem}

\subsection{Case (iii): Increasing BS; Warm-up Decaying LR}
\label{i_w}
In this case, we consider a BS $(b_t)_t$ and an LR $(\eta_t)_t$ such that $b_t \leq b_{t+1}$ and $\eta_t \leq \eta_{t+1}$ for $(t\leq T_w - 1)$ and $\eta_{t+1} \leq \eta_t$ for $(t \geq T_w)$. As examples of an increasing warm-up LR, we consider an exponential growth LR and a polynomial growth LR, both increasing every $K'$ steps. We set $K$, as defined in \textbf{Case (ii)}, to be $lK'$, where $l\in\mathbb{N}$ (thus $K>K'$). Namely, we consider a setting in which the BS is increased every $l$ times the LR is increased. To formulate examples of an increasing LR, we define $M'\coloneq\lfloor\frac{T}{K'}\rfloor$ and formalize the LR: for every $m\in\{0,\ldots,M'-1\}, t\in S'_m\coloneq[mK',(m+1)K')\cap\mathbb{Z}$, \ 
\small
\begin{minipage}{0.48\linewidth}
\begin{equation}\label{eq:exp_g_lr}
\text{Exponential Growth LR:}\ \eta_t \coloneq \eta_0 \delta^{m\left\lceil\frac{t}{(m+1)K'}\right\rceil}, 
\end{equation}
\end{minipage}\hfill
\begin{minipage}{0.48\linewidth}
\begin{equation}\label{eq:pol_g_lr}
\text{Polynomial Growth LR:}\ \eta_t \coloneq \left(sm\left\lceil\frac{t}{(m+1)K'}\right\rceil + \eta_0\right)^{q},
\end{equation}
\end{minipage}
\normalsize
where $s>0$ and $q>1$. Furthermore, we choose $\gamma, \delta>1$, and $l\in\mathbb{N}$ such that $\delta^{2l}<\gamma$ holds. Additionally, we set $l_w\in\mathbb{N}$ such that $T \geq T_w\coloneq l_w K' \geq lK'$. The examples of an increasing BS used in this case are an exponential growth BS \eqref{eq:exp_g_bs} and a polynomial growth BS \eqref{eq:pol_g_bs}. As examples of a warm-up LR, we use LRs that are increased using the exponential growth LR \eqref{eq:exp_g_lr} and the polynomial growth LR \eqref{eq:pol_g_lr} corresponding respectively to the exponential and polynomial growth BSs for the first $T_w$ steps and then decreased using the constant LR \eqref{eq:const_lr}, the diminishing LR \eqref{eq:dim_lr}, the cosine annealing LR \eqref{eq:cosan_lr}, or the polynomial decay LR \eqref{eq:poly_dec_lr} for the remaining $T-T_w$ steps. Note that $\eta_{\max} \coloneq \eta_{T_w -1}$. A more detailed version of Theorem \ref{thm:i_w} is provided in Appendix \ref{prf:i_w}.
\begin{theorem}
\label{thm:i_w}
We consider BSs \eqref{eq:exp_g_bs} and \eqref{eq:pol_g_bs} together with warm-up LRs \eqref{eq:exp_g_lr} and \eqref{eq:pol_g_lr} with decay parts given by \eqref{eq:const_lr}, \eqref{eq:dim_lr}, \eqref{eq:cosan_lr}, or \eqref{eq:poly_dec_lr} under the assumptions stated in Lemma \ref{lem:undl_anal}. Then, the following results hold for both constant and increasing BSs: (I) Decay part [Diminishing \eqref{eq:dim_lr}]: $\|G_T\|^2=O((\sqrt{T+1}-\sqrt{T_w +1})^{-1})$; (II) Decay part [Otherwise \eqref{eq:const_lr}, \eqref{eq:cosan_lr}, \eqref{eq:poly_dec_lr}]: $\|G_T\|^2=O((T-T_w)^{-1})$.
\end{theorem}

\subsection{Case (iv): Constant BS; Warm-up Decaying LR}
\label{c_w}
We consider a BS $(b_t)_t$ and an LR $(\eta_t)_t$ such that $b_t = b$ and $\eta_t \leq \eta_{t+1}$ for $(t\leq T_w - 1)$ and $\eta_{t+1} \leq \eta_t$ for $(t \geq T_w)$. As examples of a warm-up LR, we use the exponential growth LR \eqref{eq:exp_g_lr} and the polynomial growth LR \eqref{eq:pol_g_lr} for the first $T_w$ steps and then the constant LR \eqref{eq:const_lr}, the diminishing LR \eqref{eq:dim_lr}, the cosine annealing LR \eqref{eq:cosan_lr}, or the polynomial decay LR \eqref{eq:poly_dec_lr} for the remaining $T-T_w$ steps. The other conditions are the same as those in \textbf{Case (iii)}. A more detailed version of Theorem \ref{thm:c_w} is provided in Appendix \ref{prf:c_w}.
\begin{theorem}\label{thm:c_w}
We consider a constant BS $b_t=b>0$ and warm-up LRs \eqref{eq:exp_g_lr} and \eqref{eq:pol_g_lr} with decay parts given by \eqref{eq:const_lr}, \eqref{eq:dim_lr}, \eqref{eq:cosan_lr}, or \eqref{eq:poly_dec_lr} under the assumptions stated in Lemma \ref{lem:undl_anal}. Then, we obtain the following results: (I) Decay part [Diminishing \eqref{eq:dim_lr}]: $\|G_T\|^2=O(\log\frac{T}{T_w}(\sqrt{T+1}-\sqrt{T_w +1})^{-1})$, (II) Decay part [Otherwise \eqref{eq:const_lr}, \eqref{eq:cosan_lr}, \eqref{eq:poly_dec_lr}]: $\|G_T\|^2=O((T-T_w)^{-1}+b^{-1})$.
\end{theorem}

\section{Numerical Experiment}\label{sec:numl_expe}
We experimentally evaluated the performance of RSGD for the two types of BSs and various types of LRs introduced in Section \ref{secti:conv_anal}. The experiments were run on an iMac (Intel Core i5, 2017) running the macOS Ventura operating system (ver. 13.7.1). The algorithms were written in Python (3.12.7) using the NumPy (1.26.0) and Matplotlib (3.9.1) packages. The Python code is available at \url{https://github.com/iiduka-researches/RSGD_acml2025.git}. We set $p=2.0$ in \eqref{eq:poly_dec_lr} and $\eta_{\min}:=0$. In \textbf{Cases (i) and (ii)}, we used an initial LR $\eta_{\max}$ selected from $\{0.5, 0.1, 0.05, 0.01, 0.005\}$. In \textbf{Case (ii)}, we set $K=1000, \gamma=3.0$, and $a=c=2.0$. All the plots of the objective function values presented in this section are provided in Appendices \ref{subapdix:sfo_obj} and \ref{appdix:obj}. Those for \textbf{Cases (iii) and (iv)} are provided in Appendix \ref{sec:addi_exper}.

\subsection{Principal Component Analysis}\label{nume:pca}
\begin{figure}%[htbp]
\centering
%\small
\includegraphics[width=0.24\linewidth]{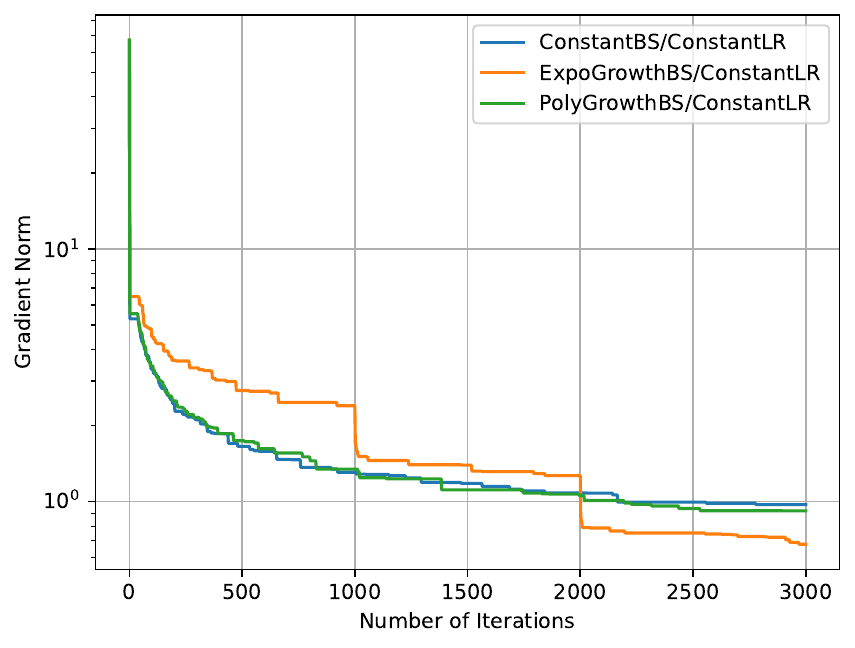}
\includegraphics[width=0.24\linewidth]{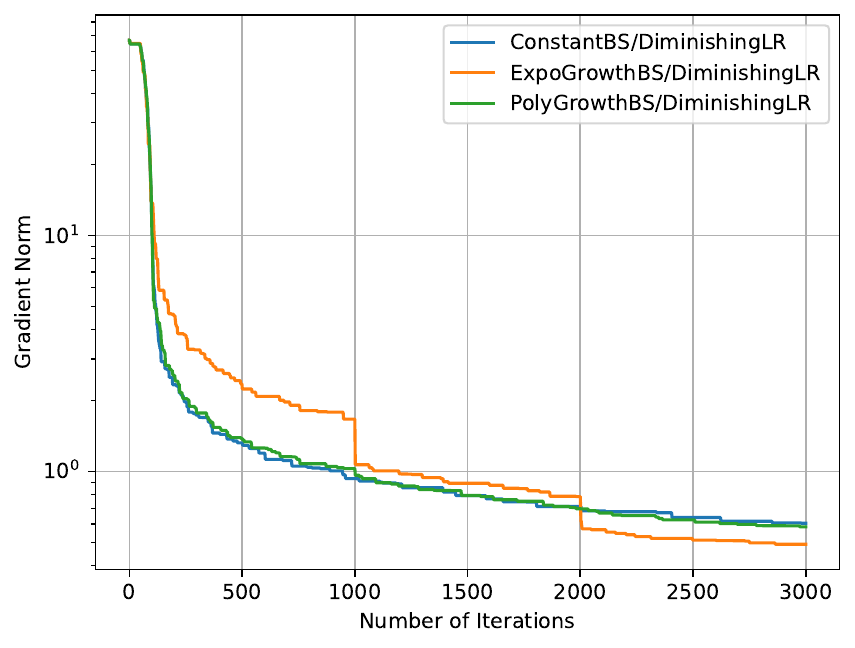}
\includegraphics[width=0.24\linewidth]{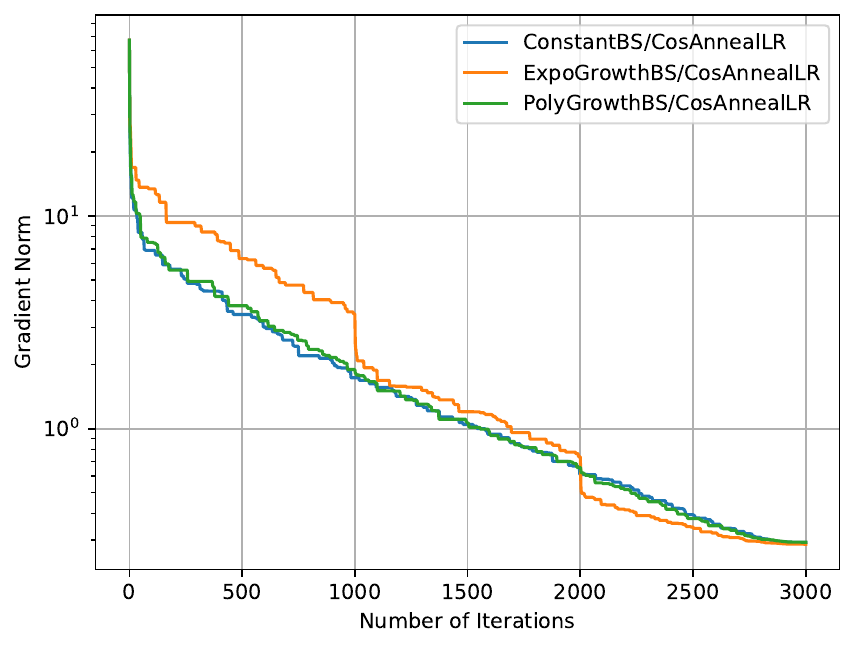}
\includegraphics[width=0.24\linewidth]{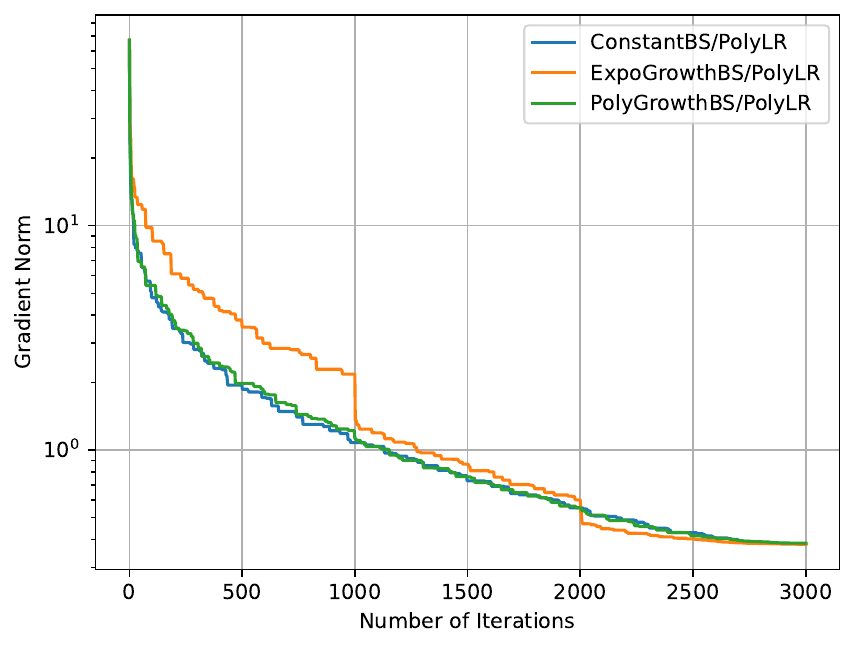}
\caption{Norm of the gradient of the objective function versus the number of iterations for LRs \eqref{eq:const_lr}, \eqref{eq:dim_lr}, \eqref{eq:cosan_lr}, and \eqref{eq:poly_dec_lr} on COIL100 dataset (PCA).}
\label{fig:COIL100_grad}
\end{figure}

\begin{figure}%[htbp]
\centering
%\small
\includegraphics[width=0.24\linewidth]{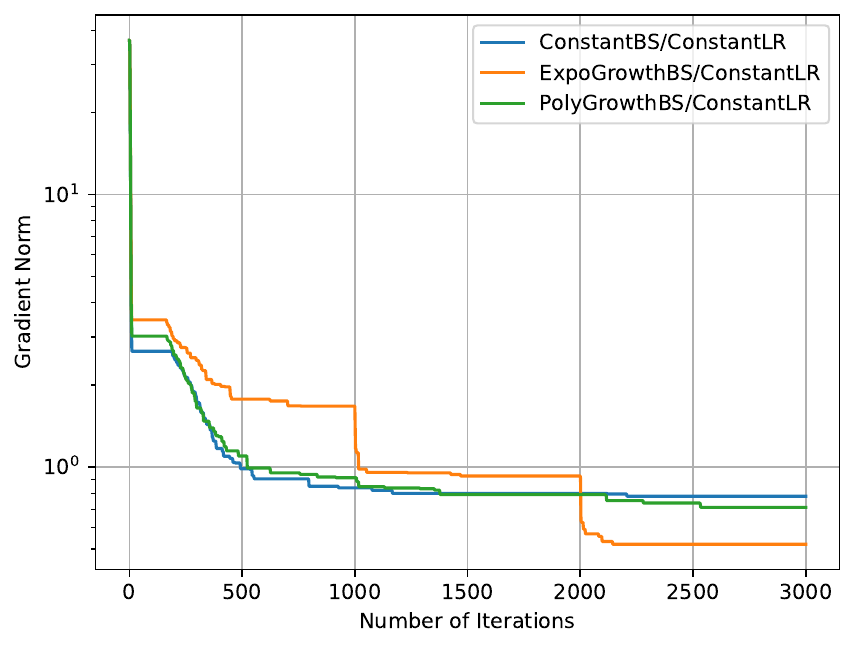}
\includegraphics[width=0.24\linewidth]{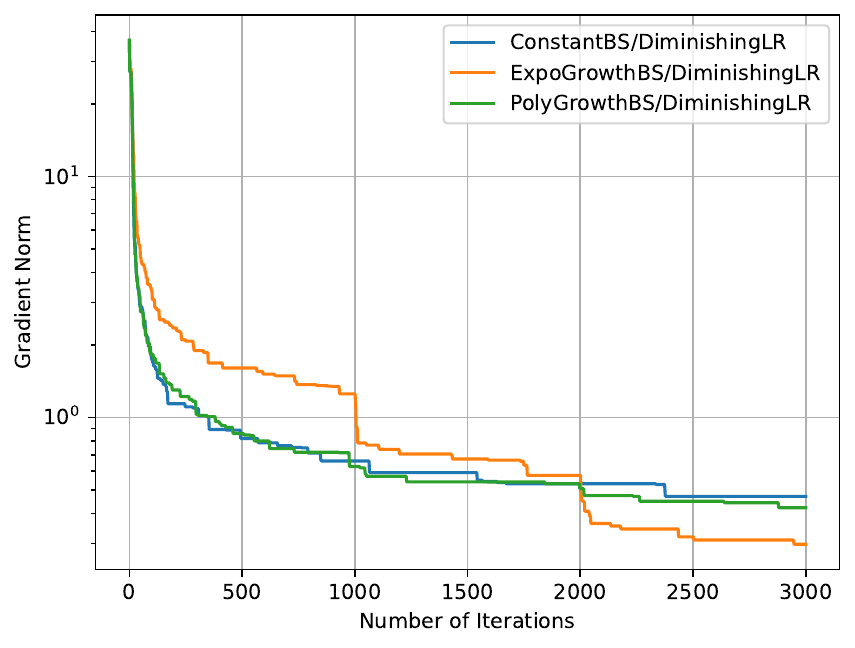}
\includegraphics[width=0.24\linewidth]{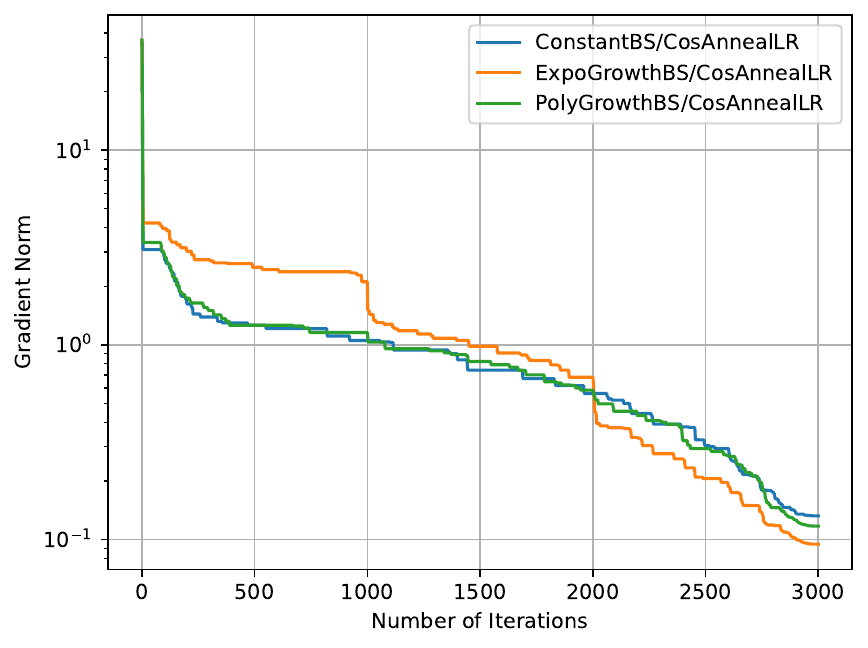}
\includegraphics[width=0.24\linewidth]{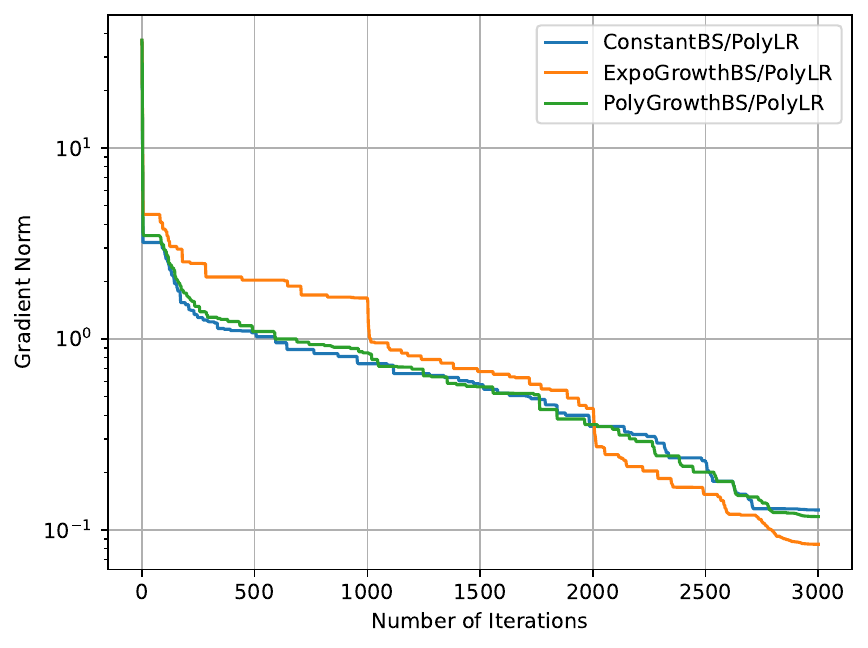}
\caption{Norm of the gradient of the objective function versus the number of iterations for LRs \eqref{eq:const_lr}, \eqref{eq:dim_lr}, \eqref{eq:cosan_lr}, and \eqref{eq:poly_dec_lr} on MNIST dataset (PCA).}
\label{fig:MNIST_grad}
\end{figure}
\noindent
We can formulate the PCA problem as an optimization problem on the Stiefel manifold \citep{Kasai2019Riem}; for a given dataset $\{x_j\}_{j=1,\ldots,N} \subset \mathbb{R}^{n}$ and $r$ $(\leq n)$, $\text{minimize}_{U\in\mathrm{St}(r,n)}f(U) \coloneq \frac{1}{N}\sum_{j=1}^N \|x_j - UU^\top x_j\|^2$, where $\mathrm{St}(r,n)\coloneq\{U\in\mathbb{R}^{r\times n}\mid U^{\top}U = I_n\}$. We set $r=10$ and used the COIL100 \citep{Nene1996Coil100} and MNIST \citep{LeCun1998Themnist} datasets. The Columbia Object Image Library (COIL100) dataset contains $7200$ color camera images of $100$ objects ($72$ poses per object) taken from different angles. We resized the images to $32\times 32$ pixels and transformed each one into a $1024$ $(=32^2)$ dimensional vector. Hence, we set $(N,n,r)=(7200,1024,10)$. The MNIST dataset contains $60,000$ $28\times28$-pixel grayscale images of handwritten digits $0$ to $9$. We transformed each image into a $784$ $(=28^2)$ dimensional vector and normalized each pixel to the range $[0,1]$. Hence, we set $(N,n,r)=(60000,784,10)$. Furthermore, we used a constant BS with $b_t\coloneq 2^{10}$, an exponential growth BS with an initial value $b_0\coloneq 3^5$, and a polynomial growth BS with an initial value $b_0\coloneq 30$.

\subsection{Low-rank Matrix Completion}
\label{nume:lrmc}
\begin{figure}%[htbp]
\centering
%\small
\includegraphics[width=0.24\linewidth]{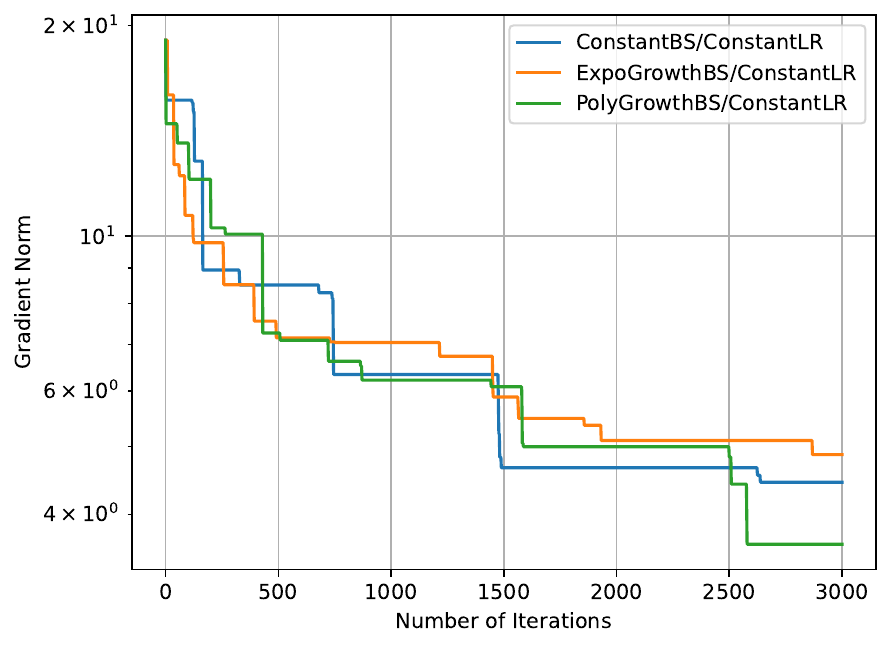}
\includegraphics[width=0.24\linewidth]{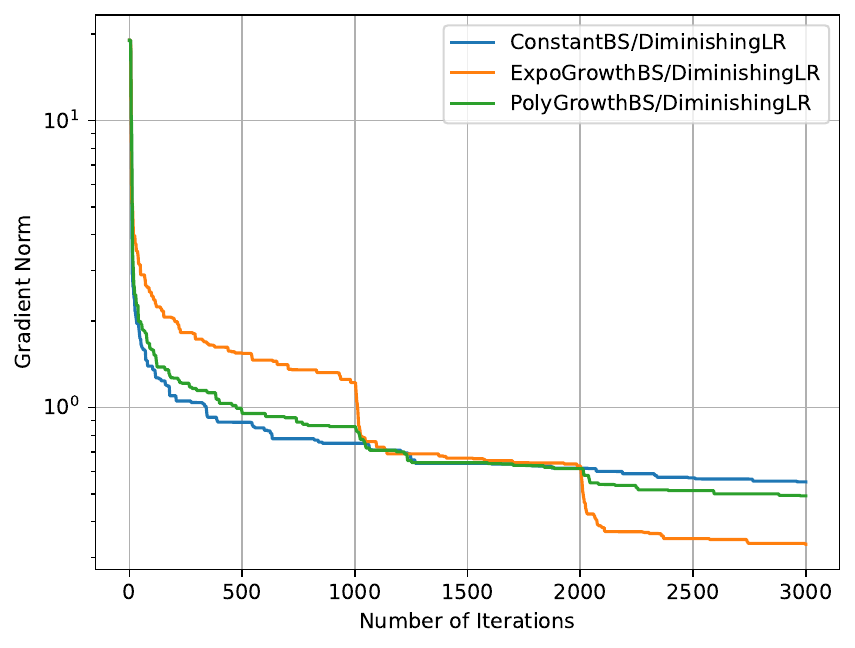}
\includegraphics[width=0.24\linewidth]{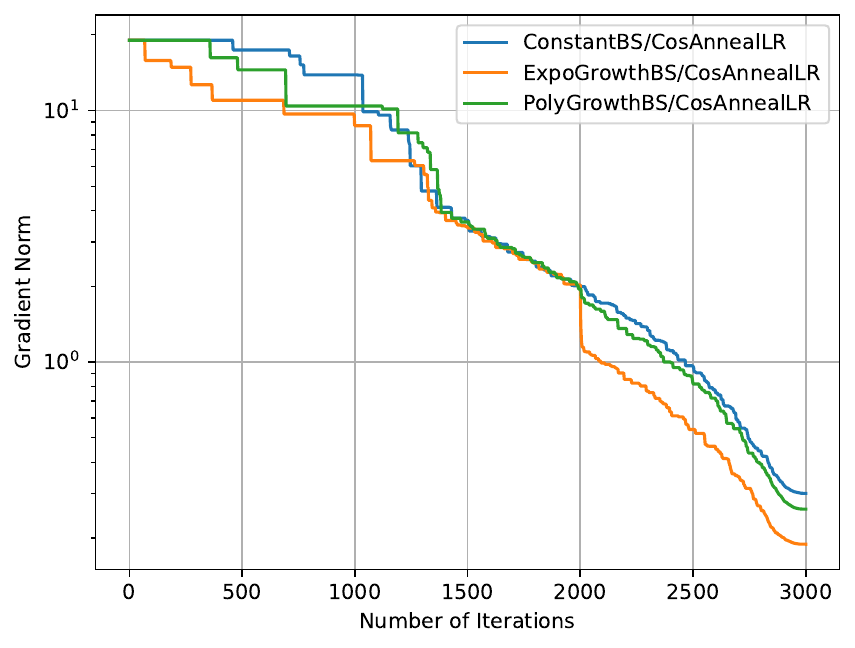}
\includegraphics[width=0.24\linewidth]{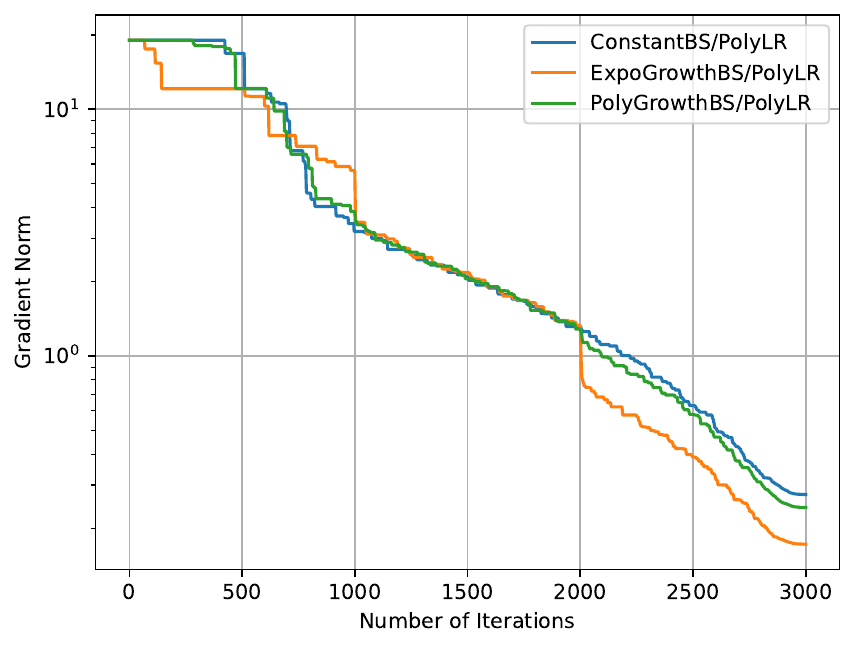}
\caption{Norm of the gradient of the objective function versus the number of iterations for LRs \eqref{eq:const_lr}, \eqref{eq:dim_lr}, \eqref{eq:cosan_lr}, and \eqref{eq:poly_dec_lr} on MovieLens-1M dataset (LRMC).}
\label{fig:ml-1m_grad}
\end{figure}

\begin{figure}%[htbp]
\centering
%\small
\includegraphics[width=0.24\linewidth]{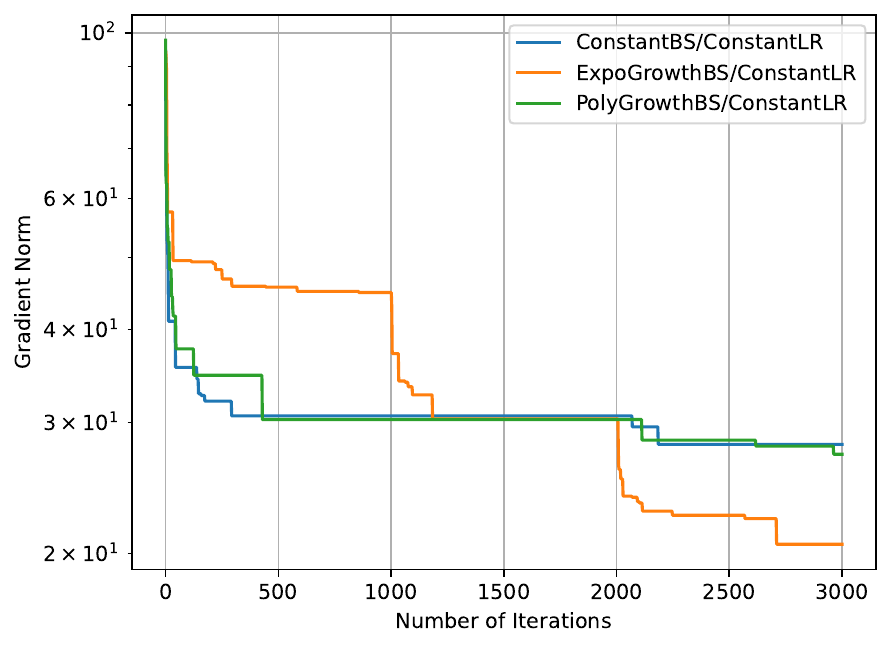}
\includegraphics[width=0.24\linewidth]{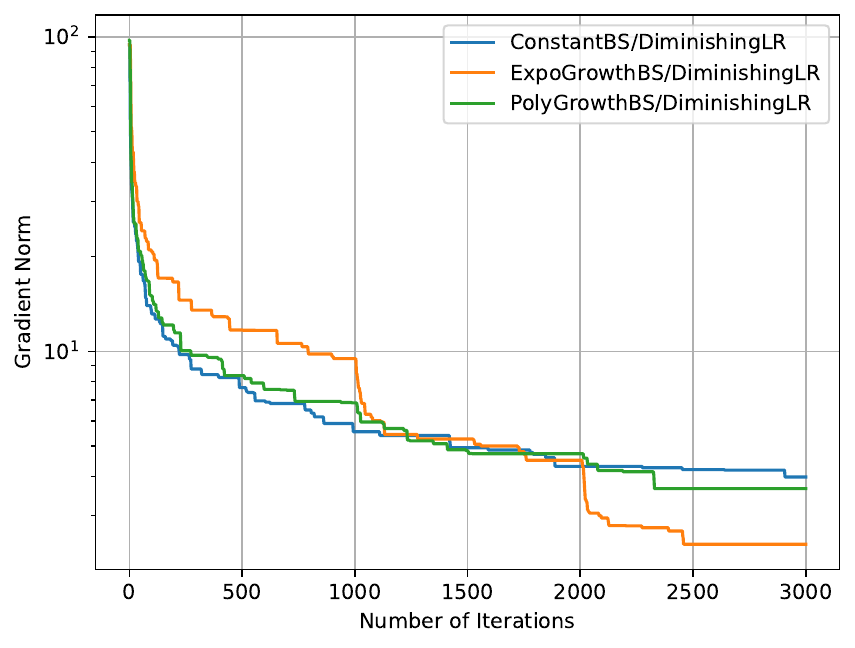}
\includegraphics[width=0.24\linewidth]{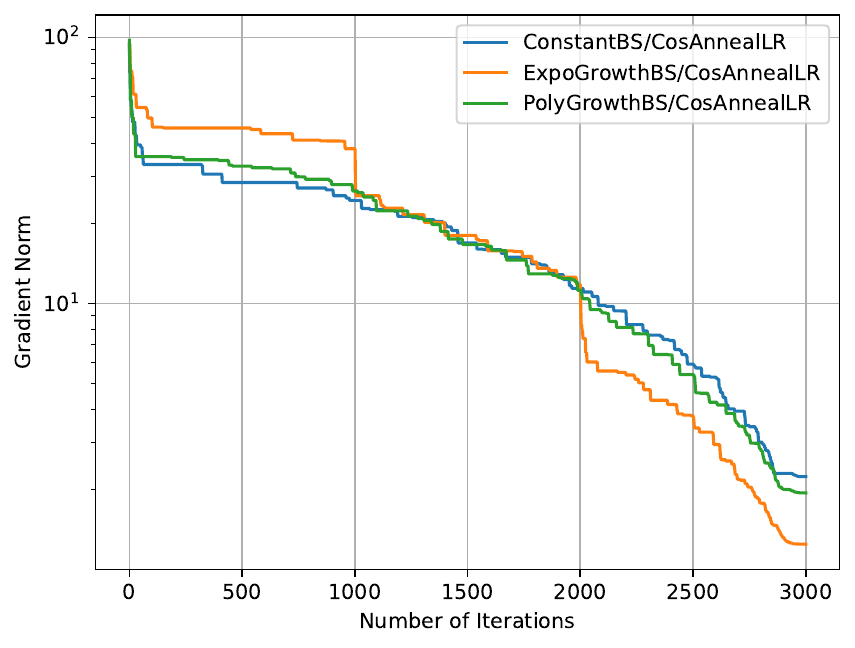}
\includegraphics[width=0.24\linewidth]{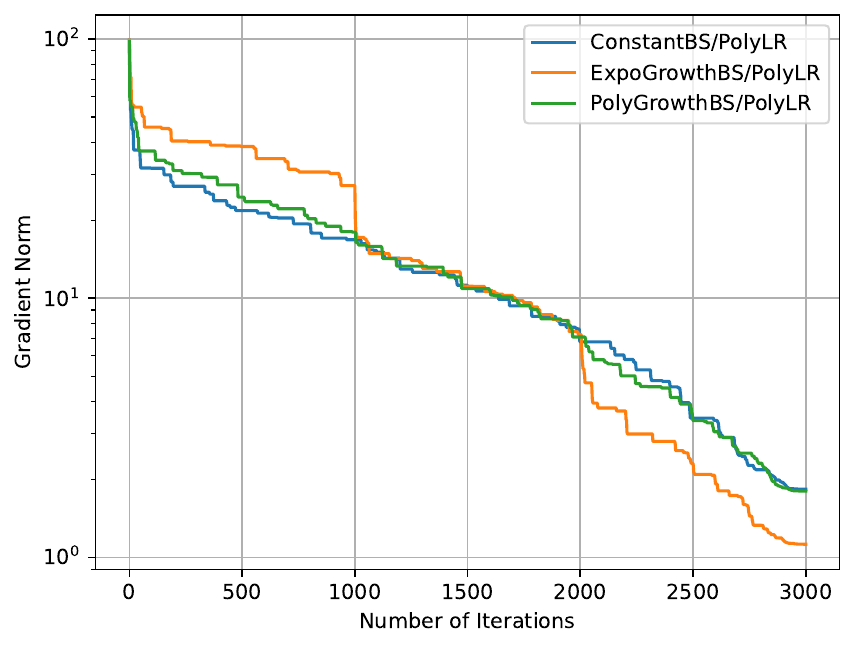}
\caption{Norm of the gradient of the objective function versus the number of iterations for LRs \eqref{eq:const_lr}, \eqref{eq:dim_lr}, \eqref{eq:cosan_lr}, and \eqref{eq:poly_dec_lr} on Jester dataset (LRMC).}
\label{fig:jester_grad}
\end{figure}
\noindent
The LRMC problem involves completing an incomplete matrix $Z=(z_1,\ldots,z_N)\in\mathbb{R}^{n\times N}$; $\Omega$ denotes a set of indices for which we know the entries in $Z$. For $a\in\mathbb{R}^n$, we define $P_{\Omega_i}(a)$ such that the $j$-th element is $a_j$ if $(i,j)\in\Omega$ and $0$ otherwise. For $U\in\mathbb{R}^{n\times r}, z\in\mathbb{R}^n$, $q_j(U,z):=\mathrm{argmin}_{a\in\mathbb{R}^r}\|P_{\Omega_j}(Ua-z)\|$. We can now formulate the LRMC problem as the following optimization problem on the Grassmann manifold \citep{Boumal2015Lowr}: $\text{minimize}_{U\in\mathrm{Gr}(r, n)}f(U) \coloneq \frac{1}{N}\sum_{j=1}^N \|P_{\Omega_j}(Uq_j(U,z_j) - z_j) \|^2$, where $\mathrm{Gr}(r, n)\coloneq \text{St}(r,n)/\text{O}(r)$. We set $r = 10$ and used the MovieLens-1M \citep{Maxwell2016Themov} and Jester datasets \citep{Goldberg2001Eigent}. The MovieLens-1M dataset contains $1,000,209$ ratings given by $6040$ users on $3952$ movies. Every rating lies in $[0,5]$. We normalized each rating to the range $[0,1]$. Hence, we set $(N,n,r)=(3952,6040,10)$. The Jester dataset contains ratings of $100$ jokes from $24,983$ users. Every rating is bounded by the range $[-10,10]$. Hence, we set $(N,n,r)=(24983,100,10)$. Furthermore, we used a constant BS $b_t\coloneq 2^8$, an exponential growth BS with an initial value of $b_0\coloneq 3^4$, and a polynomial growth BS with an initial value of $b_0\coloneq 14$.

The performances in terms of the gradient norm of the objective function versus the number of iterations for LRs \eqref{eq:const_lr}, \eqref{eq:dim_lr}, \eqref{eq:cosan_lr}, and \eqref{eq:poly_dec_lr} on the COIL100, MNIST, MovieLens-1M, and Jester datasets are plotted in Figures \ref{fig:COIL100_grad}, \ref{fig:MNIST_grad}, \ref{fig:ml-1m_grad}, and \ref{fig:jester_grad}, respectively. Achieving a small gradient norm was better with an increasing BS than with a constant BS. Among the increasing BSs, the exponential growth BS outperformed the polynomial BS. Because the magnitude of increase in the exponential BS is $O(\gamma^m)$ and that in the polynomial BS is $O(m)$, these numerical results indicate that a larger rate of increase in BS leads to better performance.

\subsection{Comparison of Computational Time between Constant BS and Increasing BS versus SFO Complexity}\label{sec:SFO}
\begin{figure}%[htbp]
\centering
\includegraphics[width=0.24\linewidth]{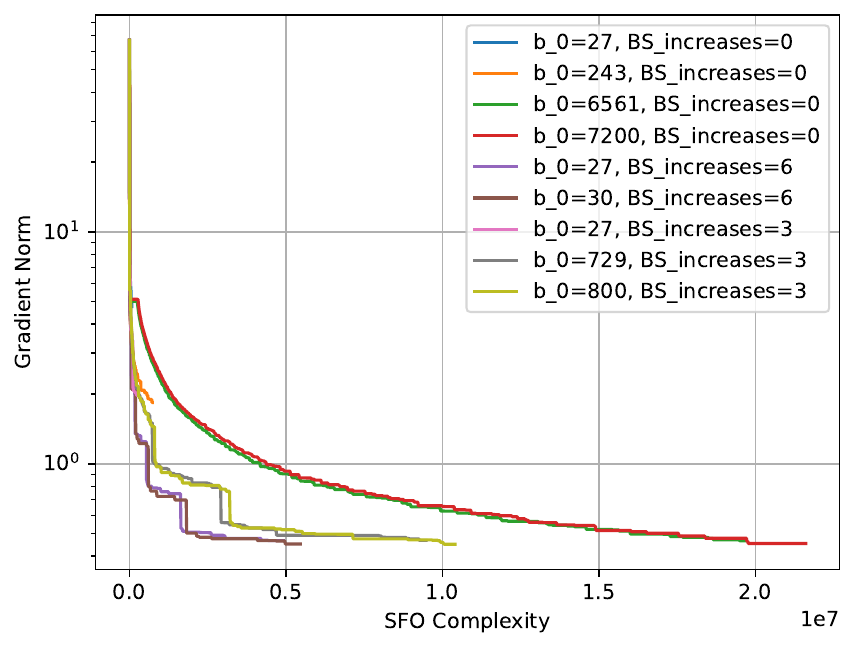}
\includegraphics[width=0.24\linewidth]{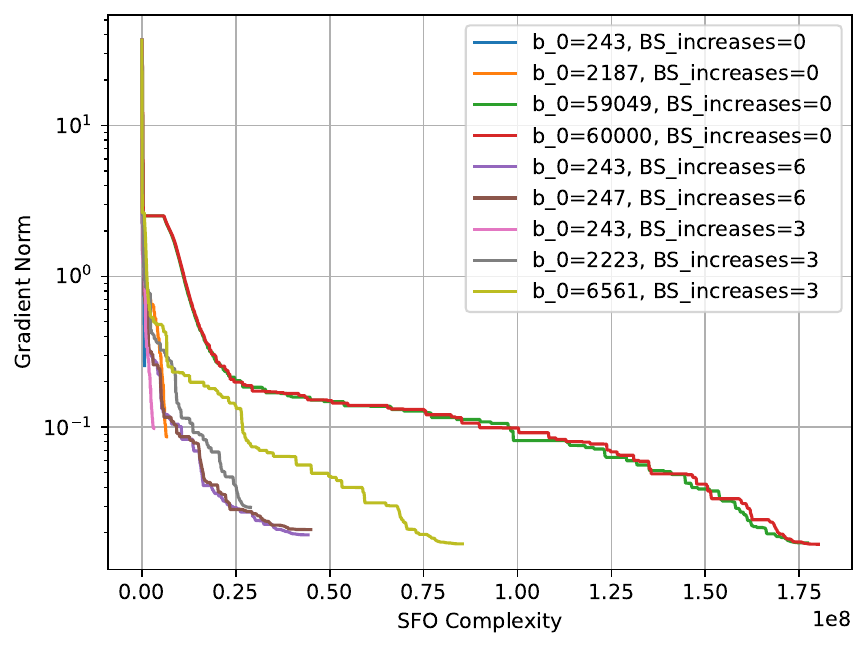}
\includegraphics[width=0.24\linewidth]{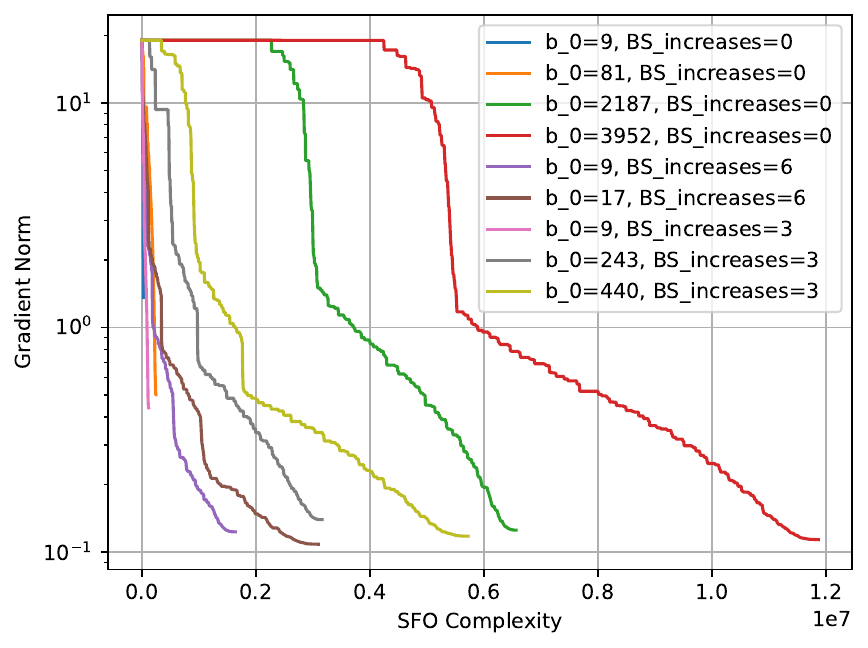}
\includegraphics[width=0.24\linewidth]{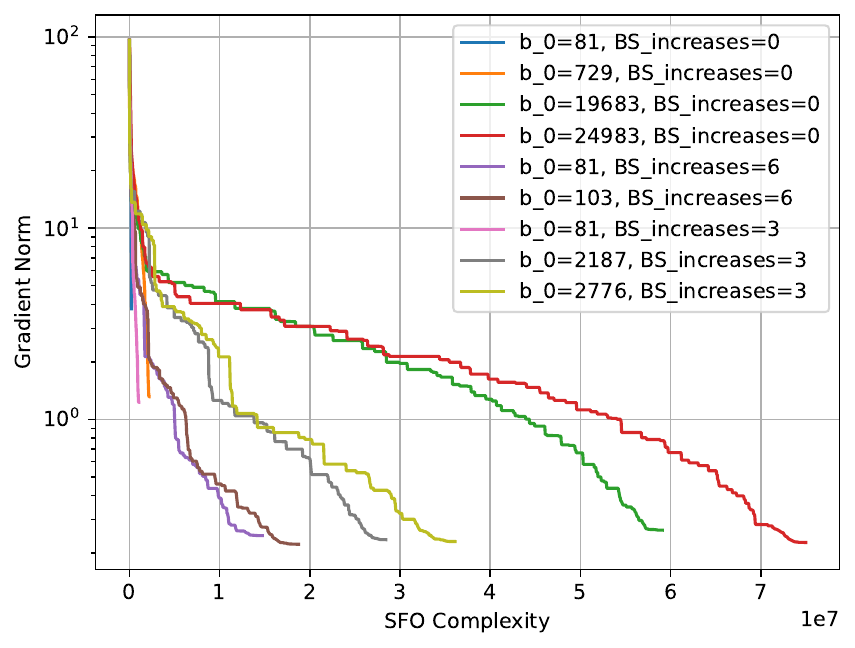}
\caption{Norm of objective function gradient versus SFO complexity. Datasets used were COIL100 (PCA), MNIST (PCA), MovieLens-1M (LRMC), and Jester (LRMC) in order from left to right. A cosine annealing LR was used except for COIL100, for which a constant LR was used. For `BS\_increases$=0$,' a constant BS $b=b_0$ was used. For `BS\_increases$=3$’ and `BS\_increases$=6$,' the BS was increased $3$ and $6$ times, respectively, in accordance with the exponential growth BS.}
\label{fig:sfo_vs_gradnorm}
\end{figure}

\begin{table}[]
\centering
\small
\begin{tabular}{|c|c|c|c|c|}
\hline
& COIL100 (PCA) & MNIST (PCA) & MovieLens-1M (LRMC) & Jester (LRMC)\\
\hline
Large constant BS & 1411.37 & 2697.24 & 4482.41 & 5432.54\\
Full constant BS & 1488.34 & 2719.66 & 12947.91 & 6816.28\\
Increasing BS & \textbf{468.53} & \textbf{745.13} & \textbf{1133.92} & \textbf{1695.85}\\
\hline
\end{tabular}
\caption{Computational time [s] (CPU time) for large-constant, full-constant and increasing batch sizes.}
\label{table:comp_time}
\end{table}
What are the differences between using a large constant BS, a small constant BS, and an increasing BS? We numerically investigated this question on the basis of SFO complexity, defined as the number of stochastic gradient evaluations executed over $T$ iterations \citep{Agarwal2015Ifo, Shallue2019Measuring, Sato2023Existe}. For a constant BS $b$, SFO complexity is represented by $bT$. For other BSs, it can be computed numerically, serves as a proxy for computational time. Figure \ref{fig:sfo_vs_gradnorm} plots the gradient norm of the objective function versus SFO complexity for PCA and LRMC. Each curve corresponds to RSGD for $3000$ steps under one of the following settings: constant BS, BS tripled every $1000$ steps (three increases in total), or BS tripled every $500$ steps (six increases in total). These settings follow the update formula for constant or exponential growth BS. Additional experimental details are provided in the caption of Figure \ref{fig:sfo_vs_gradnorm}. As shown in the figure, an increasing BS combines the advantages of both small and large BSs—namely shorter computational time and convergence to a solution with a smaller final gradient norm.

As shown in Figure \ref{fig:sfo_vs_gradnorm}, although both full and large constant BSs lead to optima achieving a small gradient norm, they require more SFO complexity. In contrast, while small constant BSs require less SFO complexity, they do not lead to optima achieving a small gradient norm. As shown in Table \ref{table:comp_time}, the computational time (CPU time) of an increasing BS is shorter than that of both large and full constant BSs. Between small and large constant BSs, there is a trade-off between convergence to a solution with a smaller final gradient norm and a shorter computational time. Our results show that an increasing BS balances this trade-off effectively.

\textbf{Why is increasing BS better?} Our theoretical analyses provide an answer to this question. From Theorem \ref{thm:c_d}, the SFO complexity of a constant BS with the critical BS for achieving $\|G_T\|^2\leq\epsilon^2$ is $O(\epsilon^{-4})$. From Theorem \ref{thm:i_d}, the SFO complexity of an exponential growth BS for achieving $\|G_T\|^2\leq\epsilon^2$ is $O(\gamma^{\epsilon^{-2}})$, where $\gamma \coloneq 3$ in this experiment. Although $O(\gamma^{\epsilon^{-2}})$ is inferior to $O(\epsilon^{-4})$ in our strict theoretical setting—with fixed $K$ (number of steps for each BS) and dynamic $M$ (number of BS increases)—the assumptions in this experiment differ. Specifically, $K$ depended on the value assigned to $M$, and our analysis can be applied to this experimental setting. Under these conditions, the SFO complexity of an exponential growth BS for achieving $\|G_T\|^2\leq\epsilon^2$ is $O(\epsilon^{-2})$, which is superior to $O(\epsilon^{-4})$ and consistent with our experimental results. This provides our theoretical justification: increasing BS achieves lower SFO complexity than a constant BS. Summarizing the above, we obtain the following theorem. The derivation is provided in Appendix \ref{apdix:SFO}.
\begin{theorem}\label{thm:reduce_sfo}
Let $\epsilon >0$, and consider an LR scheduler other than the diminishing LR one. Under the assumptions of Theorems \ref{thm:c_d} and \ref{thm:i_d}, the SFO complexity for achieving $\|G_T \|^2<\epsilon^2$ is $O(\epsilon^{-4})$ for a constant BS and $O(\epsilon^{-2})$ for an increasing BS.
\end{theorem}
\begin{remark}
Under the rigorous conditions of our theory, $M\to\infty$ implies that $T=MK + \text{const}.\to\infty$ with fixed $K$. In contrast, the conditions of our experiment do not permit an increase in BS to achieve $T \to \infty$ because it would require $K \to \infty$ for $T$ to diverge with fixed $M$, which means that the number of steps for the initial BS is infinite. However, as is done in many experimental setups, including ours, $T$ is predetermined as a finite value before conducting the experiment. This setting is equivalent to using the number of steps as the stopping criterion. Under this practical setting, the conditions of our experiment are meaningful because $T<\infty$ implies $K<\infty$ with a fixed $M$. On the other hand, by considering the conditions of our proofs, our theoretical analysis can also be applied to calculate SFO complexities, even within our experimental setting. From these reasons, our theoretical analysis can support the numerical results. To enhance understanding, the derivation of the SFO complexity for an increasing BS can be found in Appendix \ref{apdix:SFO}.
\end{remark}
\begin{figure}
\centering
\includegraphics[width=0.24\linewidth]{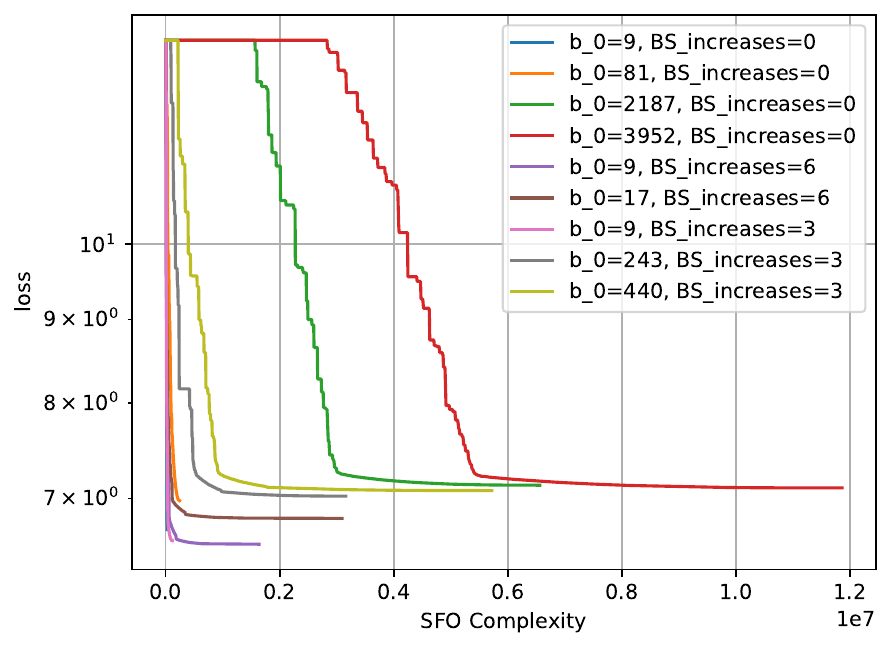}
\includegraphics[width=0.24\linewidth]{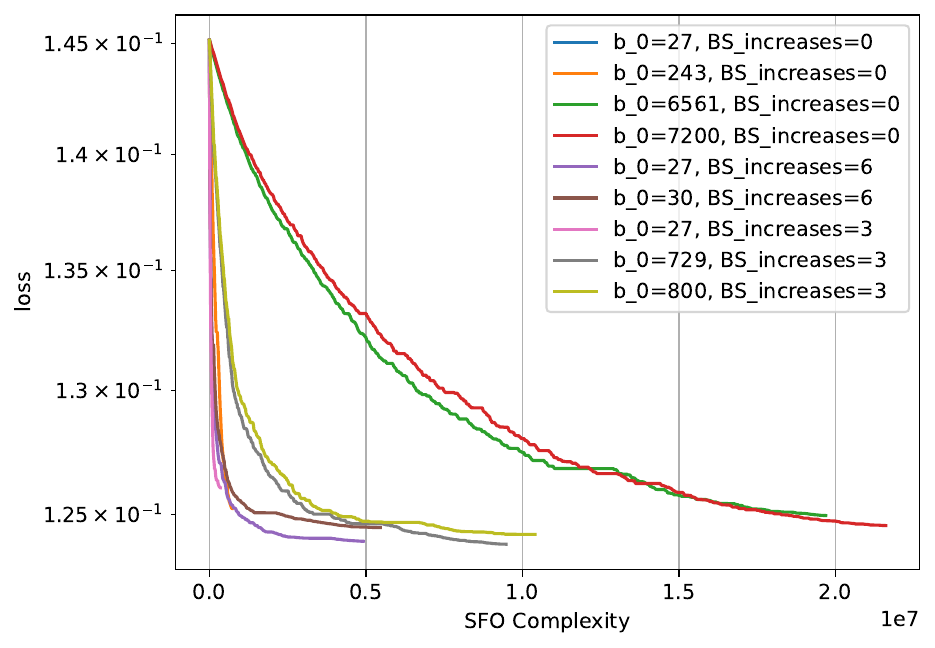}
\includegraphics[width=0.24\linewidth]{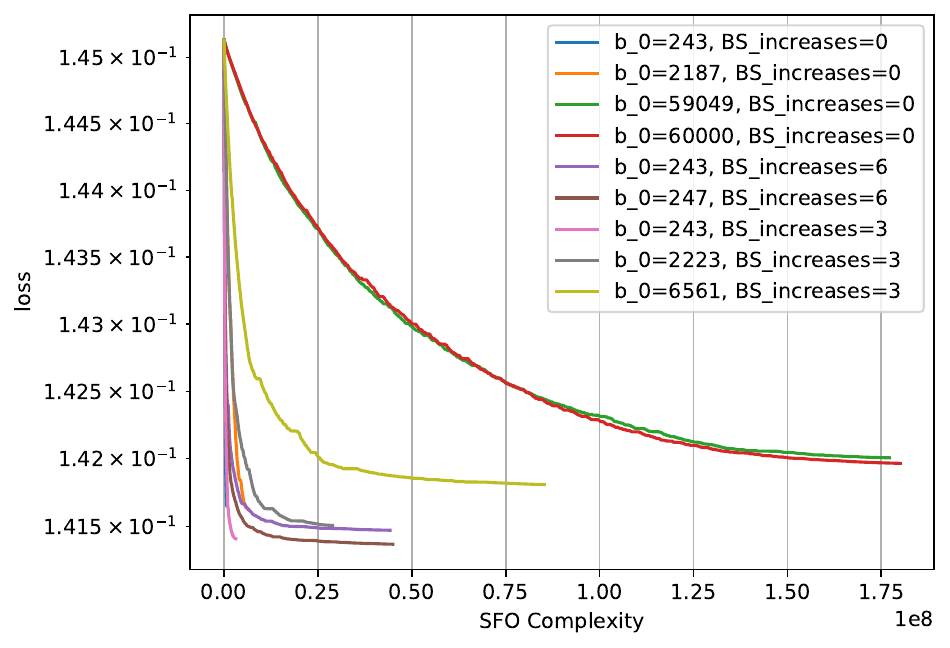}
\includegraphics[width=0.24\linewidth]{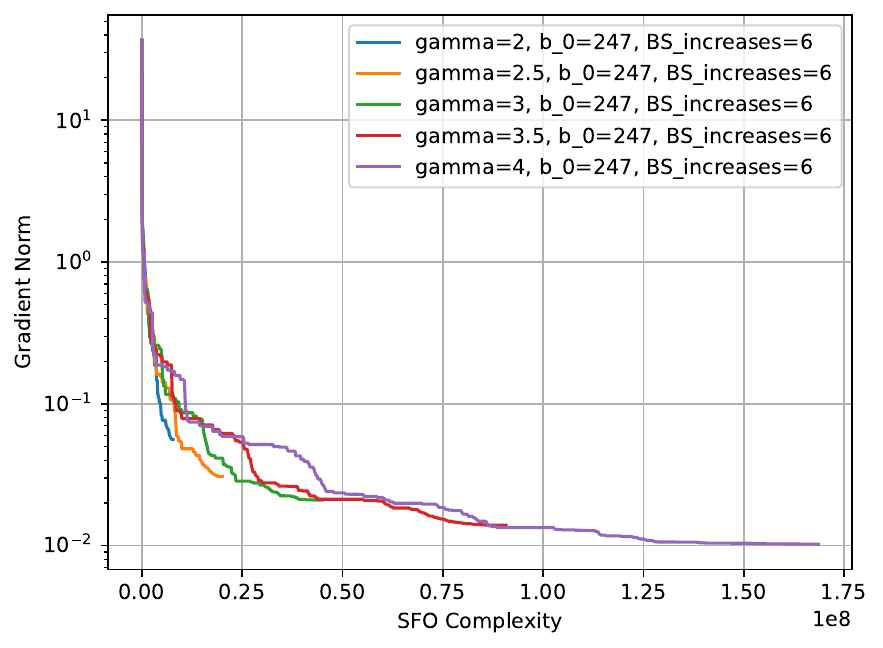}
\caption{The plot on the far left depicts the objective function value (loss) for MNIST (PCA) under the same setting as in Figure \ref{fig:sfo_vs_gradnorm}. The plot in the middle left (resp. right) depicts the objective function value (loss) for our proposed optimization problem when $(N,n)=(7200, 1024)$ (resp. $(60000, 1024)$). The plot on the far right depicts the trade-off with respect to the setting of $\gamma$.}
\label{fig:generali}
\end{figure}
\begin{table}
\centering
\footnotesize
\begin{tabular}{|c|c|c|c|c|c|c|c|c|c|}\hline
$(b_0, M)$ & (27,0) & (243, 0) & (6561, 0) & (7200, 0) &(27, 3) &(729, 3) &(800, 3) & (27, 6) &(30, 6) \\
\hline
$x$ (COIL100) & 2.233 & 1.055 & 0.903 & \textbf{0.902} & 1.173 & 0.911 & \textbf{0.902} & 0.973 & 0.951 \\
\hline\hline
$(b_0, M)$ & (243, 0) & (2187, 0) & (59049, 0) & (60000, 0) & (243, 3) & (2223, 3) & (6561, 3) & (243, 6) & (247, 6)\\
\hline
$y$ (MNIST) & 4.846 & 2.292 & 4.847 & 5.332 & 1.141 & \textbf{-4.622} & 2.134 & 4.247 & -4.139\\
\hline
\end{tabular}
\caption{Let $U_{\text{EVD}}$ be the solution to PCA by eigenvalue decomposition, and let $U_{\text{RSGD}}$ be the solution to PCA by RSGD. Then, the differences between the representations of the subspaces spanned by the principal components of $U_{\text{EVD}}$ and $U_{\text{RSGD}}$ are given by $l\coloneq\|U_{\text{EVD}}U_{\text{EVD}}^{\top}-U_{\text{RSGD}}U_{\text{RSGD}}^{\top}\|$. For comparison with the results for COIL100, we introduce $x$ such that $l=9.487+x\times 10^{-3}$, and for MNIST, we introduce $y$ such that $l=0.4+y\times 10^{-4}$. $M$ is the number of BS increments.}
\label{table:evd_pca}
\end{table}

\textbf{Objective Function Values}: Generalization performance is an important issue, and the objective function value serves as a criterion for measuring it. As shown in part of Appendix \ref{subapdix:sfo_obj}, although the exponential growth BS performed better than or equal to the constant BS with respect to the objective function values, the differences in the values were small. One possible reason for this is that the objective function may be flat around the optimal solution. To test this hypothesis, we conducted an experiment comparing the solution obtained by RSGD with that obtained by eigenvalue decomposition in PCA. When the data dimensionality is not large, PCA can obtain the exact solution by eigenvalue decomposition, and the difference from the approximate solution obtained by RSGD can be used as a measure of generalization. As shown in Table \ref{table:evd_pca}, the difference is small, which indicates that the objective functions of PCA with COIL100 and MNIST are flat around the optimal (exact) solution. Conversely, by considering optimization problems in which the objective functions do not have this property, the superiority of an increasing BS in terms of generalization performance can be revealed more clearly. In fact, this has already been demonstrated in LRMC with MovieLens-1M, as shown at the far left of Figure \ref{fig:generali} (the corresponding results for the other datasets are provided in Appendix \ref{subapdix:sfo_obj}). For further verification of this, we formulated a new optimization problem on the sphere $\text{S}^{n-1}$: $\text{minimize}_{w\in\text{S}^{n-1}}f(w)\coloneq\frac{1}{N}\sum_{j=1}^N\sqrt{|\langle x_j, w\rangle|}$, where $\{x_j\}_{j=1}^N$ is a dataset uniformly sampled from $\text{S}^{n-1}$ (see Appendix \ref{apdix:new_opt_prob} for more details). We set $(N,n)=(7200,1024)$ and $(60000,1024)$ and used the cosine annealing LR with an initial LR $\eta_{\max}=0.01$. The BS was configured in exactly the same way as that for the COIL100 and MNIST datasets in Section \ref{sec:SFO}. As in the two central graphs of Figure \ref{fig:generali}, an increasing BS achieved superior performance in terms of the objective function value for this proposed problem. Although the order of the differences (vertical axis) $O(10^{-2})-O(10^{-1})$ may appear small, it is 10–100 times larger than that observed for PCA because the differences in the values shown in Table \ref{table:evd_pca} are on the order of $O(10^{-4})-O(10^{-3})$. These results further confirm the validity of our hypothesis for this problem. Given these results, we suggest that an increasing BS should also improve generalization performance. Turning to a different topic, the objective function of our proposed optimization problem has an unbounded gradient norm (see Appendix \ref{apdix:new_opt_prob}), and the corresponding experiments numerically confirmed that our theoretical results are applicable even in this case. The graphs of the gradient norm for our proposed problem are provided in Appendix \ref{subapdix:new_prob_gradnorm}.

\textbf{Guidelines for setting $\gamma, b_0$ and $M$ in the exponential growth BS }: From Theorem \ref{thm:i_d}, $\|G_T\|^2=O(1+\frac{\gamma}{\gamma-1})=O(1+b_0^{-1})=O(M^{-1})$ and $\text{SFO}^{\text{incr}}_{\epsilon}=O(\gamma^{M})=O((b_0^2-1)^{-1}b_0^3)=O(\gamma^M M^{-1})$ hold. These theoretical results suggest a trade-off: for smaller gradient norm, $\gamma, b_0,$ and $M$ should be assigned larger values while for smaller SFO complexity, $\gamma, b_0$, and $M$ should be assigned smaller values. The far-right plot in Figure \ref{fig:generali} illustrates the trade-off between the gradient norm and SFO complexity achieved experimentally on MNIST (PCA) using the cosine annealing LR with an initial LR $\eta_{\max}=0.01$, where $(b_0,M)=(247,6)$, while varying $\gamma$ among $\{2.0, 2.5, 3.0, 3.5, 4.0\}$. For $b_0$, this trade-off can be observed by comparing the three entries under `BS\_increases$=3$' or the two entries under `BS\_increases$=6$' in Figure \ref{fig:sfo_vs_gradnorm}. For $M$, which represents the value `BS\_increases', this trade-off can be observed by comparing the purple entry with the pink entry in Figure \ref{fig:sfo_vs_gradnorm}. From these results, we were able to show, both theoretically and experimentally, that for each hyperparameter the trade-off between the gradient norm and SFO complexity holds. Figure \ref{fig:sfo_vs_gradnorm} shows that an increasing BS with small $b_0$ and large $M$ $(=6)$ leads to multiple optima characterized by a small gradient norm and relatively low SFO complexity, indicating that well-balanced configurations were achieved by taking into account the trade-offs inherent in $b_0$ and $M$ individually. The relevant derivation is provided in Appendix \ref{apdix:hypra_tra_deri}.

\section{Conclusion}
\label{sec:conclu}
Our theoretical analysis, conducted with several learning rate schedules, including cosine annealing and polynomial decay, demonstrated that using an increasing batch size rather than a constant batch size improves the convergence rate of Riemannian stochastic gradient descent. This result is supported by our experimental results. Furthermore, an increasing batch size yields optima characterized by a smaller gradient norm within a shorter computational time owing to reductions in the stochastic first-order oracle complexity. These findings indicate that an increasing batch size combines the advantages of both small and large constant batch size. Due to the nature of the experimental tasks, we were unable to directly investigate the effect on generalization performance. However, evaluation via the objective function value suggests that an increasing batch size also enhances generalization performance. We believe our results, which clarify one aspect of the effectiveness of an increasing batch size for generalization performance, provide valuable insight into this topic.

\acks{This work was supported by the Japan Society for the Promotion of Science (JSPS) KAKENHI Grant Number 24K14846 awarded to Hideaki Iiduka.}

\bibliography{biblio}
\newpage

\appendix
\section{Calculation of SFO Complexity}\label{apdix:SFO}
\subsection{SFO Complexity with Constant BS and Increasing BS}\label{subapdix:SFO}
We compute the SFO complexity required to achieve $\|G_T\|^2\leq \epsilon^2$ $(\epsilon > 0)$, as described in Section \ref{sec:SFO}.

\noindent
[\textbf{Constant BS}]\\
From Theorem \ref{thm:c_d}, the number of iterations $T$ required to achieve $\|G_T\|^2 \leq \epsilon^2$ is
\begin{equation*}
T = \frac{\tilde{Q}_1}{\epsilon^2 - \tilde{Q}_2 \sigma^2 b^{-1}}.
\end{equation*}
Because the SFO complexity of constant BS can be represented by $\text{SFO}^{\text{const}}(b)=bT$, the $\text{SFO}^{\text{const}}(b)$ required to achieve $\|G_T\|^2\leq \epsilon^2$ is given by
\begin{align*}
\text{SFO}_{\epsilon}^{\text{const}}(b)
= b T
= b \frac{\tilde{Q}_1}{\epsilon^2 - \tilde{Q}_2 \sigma^2 b^{-1}}
= \frac{b^2 \tilde{Q}_1}{b \epsilon^2 - \tilde{Q}_2 \sigma^2}.
\end{align*}
Because
\begin{align*}
\frac{d}{db}\text{SFO}_{\epsilon}^{\text{const}}(b)
= \frac{2b \tilde{Q}_1(b \epsilon^2 - \tilde{Q}_2 \sigma^2) - b^2 \tilde{Q}_1\epsilon^2}{(b \epsilon^2 - \tilde{Q}_2 \sigma^2)^2}
= \frac{b \tilde{Q}_1(b\epsilon^2 - 2\tilde{Q}_2\sigma^2)}{(b \epsilon^2 - \tilde{Q}_2 \sigma^2)^2}
\end{align*}
holds, from
\begin{align*}
b\leq \frac{2\tilde{Q}_2 \sigma^2}{\epsilon^2}\Rightarrow \frac{d}{db}\text{SFO}_{\epsilon}^{\text{const}}(b)\leq 0 \quad \text{and} \quad b\geq \frac{2\tilde{Q}_2 \sigma^2}{\epsilon^2}\Rightarrow \frac{d}{db}\text{SFO}_{\epsilon}^{\text{const}}(b)\geq 0,
\end{align*}
we have the critical BS of constant BS:
\begin{align*}
b^{\star}
\coloneq \frac{2\tilde{Q}_2 \sigma^2}{\epsilon^2}
= \underset{b>0}{\text{argmin}}\ \text{SFO}_{\epsilon}^{\text{const}}(b),
\end{align*}
which implies
\begin{align*}
\text{SFO}_{\epsilon}^{\text{const}}(b^{\star})
= \frac{4\tilde{Q}_1 \tilde{Q}_2^2 \sigma^4 \epsilon^{-4} }{\tilde{Q}_2 \sigma^2}
= 4\tilde{Q}_1 \tilde{Q}_2 \sigma^2 \epsilon^{-4}
= O(\epsilon^{-4}).
\end{align*}

\noindent
[\textbf{Increasing BS}]\newline
From Theorem \ref{thm:i_d}, the number of iterations $T$ required to achieve $\|G_T\|^2 \leq \epsilon^2$ is given by
\begin{equation*}
T = \frac{\tilde{Q}_1 + \tilde{Q}_2 \sigma^2 b_0^{-1}}{\epsilon^2}.
\end{equation*}
We assume a setting in which $T=MK$, meaning that by the $T$-th iteration, the BS has been increased $M$ times, and RSGD has been updated for $K$ steps at each BS. If the exact value of $T$ takes the form $T =MK + k$ $(1\leq k\leq K-1)$, we can set $T$ to $M(K+1)$ to ensure that $\|G_T\|^2\leq\epsilon^2$ still holds. Similarly, if $T$ takes the form $T =MK – k$ $(1\leq k\leq K-1)$, we can set $T$ to $MK$ to satisfy the same condition. Hence, our assumption that $T=MK$ is reasonable.\newline
SFO complexity for increasing BS (i.e., exponential growth BS) can be represented as
\begin{align*}
\text{SFO}^{\text{incr}}
= \left(\sum_{m=0}^{M-1} b_0 \gamma^m \right)K
= \frac{b_0}{\gamma -1} K(\gamma^M-1).
\end{align*}
Under our theoretical assumption that $K$ is fixed and $M$ is dynamic, and with $M = \frac{T}{K}$, the SFO complexity of increasing BS for achieving $\|G_T\|^2\leq\epsilon^2$ is
\begin{align*}
\text{SFO}_{\epsilon}^{\text{incr}}
= \frac{b_0 K}{\gamma -1} (\gamma^{\frac{\tilde{Q}_1 + \tilde{Q}_2 \sigma^2 b_0^{-1}}{K\epsilon^2}}-1)
= O(\gamma^{\epsilon^{-2}}).
\end{align*}
To calculate SFO complexity under our experimental setting, we must reconsider the proof \ref{prf:i_d} because it relies on letting $M\to\infty$, which is not permitted in our experimental setting. This issue can, however, be resolved by the following discussion. The key is to replace the evaluation at the penultimate inequality of \eqref{eq:esti_exp_bs} with
\begin{align*}
\sum_{t=0}^{T-1} \frac{1}{b_t}
&= \sum_{m=0}^{M-1} \frac{K}{b_0 \gamma^m} + \frac{T - KM}{b_0 \gamma^M}
\leq \sum_{m=0}^{M} \frac{K}{b_0 \gamma^m}
= \frac{K(1-\gamma^{-M-1})}{b_0(1-\gamma^{-1})}\\
&= \frac{K(\gamma^{M}-\gamma^{-1})}{b_0\gamma^{M-1}(\gamma -1)}
\leq \frac{K\gamma^{M}}{b_0\gamma^{M-1}(\gamma -1)}
=\frac{K\gamma}{b_0(\gamma -1)}.
\end{align*}
Thus,
\begin{align*}
\sum_{t=0}^{T-1} \frac{\eta_{\max}^2}{b_t}
\leq \frac{\eta_{\max}^2K\gamma}{b_0(\gamma -1)}
\end{align*}
implies the same result with Theorem \ref{thm:i_d}:
\begin{align*}
\|G_T\|^2
\leq \frac{\tilde{Q}_1+\tilde{Q_2}\sigma^2b_0^{-1}}{T}.
\end{align*}
Because $\tilde{Q}_2$ contains $K$, as shown by the formal statement of Theorem \ref{thm:i_d} (with a constant LR and an exponential growth BS), we rearrange $\tilde{Q}_2$ for calculating SFO complexity under our experimental setting as
\begin{align*}
\tilde{Q}_2 = K \tilde{Q}_3,
\end{align*} 
where $\tilde{Q_3}$ is defined through this substitution. Then, we obtain
\begin{align*}
MK 
= T
= \frac{\tilde{Q}_1 + \tilde{Q}_2 \sigma^2 b_0^{-1}}{\epsilon^2}
= \frac{\tilde{Q}_1 + \tilde{Q}_3 K \sigma^2 b_0^{-1}}{\epsilon^2},
\end{align*}
which implies 
\begin{align*}
K = \frac{\tilde{Q}_1}{M\epsilon^2 - \tilde{Q}_3 \sigma^2 b_0^{-1}}.
\end{align*}
Therefore, under our experimental setting—where $K$ is determined by fixed $M$—and with $K = \frac{T}{M}$, the SFO complexity of increasing BS for achieving $\|G_T\|^2\leq\epsilon^2$ becomes
\begin{align*}
\text{SFO}_{\epsilon}^{\text{incr}}
= \frac{b_0 (\gamma^M -1 )}{\gamma -1}\frac{\tilde{Q}_1}{M\epsilon^2 - \tilde{Q}_3 \sigma^2 b_0^{-1}}
= O(\epsilon^{-2}).
\end{align*}
The calculations in this section are for the case of a constant LR, but it is immediately clear from Theorems \ref{thm:c_d} and \ref{thm:i_d} that exactly the same results hold for a cosine annealing LR, which was also used in our experiments.

\subsection{Objective Function Value versus SFO Complexity}
\label{subapdix:sfo_obj}
\begin{figure}[H]
 \centering
 \includegraphics[width=0.24\linewidth]{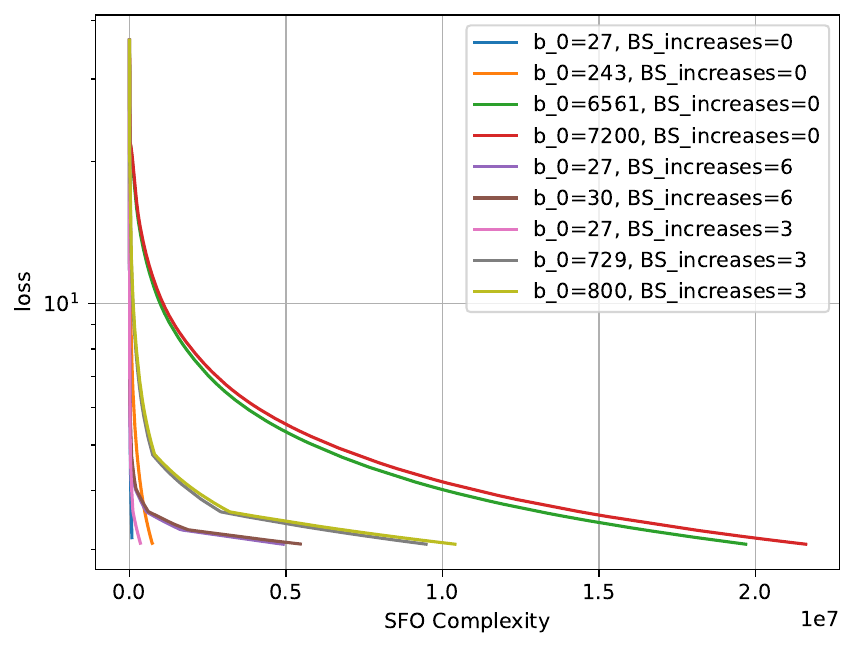}
 \includegraphics[width=0.24\linewidth]{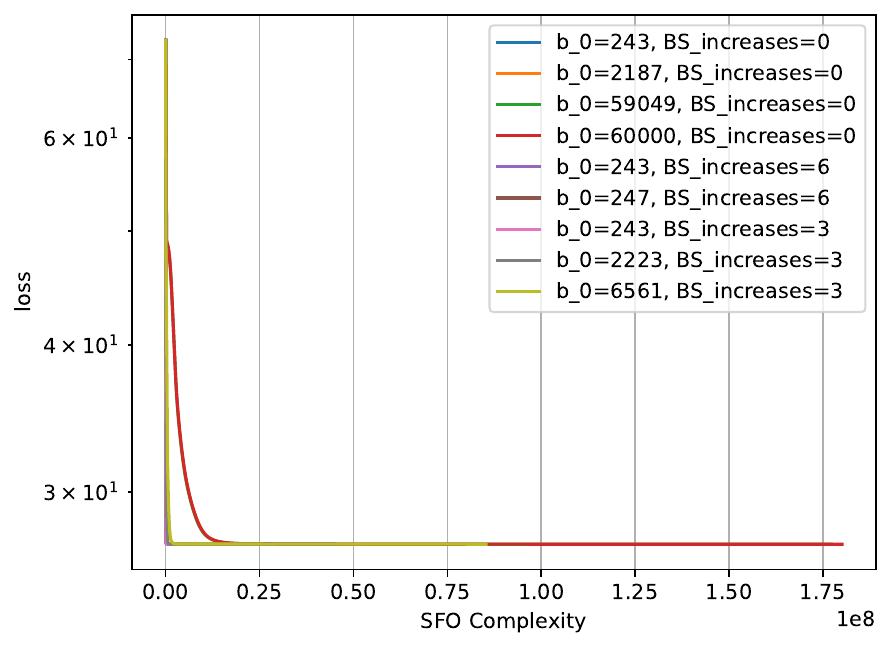}
 \includegraphics[width=0.24\linewidth]{graphs/SFO/ml-1m-loss-CosAnnealLR--0.01.pdf}
 \includegraphics[width=0.24\linewidth]{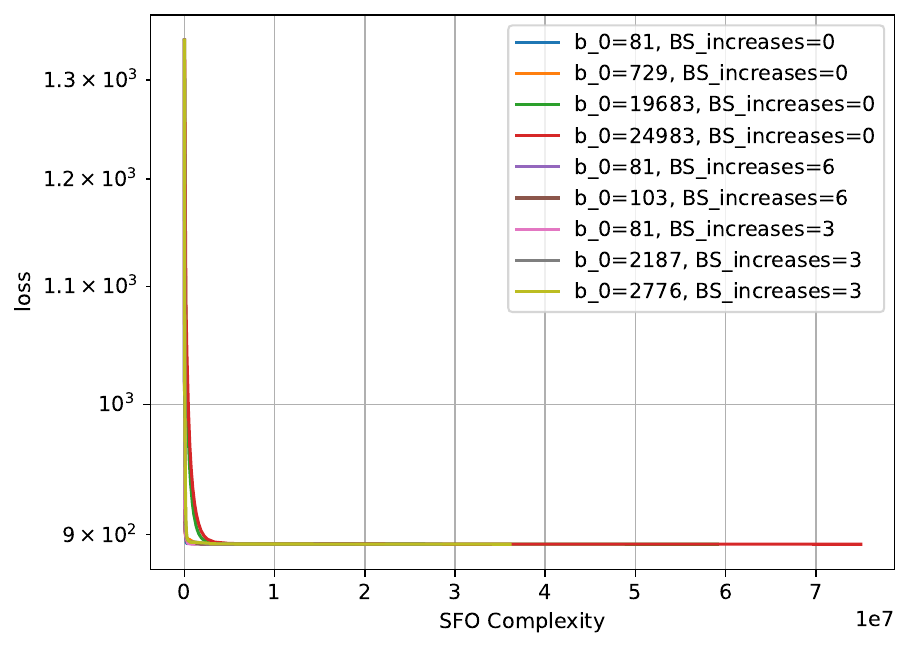}
 \caption{Objective function value (loss) versus SFO complexity on COIL100 (PCA), MNIST (PCA), MovieLens-1M (LRMC), and Jester (LRMC) datasets in order from left to right.}
 \label{fig:obj_SFO}
\end{figure}
For the COIL100 (PCA) and MovieLens-1M (LRMC), performance was better with an increasing BS than with a constant BS. Although the differences in the objective function values for the MNIST (PCA) and Jester (LRMC) are small (one possible reason for this is that the objective function may be flat around the optimal solution), the performance with an increasing BS was equal to or better than that with a constant BS. A more detailed discussion of this hypothesis is presented in Section \ref{sec:SFO}.

\section{Trade-offs for Each Hyperparameter of Exponential Growth BS Scheduler}\label{apdix:hypra_tra_deri}
We calculate the trade-offs for each hyperparameter of the exponential growth BS.

[$\gamma$]: From Theorem \ref{thm:i_d}, $\|G_T\|^2=O\left(1+\frac{\gamma}{\gamma-1}\right)$ and $\text{SFO}_{\epsilon}^{\text{incr}}=O(\gamma^M)$ hold, and $\gamma^M$ is clearly monotonically increasing with respect to $\gamma$. Let $f(\gamma)\coloneq1+\frac{\gamma}{\gamma-1}$. Because $f'(\gamma)=\frac{-1}{(\gamma-1)^2}<0$, $f$ is monotonically decreasing. Thus, for a smaller gradient norm, $\gamma$ should be set to a larger value, while for a smaller SFO complexity, $\gamma$ should be set to a smaller value.

[$b_0$]: From Theorem \ref{thm:i_d}, $\|G_T\|^2=O\left(1+b_0^{-1}\right)$ and $\text{SFO}_{\epsilon}^{\text{incr}}=O\left(\frac{b_0^3}{b_0^2-1}\right)$ hold. $1+b_0^{-1}$ is clearly monotonically decreasing. Let $g(x)\coloneq\frac{x^3}{x^2-1}\ (x>1)$. Because $g'(x)=\frac{x^2(x-\sqrt{3})(x+\sqrt{3})}{(x^2-1)^2}$, $g$ is monotonically decreasing for $1<x\leq\sqrt{3}$, and monotonically increasing for $x\geq\sqrt{3}$. Now, because $x>1$ and $b_0\in\mathbb{N}$, it suffices to consider only the case $x\geq 2$, in which $g$ is monotonically increasing. Summarizing the above and considering $b_0\in\text{dom}\ g$, we find that, for a smaller gradient norm, $b_0$ should be set to a larger value, while for a smaller SFO complexity, $b_0$ should be set to a smaller value.

[$M$]: Recall $T=MK$. From Theorem \ref{thm:i_d}, $\|G_T\|^2=O\left(\frac{1}{T}+\frac{K}{T}\right)=O\left(\frac{1}{MK}+\frac{K}{MK}\right)=O\left(M^{-1}\right)$, and $\text{SFO}_{\epsilon}^{\text{incr}}=O\left(\frac{\gamma^M}{M}\right)$ hold. $M^{-1}$ is clearly monotonically decreasing. Let $h(x)\coloneq\frac{\gamma^x}{x}\ (x\geq 1)$. Because $h'(x)=\frac{\gamma^x(x\ln\gamma-1)}{x^2}$, $h$ is monotonically increasing for $x\geq\frac{1}{\ln\gamma}$ when $\gamma\geq e$. Thus, when we set $\gamma\geq e$ ($\gamma$ was set to 3 in our experiments except for the far right plot in Figure \ref{fig:generali}), we find that, for a smaller gradient norm, $M$ should be set to a larger value, while for a smaller SFO complexity, $M$ should be set to a smaller value.

Finally, we provide a remark on the setting of $b_0$ in our experiments. Although the initial BS ($b_0$) appears different for each dataset, the values were similarly determined using the same method to enable a comparison of the effectiveness between a small constant BS, a large constant BS, and an increasing BS. The $N$ was set to the total number of data points in each dataset. For a large initial BS, it was set to $3^k$, where $k$ is the largest integer such that $3^k<N$. For a medium-size initial BS, it was set to $3^{k-3}$. For the smallest initial BS, it was set to $3^{k-5}$. For example, because $N=7200$ for the COIL100 dataset, the large, the medium-size small, and the smallest initial BSs were set to $3^8=6561, 3^5=243, 3^3=27$, respectively.

\section{Our Proposed Optimization Problem}\label{apdix:new_opt_prob}
Our proposed optimization problem:
\[
\begin{aligned}
&\text{minimize} &&f(w)\coloneq \frac{1}{N}\sum_{j=1}^N \sqrt{|\langle x_j, w\rangle|},\\
&\text{subject to} & &w \in\text{S}^{n-1}\coloneq\{w\in\mathbb{R}^n\mid\|w\|=1\}.
\end{aligned}
\]
$\{x_j\}_{j=1}^N$ is a dataset uniformly sampled from $\text{S}^{n-1}$. $N$ is a total number of data points in the dataset, and $n$ is a dimension of each one. The objective function defined above has unbounded gradient norms, which allows us to experimentally verify that our theoretical results are indeed applicable to functions with unbounded gradient norms. In our experiment, $R_x(d)\coloneq\frac{x+d}{\|x+d\|}$ was used as a retraction for $\text{S}^{n-1}$.
\begin{proposition}
The gradient norm of the objective function for this problem $f$ is not bounded.
\end{proposition}
\begin{proof}
When $\langle x_j, w\rangle\not=0\ (^{\forall}j\in\{1,\ldots,N\})$, 
\begin{align*}
\nabla f(w)=\frac{1}{N}\sum_{j=1}^N \text{Sign}(\langle x_j, w\rangle)|\langle x_j, w\rangle|^{-\frac{1}{2}}x_j
\end{align*}
is obtained by
\begin{align*}
\frac{\partial f}{\partial w_i}(w)
&=\frac{1}{N}\sum_{j=1}^N \frac{\partial }{\partial w_i}\sqrt{|\langle x_j, w\rangle|}
= \frac{1}{N}\sum_{j=1}^N |\langle x_j, w\rangle|^{-\frac{1}{2}}\text{Sign}(\langle x_j, w\rangle)\frac{\partial }{\partial w_i}\sum_{k=1}^N x_j^k w_k\\
&=\frac{1}{N}\sum_{j=1}^N \text{Sign}(\langle x_j, w\rangle)|\langle x_j, w\rangle|^{-\frac{1}{2}}x_j^i,
\end{align*}
where $w=(w_1,\cdots,w_n)^{\top}, x_j=(x_j^1,\cdots,x_j^n)^{\top}\in\text{S}^{n-1}\subset\mathbb{R}^n$. Thus, the gradient on $\text{S}^{n-1}$ is
\begin{align*}
\text{grad}f(w)=(I-ww^{\top})\nabla f(w)=\frac{1}{N}\sum_{j=1}^N \text{Sign}(\langle x_j, w\rangle)|\langle x_j, w\rangle|^{-\frac{1}{2}}(I-ww^{\top})x_j.
\end{align*}
Fix one data $x_i$ such that $x_i\not=0$. Let $u\in\text{S}^{n-1}$ be a vector obtained by orthonormalizing $x_i$ via the Gram-Schmidt process. We consider a sequence $(w_t)_{t\in\mathbb{N}}$ such that $w_t\coloneq \frac{x_i}{t\|x_i\|}+\sqrt{1-\frac{1}{t^2}}u$, which lies in $\text{S}^{n-1}$ because
\begin{align*}
\|w_t\|^2
=\frac{1}{t^2\|x_i\|^2}\|x_i\|^2+\left(1-\frac{1}{t^2}\right)\|u\|^2+\frac{2}{t\|x_i\|}\sqrt{1-\frac{1}{t^2}}\langle x_i,w\rangle
=1.
\end{align*}
By simple calculations,
\begin{align*}
w_tw_t^{\top}
&=\left(\frac{x_i}{t\|x_i\|}+\sqrt{1-\frac{1}{t^2}}u\right)\left(\frac{x_i^{\top}}{t\|x_i\|}+\sqrt{1-\frac{1}{t^2}}u^{\top}\right)\\
&=\frac{x_i x_i^{\top}}{t^2\|x_i\|^2}+\frac{1}{t^2\|x_i\|^2}\sqrt{1-\frac{1}{t^2}}(x_i u^{\top}+ u x_i^{\top})+ \left(1-\frac{1}{t^2}\right)uu^{\top}\\
&\overset{t\to\infty}{\longrightarrow} uu^{\top}
\end{align*}
and
\begin{align*}
\langle x_j, w_t\rangle
=\left\langle x_j, \frac{x_i}{t\|x_i\|}+\sqrt{1-\frac{1}{t^2}}u\right\rangle
=\frac{\langle x_j, x_i\rangle}{t\|x_i\|}+\sqrt{1-\frac{1}{t^2}}\langle x_j, u\rangle
\end{align*}
hold. Now, we define a set $\Lambda$ satisfying $\{x_j\}_{j\in \Lambda}\subset\{x_j\}_{j=1}^N$ such that $\langle x_j, u\rangle=0$. We find $\langle x_j, w_t\rangle=\frac{\langle x_j, x_i\rangle}{t\|x_i\|}(^{\forall}j\in\Lambda)$, and $\Lambda\not=\varnothing$ because of $i\in\Lambda$. Note that, for all $j\not\in\Lambda$, $\langle x_j, u\rangle\not=0$ holds, which yields $\langle x_j, w_t\rangle\not=0$. From these observations, if $\langle x_j, x_i\rangle\not=0\ (^{\forall}j\in\Lambda)$,
\begin{align*}
\|\text{grad}f(w_t)\|
&=\left\|\frac{1}{N}\sum_{j=1}^N \text{Sign}(\langle x_j, w_t\rangle)|\langle x_j, w_t\rangle|^{-\frac{1}{2}}(I-w_tw_t^{\top})x_j\right\|\\
&=\frac{1}{N}\left\|\sum_{j=1}^N |\langle x_j, w_t\rangle|^{-\frac{1}{2}}(I-w_tw_t^{\top})x_j\right\|\\
&\geq \frac{1}{N}\left\|\sum_{j\in\Lambda}|\langle x_j, w_t\rangle|^{-\frac{1}{2}}(I-w_tw_t^{\top})x_j\right\| - \frac{1}{N}\left\|\sum_{j\not\in\Lambda} |\langle x_j, w_t\rangle|^{-\frac{1}{2}}(I-w_tw_t^{\top})x_j\right\|\\
&=\frac{1}{N}\sqrt{\frac{t\|x_i\|}{\langle x_j, x_i\rangle}}\left\|\sum_{j\in\Lambda}(I-w_tw_t^{\top})x_j\right\|- \frac{1}{N}\left\|\sum_{j\not\in\Lambda} \left|\frac{\langle x_j, x_i\rangle}{t\|x_i\|}+\sqrt{1-\frac{1}{t^2}}\langle x_j, u\rangle\right|^{-\frac{1}{2}}(I-w_tw_t^{\top})x_j\right\|\\
&\geq \frac{1}{N}\sqrt{\frac{t\|x_i\|}{\langle x_j, x_i\rangle}}\left\|\sum_{j\in\Lambda}(I-w_tw_t^{\top})x_j\right\|- \frac{1}{N}\sum_{j\not\in\Lambda} \left|\frac{\langle x_j, x_i\rangle}{t\|x_i\|}+\sqrt{1-\frac{1}{t^2}}\langle x_j, u\rangle\right|^{-\frac{1}{2}}\left\|(I-w_tw_t^{\top})x_j\right\|\\
&\overset{t\to\infty}{\longrightarrow}\left(\infty - \frac{1}{N}\sum_{j\not\in\Lambda} \left|\langle x_j, u\rangle\right|^{-\frac{1}{2}}\left\|(I-uu^{\top})x_j\right\|
=\right)\infty
\end{align*}
holds; even if $\langle x_j, x_i\rangle=0\ (^{\exists}j\in\Lambda)$, for all $t\in\mathbb{N}$,
\begin{align*}
\|\text{grad}f(w_t)\|
&\geq \frac{1}{N}\left\|\sum_{j\in\Lambda}|\langle x_j, w_t\rangle|^{-\frac{1}{2}}(I-w_tw_t^{\top})x_j\right\| - \frac{1}{N}\left\|\sum_{j\not\in\Lambda} |\langle x_j, w_t\rangle|^{-\frac{1}{2}}(I-w_tw_t^{\top})x_j\right\|\\
&= \left(\infty - \frac{1}{N}\left\|\sum_{j\not\in\Lambda} |\langle x_j, w_t\rangle|^{-\frac{1}{2}}(I-w_tw_t^{\top})x_j\right\|=\right)\infty
\end{align*}
holds. These complete the proof.
\end{proof}

\subsection{Gradient Norm versus SFO Complexity for Our Proposed] Optimization Problem}\label{subapdix:new_prob_gradnorm}
\begin{figure}[H]
\centering
\includegraphics[width=0.49\linewidth]{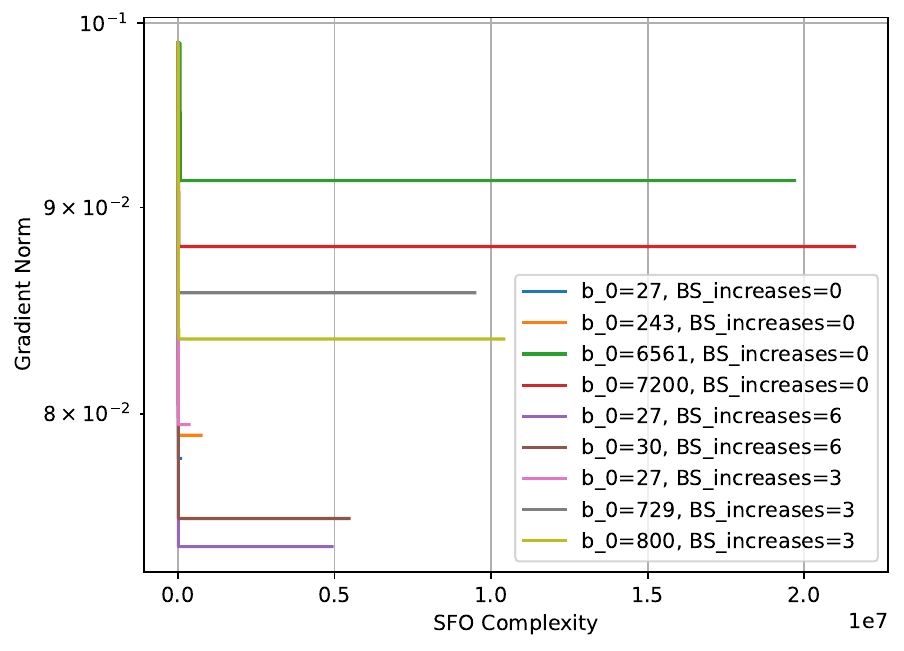}
\includegraphics[width=0.49\linewidth]{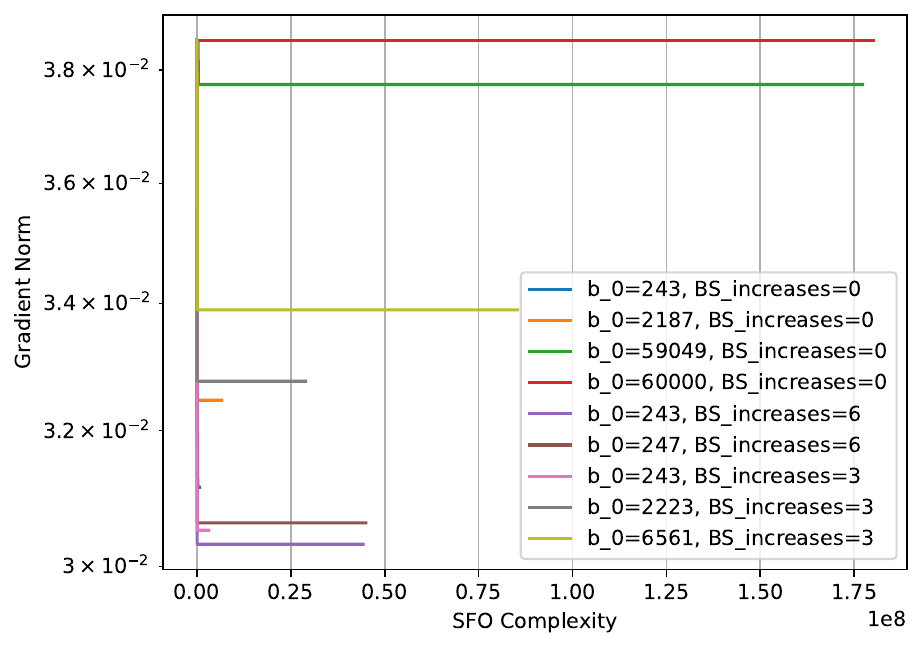}
\caption{Norm of objective function gradient versus SFO complexity. Plot on left (resp. right) shows norm of objective function gradient for our proposed problem when $(N,n)=(7200,1024)$ (resp. when $(N,n)=(60000,1024)$).}
 \label{fig:newprob_gradnorm}
\end{figure}
As shown in the plot of the norm of objective function gradient versus SFO complexity \ref{fig:newprob_gradnorm} for our proposed problem, performance was better with an increasing BS than with either a small or large constant BS. This result can be explained by Theorem \ref{thm:reduce_sfo}, which shows that an increasing BS reduces SFO complexity compared with a constant BS.

\section{Proofs of our Lemma and Theorems}\label{apdix:proofs}
Recall that $\Pi$ is a $\{1,\cdots,N\}$-valued probability distribution. Thus, we consider the probability space $(\{1,\cdots,N\},\mathfrak{P}(\{1,\cdots,N\}), \Pi)$. For a random variable $\xi$ distributed as $\Pi$ and a function $g:\{1,\cdots,N\}\times\mathcal{M}\ni(\xi,x)\mapsto g_{\xi}(x)\in\mathbb{R}$, we define $\mathbb{E}_{\xi}[g_{\xi}(x)]\coloneq\mathbb{E}_{\xi\sim\Pi}[g_{\xi}(x)]\coloneq\int_{\{1,\cdots,N\}}g_{\xi}(x)\Pi(d\xi)=\sum_{j=1}^N g_{j}(x)\Pi(\{j\})$. Hence, the variance is defined as $\mathbb{V}_{\xi\sim\Pi}(g_{\xi}(x))\coloneq\mathbb{E}_{\xi\sim\Pi}[\left\|g_{\xi}(x)-\mathbb{E}_{\xi\sim\Pi}[g_{\xi}(x)]\right\|^2]$. Let $(\xi_{i,t})_{i=1}^N$ be a sequence of i.i.d. random variables distributed as $\Pi$, let $\boldsymbol{\xi}_t\coloneq(\xi_{1,t},\cdots,\xi_{b_t,t})^{\top}\in\{1,\cdots,N\}^{b_t}$, let $g:\{1,\cdots,N\}^{b_t}\times\mathcal{M}\ni (\boldsymbol{\xi},x)\mapsto g_{\boldsymbol{\xi}}(x)\in\mathbb{R}$ be a function, and let $x_t$ be a point at the $t$-th iteration generated by the updating rule of RSGD. We can introduce a natural extension of the expectation notation to a multivariate random variable:
\begin{align*}
\mathbb{E}_{\boldsymbol{\xi}_t}[g_{\boldsymbol{\xi}_t}(x_t)]
&\coloneq\mathbb{E}_{\boldsymbol{\xi}_t\sim\Pi^{b_t}}[g_{\boldsymbol{\xi}_t}(x_t)]
\coloneq\int_{\{1,\cdots,N\}^{b_t}}g_{\boldsymbol{\xi}_t}(x_t)\Pi^{b_t}(d\boldsymbol{\xi}_t)\\
&=\sum_{j_{1}=1}^N\cdots\sum_{j_{b_t}=1}^N g_{(j_1,\cdots,j_{b_t})^{\top}}(x_t)\Pi(\{j_{1}\})\cdots\Pi(\{j_{b_t}\})\\
&=\mathbb{E}_{\xi_{1,t}}\mathbb{E}_{\xi_{2,t}}\cdots\mathbb{E}_{\xi_{b_t,t}}[g_{\boldsymbol{\xi}_t}(x_t)].
\end{align*}
For $\boldsymbol{\xi}_1, \cdots, \boldsymbol{\xi}_t$, a total expectation is defined as $\mathbb{E}\coloneq\mathbb{E}_{\boldsymbol{\xi}_1}\cdots\mathbb{E}_{\boldsymbol{\xi}_t}$. Note that, from Assumption \ref{asm:sto_gra}, $\mathbb{E}_{\xi}\left[\|\mathrm{grad} f_{\xi}(x) - \mathrm{grad} f(x)\|_{x}^2\right]=\mathbb{V}_{\xi}(\mathrm{grad} f_{\xi}(x))\leq\sigma^2$ holds, which implies $\mathbb{E}_{\boldsymbol{\xi}}\left[\|\mathrm{grad} f_{\boldsymbol{\xi}}(x) - \mathrm{grad} f(x)\|_{x}^2\right]\leq\frac{\sigma^2}{b_t}$. The following result serves as a preliminary for Lemma \ref{lem:undl_anal}.
\begin{lemma}
[Descent Lemma]\label{apdix:des_lem} Let $(x_t)_t$ be a sequence generated by RSGD and $(\eta_t)_t$ be a positive-valued sequence. Then, under Assumptions \ref{asm:rtr_smooth} and \ref{asm:sto_gra}, we obtain
\begin{align*}
\mathbb{E}[f(x_{t+1})]
\leq \mathbb{E}[f(x_t)] + \frac{L_r \sigma^2 \eta_t^2}{2b_t} - \eta_t \left(1 - \frac{L_r \eta_t}{2}\right) \mathbb{E}[\|\mathrm{grad}f(x_t)\|_{x_t}^2].
\end{align*}
\end{lemma}
\begin{proof}
Under Assumption \ref{asm:rtr_smooth}, we start with
\begin{align*}
f(x_{t+1})
&\leq f(x_t) - \eta_t \langle \mathrm{grad} f(x_t), \mathrm{grad} f_{B_t}(x_t) \rangle_{x_t} + \frac{\eta_t^2}{2}L_r \|\mathrm{grad} f_{B_t}(x_t)\|_{x_t}^2,
\end{align*}
\begin{align*}
\mathbb{E}[\|\mathrm{grad} f_{B_t}(x_t)\|_{x_t}^2 |\mathcal{F}_{t+1}]
&= \mathbb{E}[\|\mathrm{grad} f_{B_t}(x_t) - \mathrm{grad} f(x_t) + \mathrm{grad} f(x_t)\|_{x_t}^2 |\mathcal{F}_{t+1}]\\
&\leq \mathbb{E}[\|\mathrm{grad} f_{B_t}(x_t) - \mathrm{grad} f(x_t)\|_{x_t}^2 |\mathcal{F}_{t+1}] + \mathbb{E}[\|\mathrm{grad} f(x_t)\|_{x_t}^2 |\mathcal{F}_{t+1}]\\
&\quad + 2\mathbb{E}[\langle \mathrm{grad} f_{B_t}(x_t) - \mathrm{grad} f(x_t), \mathrm{grad} f(x_t) \rangle_{x_t} |\mathcal{F}_{t+1}]\\
&= \mathbb{E}[\|\mathrm{grad} f_{B_t}(x_t) - \mathrm{grad} f(x_t)\|_{x_t}^2 |\mathcal{F}_{t+1}] + \mathbb{E}[\|\mathrm{grad} f(x_t)\|_{x_t}^2 |\mathcal{F}_{t+1}]\\
&\quad + 2\langle \mathbb{E}[\mathrm{grad} f_{B_t}(x_t)|\mathcal{F}_{t+1}] - \mathrm{grad} f(x_t), \mathrm{grad} f(x_t) \rangle_{x_t}\\
&\leq \frac{\sigma^2}{b_t} + \mathbb{E}[\|\mathrm{grad} f(x_t)\|_{x_t}^2 |\mathcal{F}_{t+1}] \qquad\text{a.s.},
\end{align*}
and
\begin{align*}
\mathbb{E}[\langle \mathrm{grad} f(x_t), \mathrm{grad} f_{B_t}(x_t) \rangle_{x_t} |\mathcal{F}_{t+1}]
= \langle \mathrm{grad} f(x_t), \mathbb{E}[\mathrm{grad} f_{B_t}(x_t) |\mathcal{F}_{t+1}] \rangle_{x_t}
= \|\mathrm{grad} f(x_t)\|_{x_t}^2 \quad\text{a.s.}
\end{align*}
Taking the total expectation of the above equations, we have
\begin{align*}
\mathbb{E}[f(x_{t+1})]
&\leq \mathbb{E}[f(x_t)] - \eta_t \mathbb{E}[\langle \mathrm{grad} f(x_t), \mathrm{grad} f_{B_t}(x_t) \rangle_{x_t}] + \frac{\eta_t^2}{2}L_r \mathbb{E}[\|\mathrm{grad} f_{B_t}(x_t)\|_{x_t}^2]\\
&= \mathbb{E}[f(x_t)] - \eta_t \mathbb{E}[\mathbb{E}[\langle \mathrm{grad} f(x_t), \mathrm{grad} f_{B_t}(x_t) \rangle_{x_t} |\mathcal{F}_{t+1}]] + \frac{\eta_t^2}{2}L_r \mathbb{E} [\mathbb{E}[\|\mathrm{grad} f_{B_t}(x_t)\|_{x_t}^2 |\mathcal{F}_{t+1}]]\\
&\leq \mathbb{E}[f(x_t)] - \eta_t \mathbb{E}[\|\mathrm{grad} f(x_t)\|_{x_t}^2] + \frac{\eta_t^2 L_r}{2}\left(\frac{\sigma^2}{b_t} + \mathbb{E}[\|\mathrm{grad} f(x_t)\|_{x_t}^2]\right)\\
&= \mathbb{E}[f(x_t)] + \frac{L_r \sigma^2 \eta_t^2}{2b_t} - \eta_t \left(1 - \frac{L_r \eta_t}{2}\right) \mathbb{E}[\|\mathrm{grad}f(x_t)\|_{x_t}^2].
\end{align*}
\end{proof}

\subsection{Proof of Lemma \ref{lem:undl_anal}}
\label{prf:undl_anal}
On the basis of Lemma \ref{apdix:des_lem}, Lemma \ref{lem:undl_anal} is proven as follows.
\begin{proof}
Taking $\eta_{\max} < \frac{2}{L_r}$ into consideration, we start with
\begin{align*}
&\sum_{t=0}^{T-1} \eta_t \left(1 - \frac{L_r \eta_t}{2}\right) \mathbb{E}[\|\mathrm{grad}f(x_t)\|_{x_t}^2] 
\geq \left(1 - \frac{L_r \eta_{\max}}{2}\right) \min_{t\in\{0, \cdots, T-1\}}\mathbb{E}[\|\mathrm{grad}f(x_t)\|_{x_t}^2] \sum_{t=0}^{T-1} \eta_t.
\end{align*}
By taking the summation on both sides of the inequality for Lemma \ref{apdix:des_lem} and evaluating it using the above inequality, we obtain
\begin{align*}
&\left(1 - \frac{L_r \eta_{\max}}{2}\right) \min_{t\in\{0, \cdots, T-1\}}\mathbb{E}[\|\mathrm{grad}f(x_t)\|_{x_t}^2] \sum_{t=0}^{T-1} \eta_t\\
&\quad\leq \sum_{t=0}^{T-1} \eta_t \left(1 - \frac{L_r \eta_t}{2}\right) \mathbb{E}[\|\mathrm{grad}f(x_t)\|_{x_t}^2]\\
&\quad \leq \sum_{t=0}^{T-1} (\mathbb{E}[f(x_t)] - \mathbb{E}[f(x_{t+1})]) + \sum_{t=0}^{T-1}\frac{L_r \sigma^2 \eta_t^2}{2b_t}
= \mathbb{E}[f(x_0) - f(x_T)] + \sum_{t=0}^{T-1}\frac{L_r \sigma^2 \eta_t^2}{2b_t}.
\end{align*}
$(\mathbb{E}[f(x_t)])_{t\in\{0, \cdots, T\}}$ being a decreasing sequence implies
\begin{align*}
\mathbb{E}[f(x_0) - f(x_T)]
\leq \mathbb{E}[f(x_0) - f^{\star}]
= f(x_0) - f^{\star}.
\end{align*}
Therefore, 
\begin{align*}
\min_{t\in\{0, \cdots, T-1\}}\mathbb{E}[\|\mathrm{grad}f(x_t)\|_{x_t}^2]
&\leq \frac{2(f(x_0) - f^{\star})}{2 - L_r \eta_{\max}} \frac{1}{\sum_{t=0}^{T-1} \eta_t} + \frac{L_r \sigma^2}{2 - L_r \eta_{\max}} \frac{1}{\sum_{t=0}^{T-1} \eta_t}\sum_{t=0}^{T-1} \frac{\eta_t^2}{b_t}.
\end{align*}
\end{proof}

\subsection{Proof of Theorem \ref{thm:c_d}}
\label{prf:c_d}
\begin{proof}
We set $b_t = b$ in Lemma \ref{lem:undl_anal}.\newline
\textbf{[Constant LR \eqref{eq:const_lr}]}\newline
From Lemma \ref{lem:undl_anal}, we have
\begin{align*}
\sum_{t=0}^{T-1} \eta_t
= \sum_{t=0}^{T-1} \eta_{\max}
= \eta_{\max} T,
\quad \sum_{t=0}^{T-1} \frac{\eta_t^2}{b_t}
= \frac{1}{\sum_{t=0}^{T-1}\eta_{\max}} \sum_{t=0}^{T-1} \frac{\eta_{\max}^2}{b}
= \frac{\eta_{\max}^2}{b}T,
\end{align*}
which yields
\begin{align*}
\min_{t\in\{0, \cdots, T-1\}}\mathbb{E}[\|\mathrm{grad}f(x_t)\|_{x_t}^2]
\leq \frac{2(f(x_0) - f^{\star})}{(2 - L_r \eta_{\max})\eta_{\max} T} + \frac{L_r \eta_{\max}}{2 - L_r \eta_{\max}} \frac{\sigma^2}{b}.
\end{align*}
\textbf{[Diminishing LR \eqref{eq:dim_lr}]}\newline
From Lemma \ref{lem:undl_anal}, we have
\begin{align*}
&\sum_{t=0}^{T-1} \eta_t
= \eta_{\max} \sum_{t=0}^{T-1} \frac{1}{\sqrt{t+1}}
\geq \eta_{\max} \sum_{t=0}^{T-1} \frac{1}{\sqrt{T}}
= \eta_{\max} \frac{T}{\sqrt{T}}
= \eta_{\max} \sqrt{T}
\end{align*}
and
\begin{align*}
\sum_{t=0}^{T-1} \frac{\eta_t^2}{b_t}
= \frac{\eta_{\max}^2}{b}\sum_{t=0}^{T-1} \frac{1}{t+1}
\leq \frac{\eta_{\max}^2}{b}\left(1 + \int_{1}^T \frac{dt}{t}\right)
= \frac{\eta_{\max}^2}{b} + \frac{\eta_{\max}^2}{b} \log T,
\end{align*}
which yields
\begin{align*}
\min_{t\in\{0, \cdots, T-1\}}\mathbb{E}[\|\mathrm{grad}f(x_t)\|_{x_t}^2]
&\leq \frac{f(x_0) - f^{\star}}{(2 - L_r \eta_{\max})\eta_{\max}} \frac{1}{\sqrt{T}} + \frac{L_r \eta_{\max}}{2 - L_r \eta_{\max}} \frac{\sigma^2}{b} \frac{1+\log T}{\sqrt{T}}.
\end{align*}
\textbf{[Cosine Annealing LR \eqref{eq:cosan_lr}]}\newline
From Lemma \ref{lem:undl_anal}, we have
\begin{align*}
\sum_{t=0}^{T-1} \eta_t
&\geq \int_{0}^T \eta_t dt
=\frac{\eta_{\max} + \eta_{\min}}{2} T + \frac{\eta_{\max} - \eta_{\min}}{2} \int_{0}^{T} \cos \left(\frac{t}{T}\pi\right) dt\\
&= \frac{\eta_{\max} + \eta_{\min}}{2} T
\end{align*}
and
\begin{align*}
\sum_{t=0}^{T-1} \eta_t^2
&= \frac{(\eta_{\max} + \eta_{\min})^2}{4}T + \frac{\eta_{\max}^2 + \eta_{\min}^2}{2}\sum_{t=0}^{T-1} \cos \frac{t}{T}\pi + \frac{(\eta_{\max} - \eta_{\min})^2}{4}\sum_{t=0}^{T-1} \cos^2 \frac{t}{T}\pi\\
&\leq \frac{(\eta_{\max} + \eta_{\min})^2}{4}T + \frac{\eta_{\max}^2 - \eta_{\min}^2}{2}T + \frac{(\eta_{\max} - \eta_{\min})^2}{4}T\\
&= \eta_{\max}^2 T,
\end{align*}
which yields
\begin{align*}
\min_{t\in\{0, \cdots, T-1\}}\mathbb{E}[\|\mathrm{grad}f(x_t)\|_{x_t}^2]
&\leq \frac{4(f(x_0) - f^{\star})}{(2 - L_r \eta_{\max})(\eta_{\max} + \eta_{\min})} \frac{1}{T} + \frac{2L_r \eta_{\max}^2}{(2 - L_r \eta_{\max})(\eta_{\max} + \eta_{\min})} \frac{\sigma^2}{b}.
\end{align*}
\textbf{[Polynomial Decay LR \eqref{eq:poly_dec_lr}]}\newline
We start by recalling the Riemann integral of $g(t)\coloneq(1-t)^p \geq 0$ $(0\leq t\leq1)$. $U(P_T)\coloneq\sum_{t=0}^{T-1} \frac{1}{T}(1 - \frac{t}{T})^p$ is an upper Riemann sum for the function $g$ over the partition $P_T \coloneq\{0, \frac{1}{T}, \cdots, \frac{T-1}{T}, 1\}$ due to $g$ being monotonically decreasing. From the definition of a Riemann integral, if $\int_0^1 g(t)dt$ exists, then
\begin{align*}
\inf_{P:\text{Partitions of [0,1]}} U(P) = \int_0^1 g(t)dt.
\end{align*}
In fact, the integral value exists. Thus, we obtain
\begin{align*}
\sum_{t=0}^{T-1} \frac{1}{T}\left(1 - \frac{t}{T}\right)^p
= U(P_T)
\geq \inf_{P:\text{Partitions of [0,1]}} U(P)
=\int_0^1 (1-t)^p dt
= \frac{1}{p+1}.
\end{align*}
Similarly, if we define $L(P_T)$ as $\sum_{t=1}^{T} \frac{1}{T}\left(1 - \frac{t}{T}\right)^p$, then $L(P_T)$ is a lower Riemann sum for the function $g$ over the partition $P_T$, and the supremum of $L(P_T)$ equals $\int_0^1 g(t)dt$. Considering
\begin{align*}
\sum_{t=0}^{T-1} \frac{1}{T}\left(1 - \frac{t}{T}\right)^p
= L(P_T) + \frac{1}{T},
\end{align*}
we have
\begin{align*}
\sum_{t=0}^{T-1} \frac{1}{T}\left(1 - \frac{t}{T}\right)^p
= L(P_T) + \frac{1}{T}
\leq \sup_{P:\text{Partitions of [0,1]}} L(P) + \frac{1}{T}
= \int_0^1 (1-t)^p dt + \frac{1}{T}
= \frac{1}{p+1} + \frac{1}{T}.
\end{align*}
Hence,
\begin{align*}
\frac{T}{p+1}
\leq \sum_{t=0}^{T-1}\left(1 - \frac{t}{T}\right)^p
\leq \frac{T}{p+1} + 1
\end{align*}
holds. By replacing $g$ with $(1+t)^{2p}(0\leq t\leq 1)$ and using the same logic, we obtain
\begin{align*}
\sum_{t=0}^{T-1} \frac{1}{T}\left(1 - \frac{t}{T}\right)^{2p}
\leq \int_0^1 (1-t)^{2p} dt + \frac{1}{T}
= \frac{1}{2p+1} + \frac{1}{T}.
\end{align*}
Therefore, considering Lemma \ref{lem:undl_anal},
\begin{align*}
\sum_{t=0}^{T-1} \eta_t
&= \eta_{\min} T + (\eta_{\max} - \eta_{\min})\sum_{t=0}^{T-1} \left(1 - \frac{t}{T}\right)^p\\
&\geq \eta_{\min} T + (\eta_{\max} - \eta_{\min}) \frac{T}{p+1}
= \frac{\eta_{\max} + p\eta_{\min}}{p+1}T
\end{align*}
and
\begin{align*}
\sum_{t=0}^{T-1} \eta_t^2
&= \eta_{\min} T + 2\eta_{\min}(\eta_{\max}- \eta_{\min}) \sum_{t=0}^{T-1} \left(1 - \frac{t}{T}\right)^p + (\eta_{\max}- \eta_{\min})^2 \sum_{t=0}^{T-1} \left(1 - \frac{t}{T}\right)^{2p}\\
&\leq \eta_{\min} T + 2\eta_{\min}(\eta_{\max}- \eta_{\min})\left(1 + \frac{T}{p+1}\right) + (\eta_{\max}- \eta_{\min})^2 \left(1 + \frac{T}{2p+1}\right)\\
&= \eta_{\max}^2 - \eta_{\min}^2 + \left(\eta_{\min} + \frac{2\eta_{\min}(\eta_{\max}- \eta_{\min})}{p+1} + \frac{(\eta_{\max}- \eta_{\min})^2}{2p+1}\right)T
\end{align*}
hold, which yields
\begin{align*}
&\min_{t\in\{0, \cdots, T-1\}}\mathbb{E}[\|\mathrm{grad}f(x_t)\|_{x_t}^2]\\
&\leq \frac{p+1}{(2 - L_r \eta_{\max})(\eta_{\max} + p\eta_{\min})}\left\{2(f(x_0) - f^{\star}) + \frac{L_r(\eta_{\max}^2 - \eta_{\min}^2)\sigma^2}{b}\right\} \frac{1}{T}\\
&\quad + \frac{L_r (p+1)(\eta_{\max}^2 - \eta_{\min}^2)}{(2 - L_r \eta_{\max})(\eta_{\max} + p\eta_{\min})}\left(\eta_{\min} + \frac{2\eta_{\min}(\eta_{\max}- \eta_{\min})}{p+1} + \frac{(\eta_{\max}- \eta_{\min})^2}{2p+1}\right) \frac{\sigma^2}{b}.
\end{align*}
\end{proof}

\subsection{Proof of Theorem \ref{thm:i_d}}
\label{prf:i_d}
\begin{proof}
The evaluations in all the following cases are based on Lemma \ref{lem:undl_anal} and use estimations of $\sum_{t=0}^{T-1} \eta_t$ in the proof of Theorem \ref{thm:c_d}.\newline
\textbf{[Exponential Growth BS \eqref{eq:exp_g_bs} and Constant LR \eqref{eq:const_lr}]}\newline
From the fact that sums of positive term series are the supremum of their finite sums, we have 
\begin{align}\label{eq:esti_exp_bs}
\sum_{t=0}^{T-1} \frac{1}{b_t}
= \sum_{m=0}^{M-1} \frac{K}{b_0 \gamma^m} + \frac{T - KM}{b_0 \gamma^M}
\leq \sum_{m=0}^{M} \frac{K}{b_0 \gamma^m}
\leq \frac{K}{b_0}\sum_{m=0}^{\infty} \frac{1}{\gamma^m}
= \frac{K\gamma}{b_0(\gamma -1)},
\end{align}
which implies
\begin{align*}
\sum_{t=0}^{T-1} \frac{\eta_{\max}^2}{b_t}
= \eta_{\max}^2 \sum_{t=0}^{T-1} \frac{1}{b_t}
\leq \frac{\eta_{\max}^2 K\gamma}{b_0 (\gamma - 1)}.
\end{align*}
Therefore,
\begin{align*}
\min_{t\in\{0, \cdots, T-1\}}\mathbb{E}[\|\mathrm{grad}f(x_t)\|_{x_t}^2]
\leq \frac{1}{(2 - L_r \eta_{\max})\eta_{\max}}\left\{2(f(x_0) - f^{\star}) + \frac{L_r \eta_{\max}^2 K\gamma}{\gamma - 1}\frac{\sigma^2}{b_0}\right\}\frac{1}{T}.
\end{align*}
\textbf{[Exponential Growth BS \eqref{eq:exp_g_bs} and Diminishing LR \eqref{eq:dim_lr}]}\newline
Using \eqref{eq:esti_exp_bs}, we have
\begin{align*}
\sum_{t=0}^{T-1} \frac{\eta_t^2}{b_t}
= \eta_{\max}^2 \sum_{t=0}^{T-1}\frac{1}{(t+1)b_t}
\leq \eta_{\max}^2 \sum_{t=0}^{T-1}\frac{1}{b_t}
\leq \frac{\eta_{\max}^2 K\gamma}{b_0 (\gamma - 1)},
\end{align*}
which yields
\begin{align*}
\min_{t\in\{0, \cdots, T-1\}}\mathbb{E}[\|\mathrm{grad}f(x_t)\|_{x_t}^2]
\leq \frac{1}{(2 - L_r \eta_{\max})\eta_{\max}}\left\{2(f(x_0) - f^{\star}) + \frac{L_r \eta_{\max}^2 K \gamma}{\gamma - 1}\frac{\sigma^2}{b_0}\right\}\frac{1}{\sqrt{T}}.
\end{align*}
\textbf{[Exponential Growth BS \eqref{eq:exp_g_bs} and Cosine Annealing LR \eqref{eq:cosan_lr}]}\newline
From \eqref{eq:esti_exp_bs} and $|\cos x|\leq 1$, we have
\begin{align*}
\sum_{t=0}^{T-1} \frac{\eta_t^2}{b_t}
&= \sum_{t=0}^{T-1} \frac{1}{b_t}\left(\frac{(\eta_{\max}+\eta_{\min})^2}{4} + \frac{(\eta_{\max}-\eta_{\min})^2}{4}\cos^2 \frac{t}{T}\pi + \frac{\eta_{\max}^2 - \eta_{\min}^2}{2}\cos\frac{t}{T}\pi \right)\\
&\leq \eta_{\max}^2 \sum_{t=0}^{T-1} \frac{1}{b_t}
\leq \frac{\eta_{\max}^2 K\gamma}{b_0(\gamma - 1)},
\end{align*}
which yields
\begin{align*}
\min_{t\in\{0, \cdots, T-1\}}\mathbb{E}[\|\mathrm{grad}f(x_t)\|_{x_t}^2]
\leq \frac{2}{(2 - L_r \eta_{\max})(\eta_{\max} +\eta_{\min})} \left\{2(f(x_0) - f^{\star}) + \frac{L_r \eta_{\max}^2 K\gamma}{\gamma -1 } \frac{\sigma^2}{b_0}\right\}\frac{1}{T}.
\end{align*}
\textbf{[Exponential Growth BS \eqref{eq:exp_g_bs} and Polynomial Decay LR \eqref{eq:poly_dec_lr}]}\newline
From \eqref{eq:esti_exp_bs} and $t\leq T$, we have
\begin{align*}
\sum_{t=0}^{T-1} \frac{\eta_t^2}{b_t}
&= \sum_{t=0}^{T-1} \frac{1}{b_t} \left\{\eta_{\min}^2 + (\eta_{\max} - \eta_{\min})^2 \left(1 - \frac{t}{T}\right)^{2p} + 2\eta_{\min}(\eta_{\max} - \eta_{\min})\left(1 - \frac{t}{T}\right)^p\right\}\\
&\leq \eta_{\max}^2 \sum_{t=0}^{T-1} \frac{1}{b_t}
\leq \frac{\eta_{\max}^2 K \gamma}{b_0(\gamma - 1)},
\end{align*}
which yields
\begin{align*}
\min_{t\in\{0, \cdots, T-1\}}\mathbb{E}[\|\mathrm{grad}f(x_t)\|_{x_t}^2]
&\leq \frac{p + 1}{(2 - L_r \eta_{\max})(\eta_{\max} + p\eta_{\min})} \left\{2(f(x_0) - f^{\star}) + \frac{L_r \eta_{\max}^2 K \gamma}{\gamma - 1}\frac{\sigma^2}{b_0}\right\}\frac{1}{T}.
\end{align*}
\textbf{[Polynomial Growth BS \eqref{eq:pol_g_bs} and Constant LR \eqref{eq:const_lr}]}\newline
We set $\underline{a} \coloneq a \land b_0$. Considering $c>1$, we have
\begin{align}\label{eq:esti_pol_bs}
\sum_{t=0}^{T-1} \frac{1}{b_t}
&= \sum_{m=0}^{M-1} \frac{K}{(am+b_0)^c} + \frac{T - KM}{(aM + b_0)^c}
\leq \sum_{m=0}^{M} \frac{K}{(am+b_0)^c}\nonumber
\leq \frac{K}{\underline{a}^{\lfloor c\rfloor}}\sum_{m=0}^{M} \frac{1}{(m + 1)^c}\nonumber\\
&= \frac{K}{\underline{a}^{\lfloor c\rfloor}}\sum_{m=1}^{\infty} \frac{1}{m^c}
= \frac{K \zeta(c)}{\underline{a}^{\lfloor c\rfloor}},
\end{align}
which yields
\begin{align*}
\sum_{t=0}^{T-1} \frac{\eta_t^2}{b_t}
= \eta_{\max}^2 \sum_{t=0}^{T-1} \frac{1}{b_t}
\leq \frac{\eta_{\max}^2K \zeta(c)}{\underline{a}^{\lfloor c\rfloor}}.
\end{align*}
Therefore,
\begin{align*}
\min_{t\in\{0, \cdots, T-1\}}\mathbb{E}[\|\mathrm{grad}f(x_t)\|_{x_t}^2]
\leq \frac{1}{(2 - L_r \eta_{\max})\eta_{\max}} \left(2(f(x_0) - f^{\star}) + \frac{L_r\eta_{\max}^2K \zeta(c) \sigma^2}{\underline{a}^{\lfloor c\rfloor}} \right)\frac{1}{T},
\end{align*}
where, for $c>1$, Riemann zeta function $\zeta(c) < \infty$; $\zeta(c)$ is monotonically decreasing on $c\in (1,\infty)$ and $\underset{c\to\infty}{\lim}\zeta(c)=1$.\newline
\textbf{[Polynomial Growth BS \eqref{eq:pol_g_bs} and Diminishing LR \eqref{eq:dim_lr}]}\newline
From \eqref{eq:esti_pol_bs}, we have
\begin{align*}
\sum_{t=0}^{T-1} \frac{\eta_t^2}{b_t}
= \eta_{\max}^2 \sum_{t=0}^{T-1} \frac{1}{(t+1)b_t}
\leq \eta_{\max}^2 \sum_{t=0}^{T-1} \frac{1}{b_t}
\leq \frac{\eta_{\max}^2K \zeta(c)}{\underline{a}^{\lfloor c\rfloor}},
\end{align*}
which yields
\begin{align*}
\min_{t\in\{0, \cdots, T-1\}}\mathbb{E}[\|\mathrm{grad}f(x_t)\|_{x_t}^2]
\leq \frac{1}{(2 - L_r \eta_{\max})\eta_{\max}} \left(2(f(x_0) - f^{\star}) + \frac{L_r\eta_{\max}^2K \zeta(c) \sigma^2}{\underline{a}^{\lfloor c\rfloor}} \right)\frac{1}{\sqrt{T}}.
\end{align*}
\textbf{[Polynomial Growth BS \eqref{eq:pol_g_bs} and Cosine Annealing LR \eqref{eq:cosan_lr}]}\newline
From \eqref{eq:esti_pol_bs} and the previous analysis for Cosine Annealing LR \eqref{eq:cosan_lr}, we have
\begin{align*}
\sum_{t=0}^{T-1} \frac{\eta_t^2}{b_t}
\leq \eta_{\max}^2 \sum_{t=0}^{T-1} \frac{1}{b_t}
\leq \frac{\eta_{\max}^2K \zeta(c)}{\underline{a}^{\lfloor c\rfloor}},
\end{align*}
which yields
\begin{align*}
\min_{t\in\{0, \cdots, T-1\}}\mathbb{E}[\|\mathrm{grad}f(x_t)\|_{x_t}^2]
\quad \leq \frac{2}{(2 - L_r \eta_{\max})(\eta_{\max} + \eta_{\min})} \left(2(f(x_0) - f^{\star}) + \frac{L_r\eta_{\max}^2K \zeta(c) \sigma^2}{\underline{a}^{\lfloor c\rfloor}}\right) \frac{1}{T}.
\end{align*}
\textbf{[Polynomial Growth BS \eqref{eq:pol_g_bs} and Polynomial Decay LR \eqref{eq:poly_dec_lr}]}\newline
From \eqref{eq:esti_pol_bs} and the previous analysis for Polynomial Decay LR \eqref{eq:poly_dec_lr}, we have
\begin{align*}
\sum_{t=0}^{T-1} \frac{\eta_t^2}{b_t}
\leq \eta_{\max}^2 \sum_{t=0}^{T-1} \frac{1}{b_t}
\leq \frac{\eta_{\max}^2K \zeta(c)}{\underline{a}^{\lfloor c\rfloor}},
\end{align*}
which yields
\begin{align*}
\min_{t\in\{0, \cdots, T-1\}}\mathbb{E}[\|\mathrm{grad}f(x_t)\|_{x_t}^2]
\quad \leq \frac{p+1}{(2 - L_r \eta_{\max})(\eta_{\max}+p\eta_{\min})} \left(2(f(x_0) - f^{\star}) + \frac{L_r\eta_{\max}^2K \zeta(c) \sigma^2}{\underline{a}^{\lfloor c\rfloor}}\right)\frac{1}{T}.
\end{align*}
\end{proof}

\subsection{Proof of Theorem \ref{thm:i_w}}\label{prf:i_w}
\begin{theorem}
[Detailed Version of Theorem \ref{thm:i_w}] \label{thm:i_w_detailed} We consider BSs \eqref{eq:exp_g_bs} and \eqref{eq:pol_g_bs} and warm-up LRs \eqref{eq:exp_g_lr} and \eqref{eq:pol_g_lr} with decay parts given by \eqref{eq:const_lr}, \eqref{eq:dim_lr}, \eqref{eq:cosan_lr}, or \eqref{eq:poly_dec_lr} under the assumptions of Lemma \ref{lem:undl_anal}. Then, the following holds for both constant and increasing BSs.
\begin{itemize}
\item Decay part: Diminishing \eqref{eq:dim_lr}
 \begin{align*} 
 \min_{t\in\{T_w, \cdots, T-1\}}\mathbb{E}[\|\mathrm{grad}f(x_t)\|_{x_t}^2]
 &\leq \frac{Q_1 +Q_2\sigma^2 b_0^{-1}}{\sqrt{T+1}-\sqrt{T_w +1}},
 \end{align*}
\item Decay part: Otherwise \eqref{eq:const_lr}, \eqref{eq:cosan_lr}, \eqref{eq:poly_dec_lr}
 \begin{align*}
 \min_{t\in\{T_w, \cdots, T-1\}}\mathbb{E}[\|\mathrm{grad}f(x_t)\|_{x_t}^2]
 &\leq \frac{\tilde{Q}_1 + \tilde{Q}_2\sigma^2 b_0^{-1}}{T-T_w},
 \end{align*}
\end{itemize}
where $Q_1, Q_2, \tilde{Q}_1$, and $\tilde{Q}_2$ are constants that do not depend on $T$.
\end{theorem}
\begin{proof}
The evaluations in all these cases are based on Lemma \ref{lem:undl_anal}. We start with an exponential growth BS and a warm-up LR. From $l, l_w \in \mathbb{N}$, and $l_w \geq l$, there exist $\alpha, \beta \in \mathbb{N}\cup\{0\}$ such that $l_w = \alpha l + \beta$. Note that we define summations from $0$ to $-1$ as $0$. Furthermore, from $\forall u,v>0: 0\leq u^2+v^2+uv$ holds.\newline
\textbf{[LR Decay Part: Constant \eqref{eq:const_lr}]}\newline
Considering $T_w = l_w K' = \alpha (lK') + \beta K'$, we have
\begin{align*}
\sum_{t=T_w}^{T-1} \eta_t
&= \sum_{t=T_w}^{T-1} \eta_{\max}
= (T - T_w)\eta_{\max}
\end{align*}
and
\begin{align*}
\sum_{t=T_w}^{T-1} \frac{\eta_t^2}{b_t}
= \sum_{t=l_w K'}^{(\alpha + 1)l-1} \frac{\eta_{\max}^2}{b_t} + \sum_{t=(\alpha+1)l}^{T-1}\frac{\eta_{\max}^2}{b_t}
= \sum_{t=T_w}^{T-1} \frac{\eta_{\max}^2}{b_t}
\leq \frac{\eta_{\max}^2 K}{b_0}\sum_{m=0}^{M} \frac{1}{\gamma^m}
\leq \frac{\eta_{\max}^2 K}{b_0}\sum_{m=0}^{\infty} \frac{1}{\gamma^m}
= \frac{\eta_{\max}^2 \gamma K}{(\gamma -1)b_0},
\end{align*}
which yields
\begin{align*}
\min_{t\in\{T_w, \cdots, T-1\}}\mathbb{E}[\|\mathrm{grad}f(x_t)\|_{x_t}^2]
\leq \frac{1}{(2 - L_r \eta_{\max})\eta_{\max}}\left(2(f(x_0) - f^{\star}) + \frac{L_r \eta_{\max}^2 K\gamma}{\gamma -1}\frac{\sigma^2}{b_0}\right)\frac{1}{T-T_w}.
\end{align*}
\textbf{[LR Decay Part: Diminishing \eqref{eq:dim_lr}]}\newline
Similarly, we have
\begin{align*}
\sum_{t=T_w}^{T-1} \eta_t
&= \sum_{t=T_w}^{T-1} \frac{\eta_{\max}}{\sqrt{t+1}}
\geq \eta_{\max} \int_{T_w}^T \frac{dt}{\sqrt{t+1}}
= 2\eta_{\max} (\sqrt{T+1}-\sqrt{T_w +1})
\end{align*}
and
\begin{align*}
\sum_{t=0}^{T-1} \frac{\eta_t^2}{b_t}
\leq \sum_{t=T_w}^{T-1} \frac{\eta_{\max}^2}{b_t(t+1)}
\leq \sum_{t=T_w}^{T-1} \frac{\eta_{\max}^2}{b_t}
\leq \frac{\eta_{\max}^2 \gamma K}{(\gamma -1)b_0},
\end{align*}
which yields
\small
\begin{align*}
\min_{t\in\{T_w, \cdots, T-1\}}\mathbb{E}[\|\mathrm{grad}f(x_t)\|_{x_t}^2]
\leq \frac{1}{2(2 - L_r \eta_{\max})\eta_{\max}}\left(2(f(x_0) - f^{\star}) + \frac{L_r \eta_{\max}^2 K\gamma}{\gamma -1}\frac{\sigma^2}{b_0}\right)\frac{1}{\sqrt{T+1}-\sqrt{T_w+1}}.
\end{align*}
\normalsize
\textbf{[LR Decay Part: Cosine Annealing \eqref{eq:cosan_lr}]}\newline
We have
\begin{align*}
\sum_{t=T_w}^{T-1} \eta_t
&= \sum_{t=T_w}^{T-1}\left(\frac{\eta_{\max} +\eta_{\min}}{2} + \frac{\eta_{\max} -\eta_{\min}}{2}\cos\frac{t-T_w}{T-T_w}\pi\right)\\
&= \frac{\eta_{\max} +\eta_{\min}}{2}(T -T_w) + \frac{\eta_{\max} -\eta_{\min}}{2}\sum_{t=T_w}^{T-1}\cos\frac{t-T_w}{T-T_w}\pi\\
&\geq \frac{\eta_{\max} +\eta_{\min}}{2}(T -T_w) + \frac{\eta_{\max} -\eta_{\min}}{2}\int_{T_w}^{T}\cos\left(\frac{t-T_w}{T-T_w}\pi\right) dt
= \frac{\eta_{\max} +\eta_{\min}}{2}(T -T_w)
\end{align*}
and
\begin{align*}
\sum_{t=T_w}^{T-1} \frac{\eta_t^2}{b_t}
&\leq \sum_{t=T_w}^{T-1} \frac{1}{b_t}\left(\frac{(\eta_{\max} + \eta_{\min})^2}{4} + \frac{\eta_{\max}^2 - \eta_{\min}^2}{2}\cos\frac{t-T_w}{T-T_w}\pi + \frac{(\eta_{\max} - \eta_{\min})^2}{4}\cos^2 \frac{t-T_w}{T-T_w}\pi\right)\\
&\leq \sum_{t=0}^{T-1} \frac{1}{b_t}\left(\frac{(\eta_{\max} + \eta_{\min})^2}{4} + \frac{\eta_{\max}^2 - \eta_{\min}^2}{2} + \frac{(\eta_{\max} - \eta_{\min})^2}{4}\right)
= \sum_{t=0}^{T-1}\frac{\eta_{\max}^2}{b_t}
\leq \frac{\eta_{\max}^2 \gamma K}{(\gamma -1)b_0},
\end{align*}
which yields
\begin{align*}
\min_{t\in\{T_w, \cdots, T-1\}}\mathbb{E}[\|\mathrm{grad}f(x_t)\|_{x_t}^2]
\leq \frac{2}{(2 - L_r \eta_{\max})(\eta_{\max} + \eta_{\min})}\left(2(f(x_0) - f^{\star}) + \frac{L_r \eta_{\max}^2 K\gamma}{\gamma -1}\frac{\sigma^2}{b_0}\right)\frac{1}{T-T_w}.
\end{align*}
\textbf{[LR Decay Part: Polynomial Decay \eqref{eq:poly_dec_lr}]}\newline
As with the proof for a polynomial decay LR \eqref{eq:poly_dec_lr} in Theorem \ref{thm:c_d}, we have
\begin{align*}
\sum_{t=T_w}^{T-1} \eta_t
&= \sum_{t=T_w}^{T-1}\left\{\eta_{\min} + (\eta_{\max} -\eta_{\min})\left(1 - \frac{t -T_w}{T -T_w}\right)^p \right\}\\
&\geq \eta_{\min}(T-T_w) + (\eta_{\max} -\eta_{\min})\frac{T-T_w}{p+1}
= \frac{\eta_{\max} + p\eta_{\min}}{p+1}(T -T_w)
\end{align*}
and
\begin{align*}
\sum_{t=T_w}^{T-1} \frac{\eta_t^2}{b_t}
&\leq \sum_{t=T_w}^{T-1} \frac{1}{b_t}\left(\eta_{\min}^2 + (\eta_{\max} - \eta_{\min})^2\left(1 - \frac{t -T_w}{T-T_w}\right)^{2p} + 2\eta_{\min}(\eta_{\max} - \eta_{\min})\left(1 - \frac{t -T_w}{T-T_w}\right)^{p}\right)\\
&\leq \sum_{t=T_w}^{T-1}\frac{\eta_{\max}^2}{b_t}
\leq \frac{\eta_{\max}^2 \gamma K}{(\gamma -1)b_0},
\end{align*}
which yields
\begin{align*}
\min_{t\in\{T_w, \cdots, T-1\}}\mathbb{E}[\|\mathrm{grad}f(x_t)\|_{x_t}^2]
\leq \frac{p+1}{(2 - L_r \eta_{\max})(\eta_{\max} +p\eta_{\min})}\left(2(f(x_0) - f^{\star}) + \frac{L_r \eta_{\max}^2 K\gamma}{\gamma -1}\frac{\sigma^2}{b_0}\right)\frac{1}{T-T_w}.
\end{align*}
Next we consider a polynomial growth BS and a warm-up LR.\newline
\textbf{[LR Decay Part: Constant \eqref{eq:const_lr}]}\newline
We have
\begin{align*}
\sum_{t=T_w}^{T-1} \eta_t
= \sum_{t=T_w}^{T-1} \eta_{\max}
= \eta_{\max}(T - T_w)
\end{align*}
and
\small
\begin{align*}
\sum_{t=T_w}^{T-1}\frac{\eta_t^2}{b_t}
&= \sum_{t=l_w K'}^{(\alpha +1)l -1}\frac{\eta_{\max}^2}{b_t} + \sum_{(\alpha +1)l}^{T-1}\frac{\eta_{\max}^2}{b_t}
= \sum_{t=T_w}^{T-1}\frac{\eta_{\max}^2}{b_t}
\leq \frac{\eta_{\max}^2 K}{\underline{a}^{\lfloor c\rfloor}}\sum_{m=0}^{M-1}\frac{1}{(m+1)^c}
\leq \frac{\eta_{\max}^2 K}{\underline{a}^{\lfloor c\rfloor}}\sum_{m=1}^{\infty}\frac{1}{m^c}
= \frac{\eta_{\max}^2 K \zeta(c)}{\underline{a}^{\lfloor c\rfloor}},
\end{align*}
\normalsize
which yields
\begin{align*}
\min_{t\in\{T_w, \cdots, T-1\}}\mathbb{E}[\|\mathrm{grad}f(x_t)\|_{x_t}^2]
\leq \frac{1}{(2 - L_r \eta_{\max})\eta_{\max}}\left(2(f(x_0) - f^{\star}) + \frac{L_r \eta_{\max}^2 K\zeta(c)\sigma^2}{\underline{a}^{\lfloor c\rfloor}}\right)\frac{1}{T-T_w}.
\end{align*}
\textbf{[LR Decay Part: Diminishing \eqref{eq:dim_lr}]}\newline
Similarly, we have
\begin{align*}
\sum_{t=T_w}^{T-1}\eta_t
= \sum_{t=T_w}^{T-1} \frac{\eta_{\max}}{\sqrt{t+1}}
\geq \eta_{\max} \int_{T_w}^T \frac{dt}{\sqrt{t+1}}
= 2\eta_{\max} (\sqrt{T+1} - \sqrt{T_w +1})
\end{align*}
and
\begin{align*}
\sum_{t=T_w}^{T-1}\frac{\eta_t^2}{b_t}
= \sum_{t=T_w}^{T-1}\frac{\eta_{\max}^2}{b_t(t+1)}
\leq \sum_{t=T_w}^{T-1}\frac{\eta_{\max}^2}{b_t}
\leq \frac{\eta_{\max}^2 K \zeta(c)}{\underline{a}^{\lfloor c\rfloor}},
\end{align*}
which yields
\small
\begin{align*}
\min_{t\in\{T_w, \cdots, T-1\}}\mathbb{E}[\|\mathrm{grad}f(x_t)\|_{x_t}^2]
\leq \frac{1}{2(2 - L_r \eta_{\max})\eta_{\max}}\left(2(f(x_0) - f^{\star}) + \frac{L_r \eta_{\max}^2 K\zeta(c)\sigma^2}{\underline{a}^{\lfloor c\rfloor}}\right)\frac{1}{\sqrt{T+1}-\sqrt{T_w+1}}.
\end{align*}
\normalsize
\textbf{[LR Decay Part: Cosine Annealing \eqref{eq:cosan_lr}]}\newline
We have
\begin{align*}
\sum_{t=T_w}^{T-1}\eta_t
&= \sum_{t=T_w}^{T-1} \left(\frac{\eta_{\max} + \eta_{\min}}{2} + \frac{\eta_{\max} - \eta_{\min}}{2}\cos \frac{t-T_w}{T-T_w}\pi \right)\\
&\geq \frac{\eta_{\max} + \eta_{\min}}{2}(T -T_w) + \frac{\eta_{\max} - \eta_{\min}}{2} \int_{T_w}^T \cos \left(\frac{t-T_w}{T-T_w}\pi \right)dt
= \frac{\eta_{\max} + \eta_{\min}}{2}(T -T_w)
\end{align*}
and
\begin{align*}
\sum_{t=T_w}^{T-1} \frac{\eta_t^2}{b_t}
&= \sum_{t=T_w}^{T-1} \frac{1}{b_t}\left(\frac{(\eta_{\max} + \eta_{\min})^2}{4} + \frac{\eta_{\max}^2 - \eta_{\min}^2}{2}\cos\frac{t-T_w}{T-T_w}\pi + \frac{(\eta_{\max} - \eta_{\min})^2}{4}\cos^2 \frac{t-T_w}{T-T_w}\pi\right)\\
&\leq \sum_{t=T_w}^{T-1} \frac{1}{b_t}\left(\frac{(\eta_{\max} + \eta_{\min})^2}{4} + \frac{\eta_{\max}^2 - \eta_{\min}^2}{2} + \frac{(\eta_{\max} - \eta_{\min})^2}{4}\right)
= \sum_{t=T_w}^{T-1} \frac{\eta_{\max}^2}{b_t}
\leq \frac{\eta_{\max}^2 K \zeta(c)}{\underline{a}^{\lfloor c\rfloor}},
\end{align*}
which yields
\small
\begin{align*}
\min_{t\in\{T_w, \cdots, T-1\}}\mathbb{E}[\|\mathrm{grad}f(x_t)\|_{x_t}^2]
\leq \frac{2}{(2 - L_r \eta_{\max})(\eta_{\max}+\eta_{\min})}\left(2(f(x_0) - f^{\star}) + \frac{L_r \eta_{\max}^2 K\zeta(c)\sigma^2}{\underline{a}^{\lfloor c\rfloor}}\right)\frac{1}{T-T_w}.
\end{align*}
\normalsize
\textbf{[LR Decay Part: Polynomial Decay \eqref{eq:poly_dec_lr}]}\newline
Considering the proof of the polynomial decay LR \eqref{eq:poly_dec_lr} in Theorem \ref{thm:c_d}, we have
\begin{align*}
\sum_{t=T_w}^{T-1}\eta_t
&= \sum_{t=T_w}^{T-1}\left(\eta_{\min} + (\eta_{\max} - \eta_{\min})\left(1 - \frac{t-T_w}{T-T_w}\right)^p\right)
\geq \eta_{\min}(T -T_w) + (\eta_{\max} - \eta_{\min})\frac{T -T_w}{p+1}\\
&= \frac{\eta_{\max} + p\eta_{\min}}{p+1}(T-T_w)
\end{align*}
and
\begin{align*}
\sum_{t=T_w}^{T-1} \frac{\eta_t^2}{b_t}
&= \sum_{t=T_w}^{T-1} \frac{1}{b_t}\left(\eta_{\min}^2 + (\eta_{\max}^2 - \eta_{\min})^2\left(1 - \frac{t -T_w}{T-T_w}\right)^{2p} + 2\eta_{\min}(\eta_{\max} - \eta_{\min})\left(1 - \frac{t -T_w}{T-T_w}\right)^{p}\right)\\
&\leq \sum_{t=T_w}^{T-1}\frac{\eta_{\max}^2}{b_t}
\leq \frac{\eta_{\max}^2 K \zeta(c)}{\underline{a}^{\lfloor c\rfloor}},
\end{align*}
which yields
\small
\begin{align*}
\min_{t\in\{T_w, \cdots, T-1\}}\mathbb{E}[\|\mathrm{grad}f(x_t)\|_{x_t}^2]
\leq \frac{p+1}{(2 - L_r \eta_{\max})(\eta_{\max}+p\eta_{\min})}\left(2(f(x_0) - f^{\star}) + \frac{L_r \eta_{\max}^2 K\zeta(c)\sigma^2}{\underline{a}^{\lfloor c\rfloor}}\right)\frac{1}{T-T_w}.
\end{align*}
\normalsize
\end{proof}

\subsection{Proof of Theorem \ref{thm:c_w}}\label{prf:c_w}
\begin{theorem}[Detailed Version of Theorem \ref{thm:c_w}]\label{thm:c_w_detailed}
We consider a constant BS ($b_t=b>0$) and warm-up LRs \eqref{eq:exp_g_lr} and \eqref{eq:pol_g_lr} with decay parts given by \eqref{eq:const_lr}, \eqref{eq:dim_lr}, \eqref{eq:cosan_lr}, or \eqref{eq:poly_dec_lr} under the assumptions of Lemma \ref{lem:undl_anal}. Then, we obtain
\begin{itemize}
\item Decay part: Diminishing \eqref{eq:dim_lr}
 \begin{align*}
 \min_{t\in\{T_w, \cdots, T-1\}}\mathbb{E}[\|\mathrm{grad}f(x_t)\|_{x_t}^2]
 \leq (Q_1 + \frac{Q_2 \sigma^2}{b}\log\frac{T}{T_w})\frac{1}{\sqrt{T+1}-\sqrt{T_w +1}},
 \end{align*}
\item Decay part: Otherwise \eqref{eq:const_lr}, \eqref{eq:cosan_lr}, \eqref{eq:poly_dec_lr}
 \begin{align*}
 \min_{t\in\{T_w, \cdots, T-1\}}\mathbb{E}[\|\mathrm{grad}f(x_t)\|_{x_t}^2]
 \leq \frac{\tilde{Q}_1}{T-T_w}+\frac{\tilde{Q}_2\sigma^2}{b},
 \end{align*}
\end{itemize}
where $Q_1, Q_2, \tilde{Q}_1$, and $\tilde{Q}_2$ are constants that do not depend on $T$.
\end{theorem}

\begin{proof}
\textbf{[LR Decay Part: Constant \eqref{eq:const_lr}]}\newline
Considering \ref{prf:i_w}, we have
\begin{align*}
\sum_{t= T_w}^{T-1}\frac{\eta_t^2}{b_t} = \frac{\eta_{\max}^2}{b}(T-T_w),
\end{align*}
which yields
\begin{align*}
\min_{t\in\{T_w, \cdots, T-1\}}\mathbb{E}[\|\mathrm{grad}f(x_t)\|_{x_t}^2]
\leq \frac{2(f(x_0) - f^{\star})}{(2 - L_r \eta_{\max})\eta_{\max}}\frac{1}{T-T_w} + \frac{L_r \eta_{\max}}{2 - L_r \eta_{\max}}\frac{\sigma^2}{b}.
\end{align*}
\textbf{[LR Decay Part: Diminishing \eqref{eq:dim_lr}]}\newline
Similarly, we have
\begin{align*}
\sum_{t=T_w}^{T-1}\frac{\eta_t^2}{b_t}
= \frac{\eta_{\max}^2}{b}\sum_{t=T_w}^{T-1} \frac{1}{t+1}
\leq \frac{\eta_{\max}^2}{b}\left(1+\int_{T_w}^T \frac{dt}{t}\right)
= \frac{\eta_{\max}^2}{b} + \frac{\eta_{\max}^2}{b}\log\frac{T}{T_w},
\end{align*}
which yields
\small
\begin{align*}
\min_{t\in\{T_w, \cdots, T-1\}}\mathbb{E}[\|\mathrm{grad}f(x_t)\|_{x_t}^2]
\leq \frac{f(x_0) - f^{\star}}{(2 - L_r \eta_{\max})\eta_{\max}}\frac{1}{\sqrt{T+1}-\sqrt{T_w+1}} + \frac{L_r \eta_{\max}}{2(2 - L_r \eta_{\max})}\frac{\sigma^2}{b}\frac{1+\log\frac{T}{T_w}}{\sqrt{T+1}-\sqrt{T_w+1}}.
\end{align*}
\normalsize
\textbf{[LR Decay Part: Cosine Annealing \eqref{eq:cosan_lr}]}\newline
Doing the same with \ref{prf:i_w}, we have
\begin{align*}
\sum_{t=T_w}^{T-1}\frac{\eta_t^2}{b_t}
&= \frac{1}{b}\sum_{t=T_w}^{T-1}\left(\frac{(\eta_{\max} + \eta_{\min})^2}{4} + \frac{\eta_{\max}^2 - \eta_{\min}^2}{2}\cos\frac{t-T_w}{T-T_w}\pi + \frac{(\eta_{\max} - \eta_{\min})^2}{4}\cos^2 \frac{t-T_w}{T-T_w}\pi\right)\\
&\leq \frac{1}{b} \sum_{t=T_w}^{T-1} \left(\frac{(\eta_{\max} + \eta_{\min})^2}{4} + \frac{\eta_{\max}^2 - \eta_{\min}^2}{2} + \frac{(\eta_{\max} - \eta_{\min})^2}{4}\right)
= \frac{\eta_{\max}^2}{b}(T-T_w),
\end{align*}
which yields
\begin{align*}
\min_{t\in\{T_w, \cdots, T-1\}}\mathbb{E}[\|\mathrm{grad}f(x_t)\|_{x_t}^2]
\leq \frac{4(f(x_0) - f^{\star})}{(2 - L_r \eta_{\max})(\eta_{\max}+\eta_{\min})}\frac{1}{T-T_w} + \frac{2L_r \eta_{\max}^2}{(2 - L_r \eta_{\max})(\eta_{\max}+\eta_{\min})}\frac{\sigma^2}{b}.
\end{align*}
\textbf{[LR Decay Part: Polynomial Decay \eqref{eq:poly_dec_lr}]}\newline
Similarly, we have
\begin{align*}
\sum_{t=T_w}^{T-1} \frac{\eta_t^2}{b_t}
&= \frac{1}{b} \sum_{t=T_w}^{T-1} \left(\eta_{\min}^2 + (\eta_{\max}^2 - \eta_{\min})^2\left(1 - \frac{t -T_w}{T-T_w}\right)^{2p} + 2\eta_{\min}(\eta_{\max} - \eta_{\min})\left(1 - \frac{t -T_w}{T-T_w}\right)^{p}\right)\\
&\leq \sum_{t=T_w}^{T-1}\frac{\eta_{\max}^2}{b_t}
= \frac{\eta_{\max}^2}{b}(T-T_w),
\end{align*}
which yields
\begin{align*}
\min_{t\in\{T_w, \cdots, T-1\}}\mathbb{E}[\|\mathrm{grad}f(x_t)\|_{x_t}^2]
\leq \frac{2(f(x_0) - f^{\star})(p+1)}{(2 - L_r \eta_{\max})(\eta_{\max}+p\eta_{\min})}\frac{1}{T-T_w} + \frac{2L_r (p+1) \eta_{\max}^2}{(2 - L_r \eta_{\max})(\eta_{\max}+p\eta_{\min})}\frac{\sigma^2}{b}.
\end{align*}
\end{proof}

\section{Objective Function Values in Sections \ref{nume:pca} and \ref{nume:lrmc}}\label{appdix:obj}
\begin{figure}[htbp]
\centering
\includegraphics[width=0.24\linewidth]{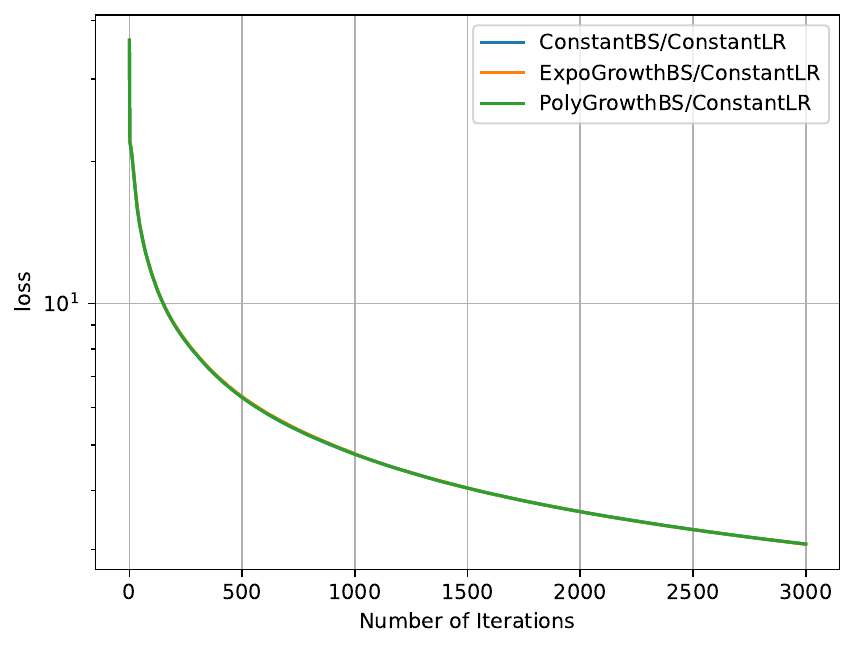}
\includegraphics[width=0.24\linewidth]{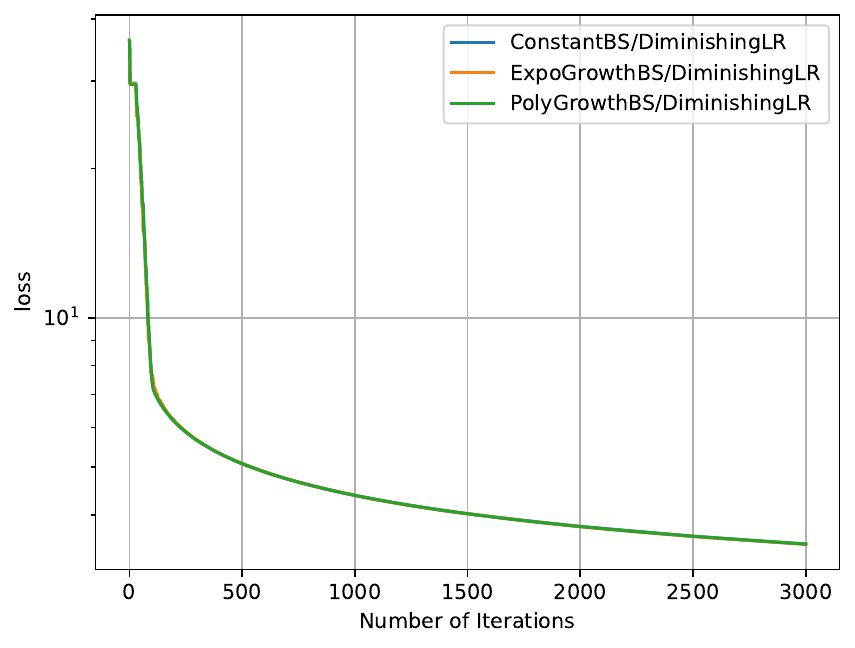}
\includegraphics[width=0.24\linewidth]{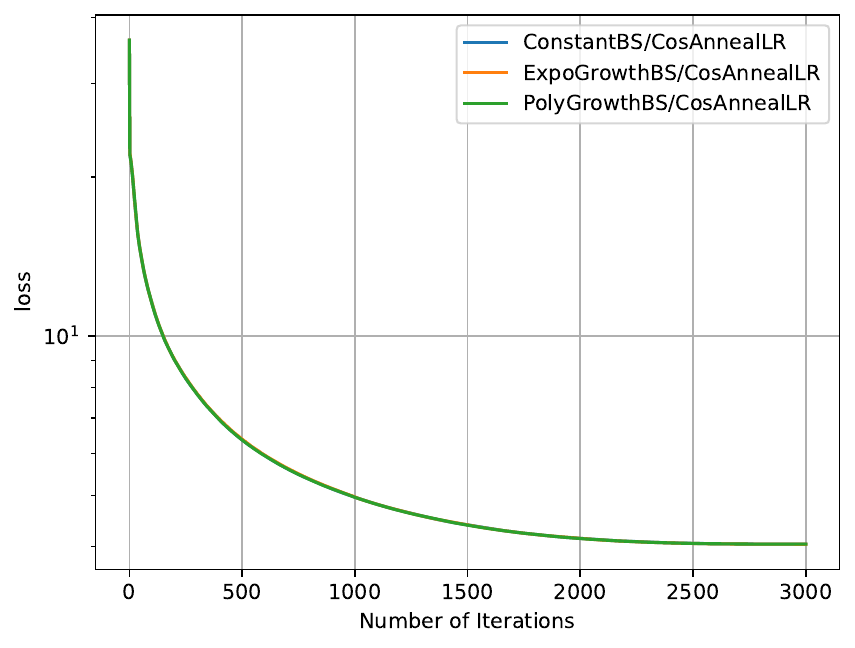}
\includegraphics[width=0.24\linewidth]{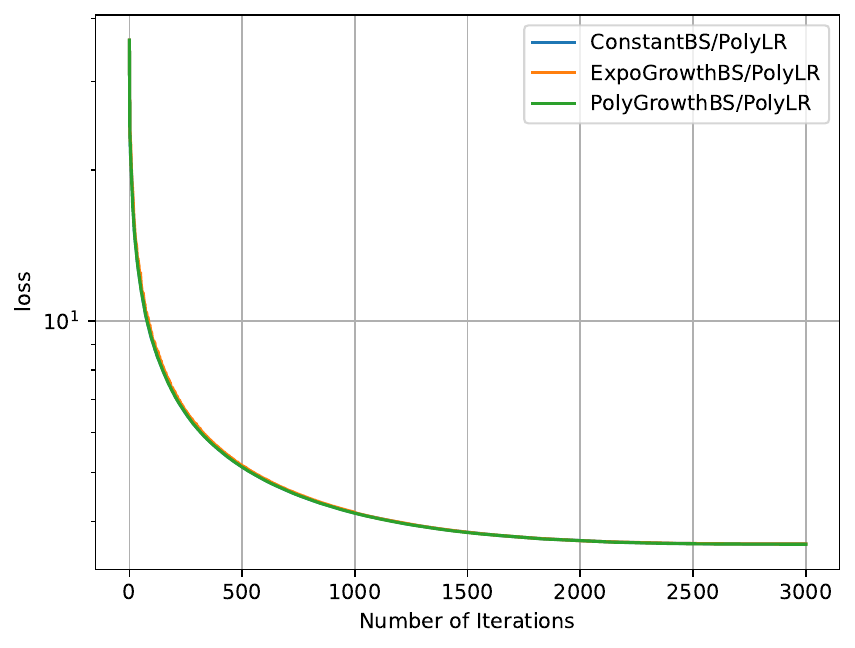}
\caption{Objective function value (loss) versus number of iterations for LRs \eqref{eq:const_lr}, \eqref{eq:dim_lr}, \eqref{eq:cosan_lr}, and \eqref{eq:poly_dec_lr} in order from left to right on COIL100 dataset (PCA).}
\label{fig:COIL100_loss}
\end{figure}

\begin{figure}[htbp]
\centering
\includegraphics[width=0.24\linewidth]{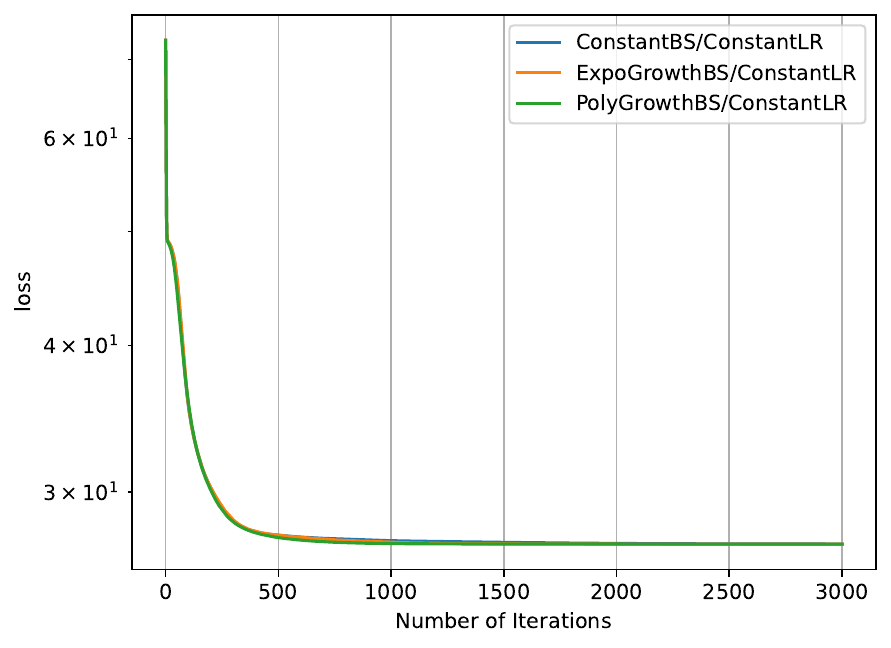}
\includegraphics[width=0.24\linewidth]{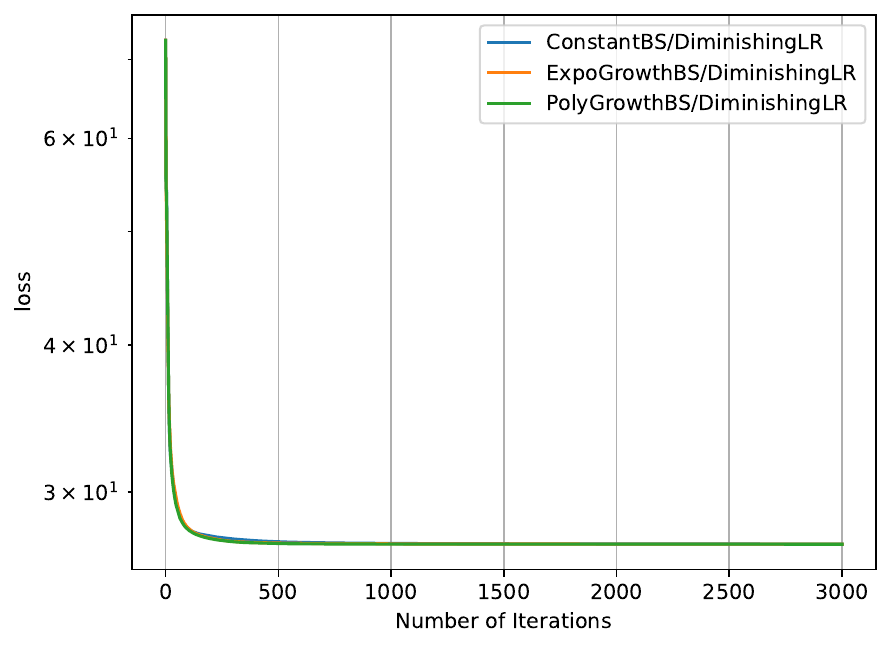}
\includegraphics[width=0.24\linewidth]{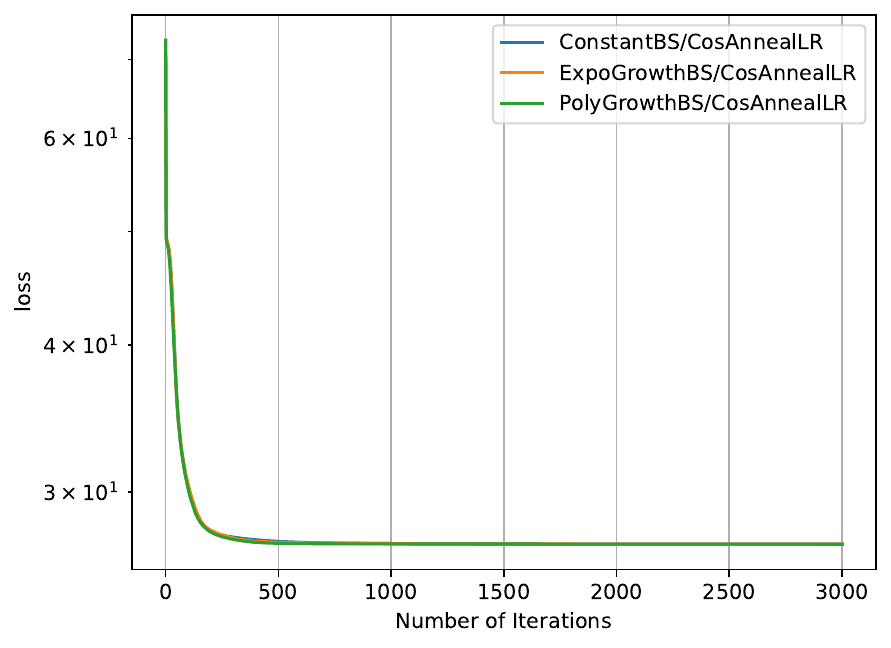}
\includegraphics[width=0.24\linewidth]{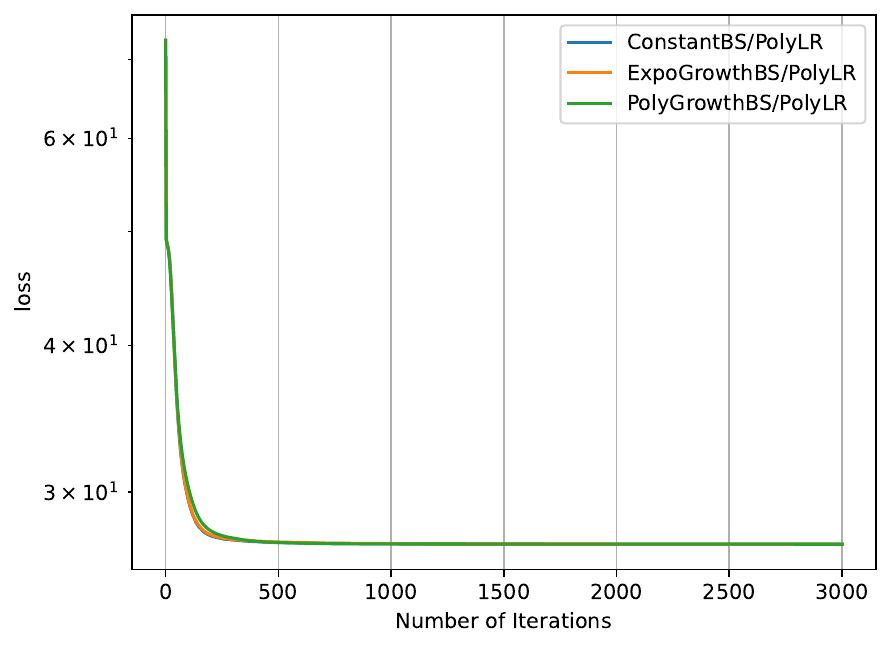}
\caption{Objective function value (loss) versus number of iterations for LRs \eqref{eq:const_lr}, \eqref{eq:dim_lr}, \eqref{eq:cosan_lr}, and \eqref{eq:poly_dec_lr} in order from left to right on MNIST dataset (PCA).}
\label{fig:MNIST_loss}
\end{figure}

\begin{figure}[htbp]
\centering
\includegraphics[width=0.24\linewidth]{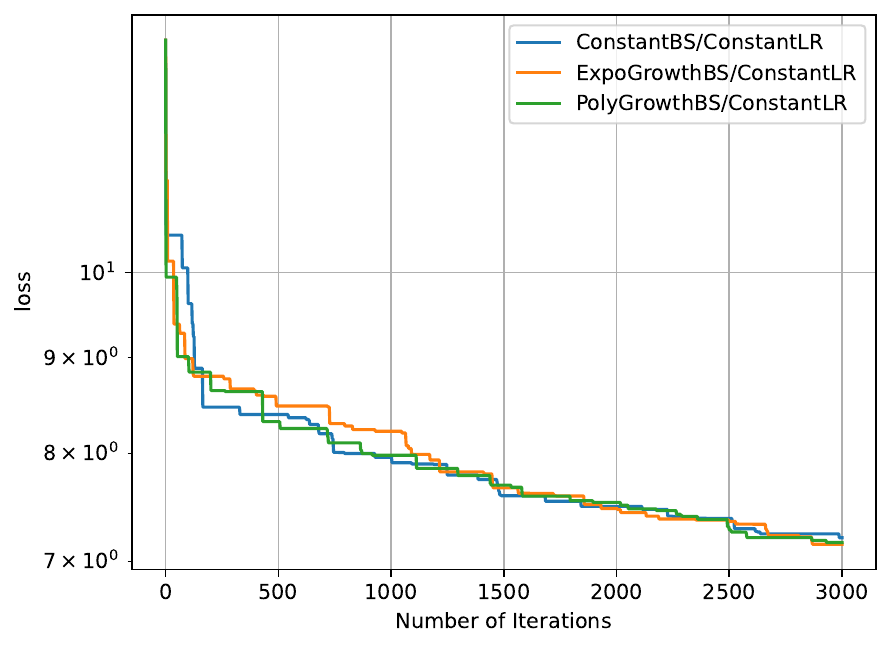}
\includegraphics[width=0.24\linewidth]{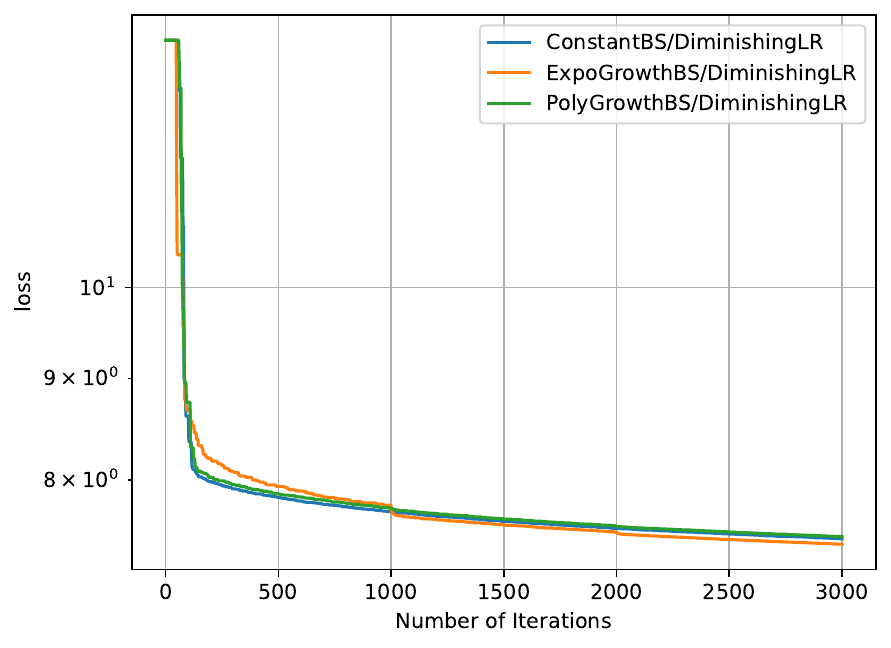}
\includegraphics[width=0.24\linewidth]{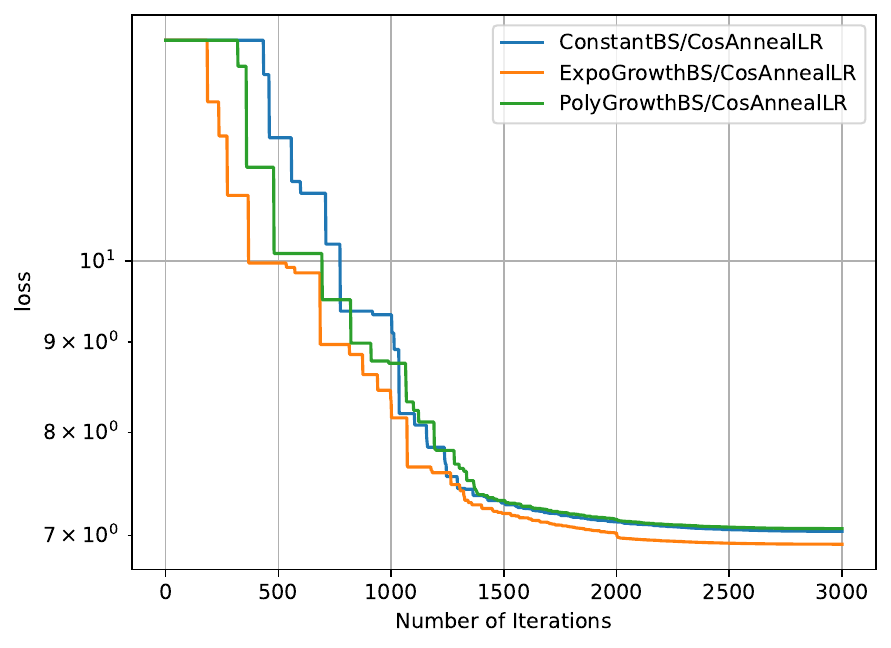}
\includegraphics[width=0.24\linewidth]{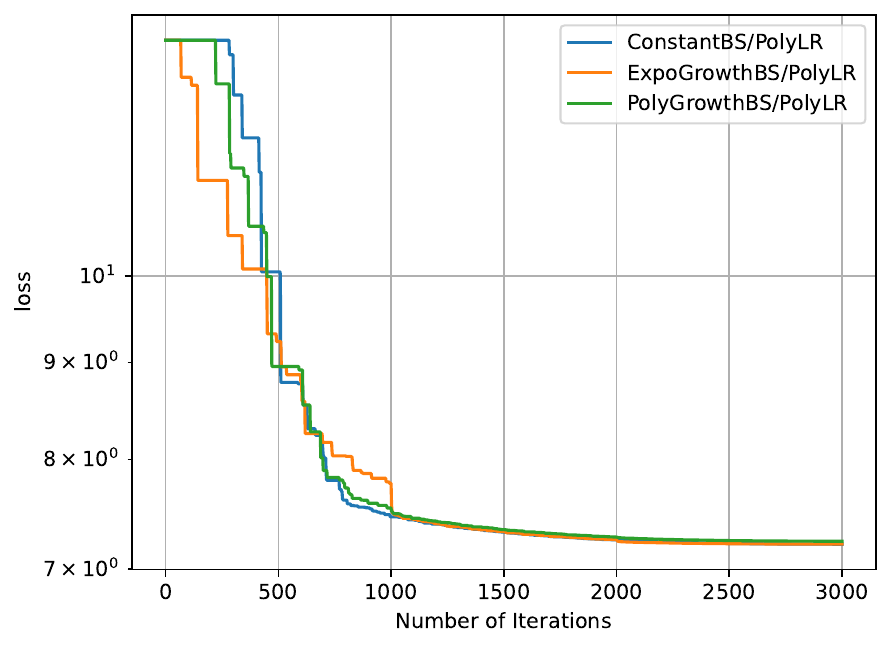}
\caption{Objective function value (loss) versus number of iterations for LRs \eqref{eq:const_lr}, \eqref{eq:dim_lr}, \eqref{eq:cosan_lr}, and \eqref{eq:poly_dec_lr} in order from left to right on MovieLens-1M dataset (LRMC).}
\label{fig:ml-1m_loss}
\end{figure}

\begin{figure}[htbp]
\centering
\includegraphics[width=0.24\linewidth]{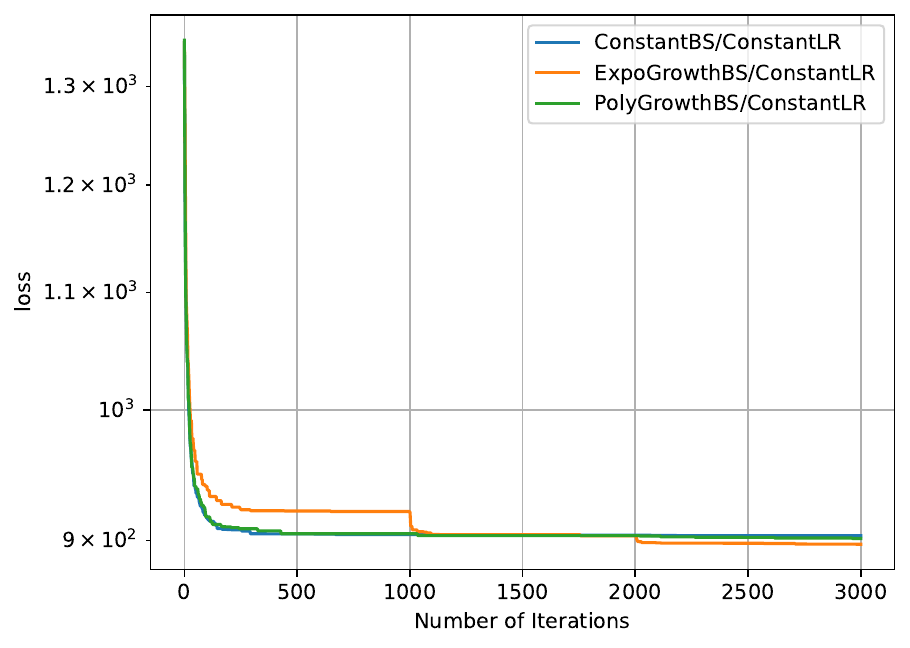}
\includegraphics[width=0.24\linewidth]{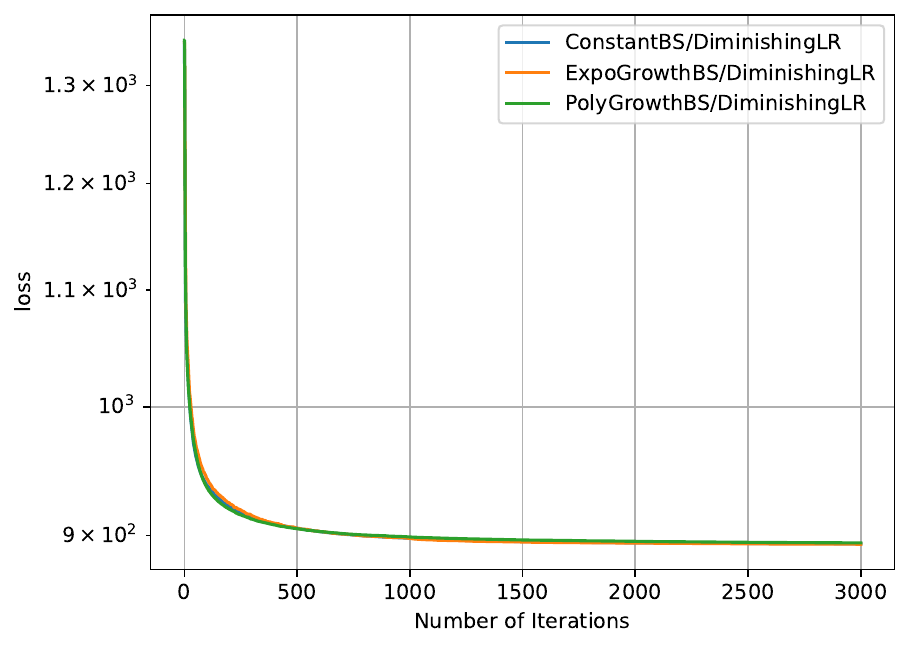}
\includegraphics[width=0.24\linewidth]{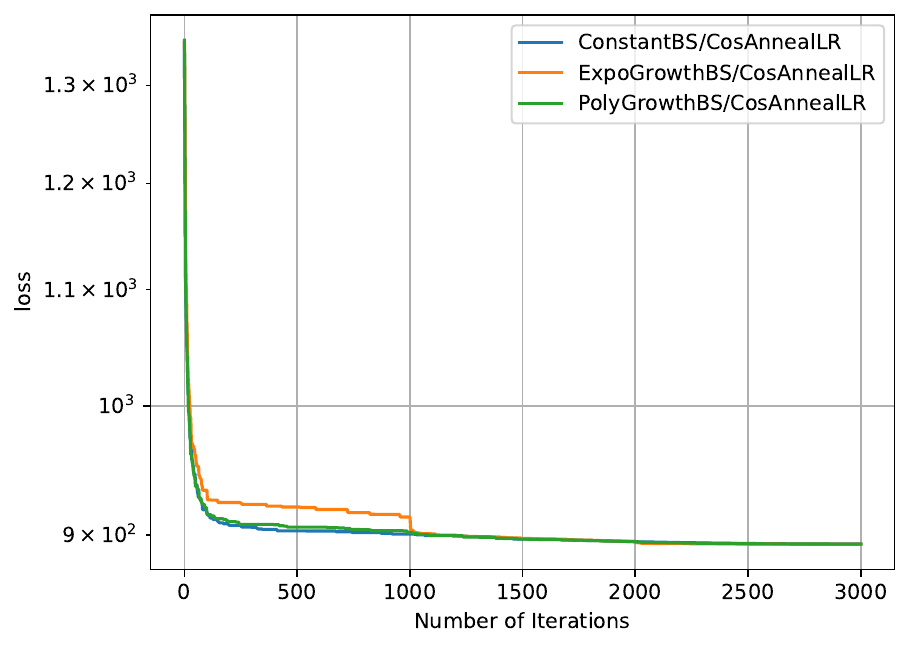}
\includegraphics[width=0.24\linewidth]{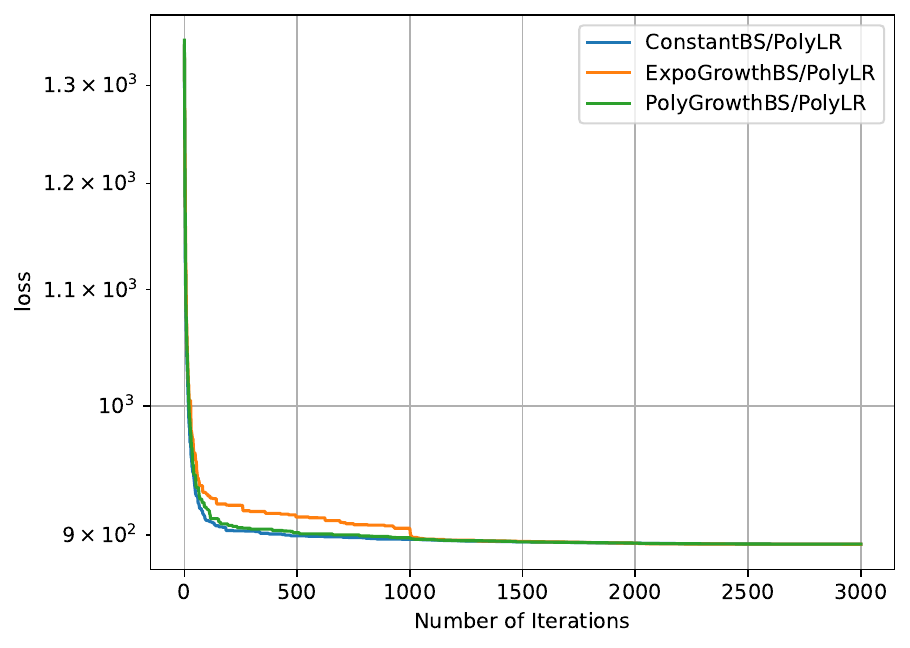}
\caption{Objective function value (loss) versus number of iterations for LRs \eqref{eq:const_lr}, \eqref{eq:dim_lr}, \eqref{eq:cosan_lr}, and \eqref{eq:poly_dec_lr} in order from left to right on Jester dataset (LRMC).}
\label{fig:jester_loss}
\end{figure}
\noindent
The performance in terms of the objective function value versus the number of iterations for LRs \eqref{eq:const_lr}, \eqref{eq:dim_lr}, \eqref{eq:cosan_lr}, and \eqref{eq:poly_dec_lr} on the COIL100, MNIST, MovieLens-1M, and Jester datasets are shown in Figures \ref{fig:COIL100_loss}, \ref{fig:MNIST_loss}, \ref{fig:ml-1m_loss}, and \ref{fig:jester_loss}, respectively. Although the differences in the objective function values are small (one possible reason for this is that the objective function may be flat around the optimal solution), the performance with an increasing BS was equal to or better than that with a constant BS. A more detailed discussion of this hypothesis is provided in Section \ref{sec:SFO}.

\section{Convergence Criteria in Table \ref{table:exi_rslt}}\label{apdix:explain_conv_criteria}
In previous studies and in our work, three types of convergence criteria were used:
\begin{enumerate}
 \item the objective function value $\mathbb{E}[f(x_T)-f^{\star}]$,
 \item the gradient value $\underset{t\to \infty}{\lim}\mathbb{E}[\mathrm{grad}f(x_t)]= 0$ or $\|G_T\|^2$,
 \item (extended) Riemannian distance between $x_T$ and an optimal solution $x^{\star}$: $\mathbb{E}[\|\Delta_T\|]\coloneq\mathbb{E}[\|R_{x^{\star}}^{-1}(x_T)\|]$.
\end{enumerate}
Criterion (3) applies because $\mathbb{E}[\|\Delta_T\|]=\mathbb{E}[d(x_T,x^{\star})]$ holds when $R=\text{Exp}$.

\section{Additional Numerical Results}
\label{sec:addi_exper}
This section presents the numerical results for a warm-up LR. We used a constant BS, an exponential growth BS \eqref{eq:exp_g_bs}, a polynomial growth BS \eqref{eq:pol_g_bs}, and a warm-up LR with an increasing part (exponential growth LR) and a decaying part (either a constant LR \eqref{eq:const_lr}, a diminishing LR \eqref{eq:dim_lr}, a cosine annealing LR \eqref{eq:cosan_lr}, or a polynomial decay LR \eqref{eq:poly_dec_lr}). We set $K'=200$. When we used an exponential growth LR, a polynomial growth BS was not required to satisfy $\delta^{2l}<\gamma$, whereas an exponential growth BS was required to satisfy it (see Section \ref{i_w}). For comparison, even when using a polynomial growth BS, we adopted a setting that satisfies this condition. Furthermore, we chose $\eta_{\max}$ (the initial value of decaying part) from $\{0.5, 0.05, 0.005\}$. When using an exponential BS, we set the initial BS $b_0\coloneq 3^4$ in \textbf{Cases A, B} and $b_0\coloneq 3^3$ in \textbf{Cases C, D}. When using the other BSs, we set the same initial batch size as in Section \ref{sec:numl_expe}.

From the definition of an exponential growth LR \eqref{eq:exp_g_lr}, we can represent the part after warm-up as $\eta_{T_w -1} = \eta_0 \delta^{l_w}$. Consequently, we chose hyperparameters $(l,l_w,\gamma,\eta_{\max},\eta_0)$ satisfying $\eta_0 \delta^{l_w} = \eta_{\max}$ and $\delta^{2l} < \gamma$; i.e., $\delta=(\frac{\eta_{\max}}{\eta_0})^{\frac{1}{l_w}} < \gamma^{\frac{1}{2l}}$. Hence, when $l=l_w=3$ and $\gamma=3.0$, we set $\eta_{\max}=0.5, 0.05, 0.005$ and $\eta_0=\frac{5}{17},\frac{5}{170},\frac{5}{1700}$, respectively. In this setting, $\delta= \sqrt[3]{\frac{\eta_{\max}}{\eta_0}}=\sqrt[3]{1.7}< \sqrt[6]{3}=\sqrt[6]{\gamma}=\gamma^{\frac{1}{2l}}$ holds.  Similarly, when $l=3, l_w=8, \gamma=3$, we set $\eta_{\max}=0.5, 0.05, 0.005$, $\eta_0=\frac{1}{8},\frac{1}{80},\frac{1}{800}$. respectively. In this setting, $\delta=\sqrt[8]{\frac{\eta_{\max}}{\eta_0}}=\sqrt[8]{4}<\sqrt[6]{3}=\sqrt[6]{\gamma}=\gamma^{\frac{1}{2l}}$ holds. Because we terminated RSGD after the $3000$th iteration and set $K'=200$, we used $l=l_w=3$ (resp. $l=3, l_w=8$). In the first setting, the batch size increases five times and the learning rate three times; in the second, the batch size increases five times and the learning rate eight times.

\subsection{Principle Components Analysis}
\noindent
$\textbf{[Case A]} \quad l=l_w=3$\\
\noindent
Figures \ref{fig:COIL100_grad_wup3} and \ref{fig:MNIST_grad_wup3} plot performance in terms of the gradient norm of the objective function versus the number of iterations for a warm-up LR with decay parts given by \eqref{eq:const_lr}, \eqref{eq:dim_lr}, \eqref{eq:cosan_lr}, and \eqref{eq:poly_dec_lr} on the COIL100 and MNIST datasets, respectively. Figures \ref{fig:COIL100_loss_wup3} and \ref{fig:MNIST_loss_wup3} plot performance in terms of the objective function value versus the number of iterations for a warm-up LR with decay parts given by \eqref{eq:const_lr}, \eqref{eq:dim_lr}, \eqref{eq:cosan_lr}, and \eqref{eq:poly_dec_lr} for the COIL100 and MNIST datasets, respectively.
\begin{figure}[htbp]
\centering
\includegraphics[width=0.24\linewidth]{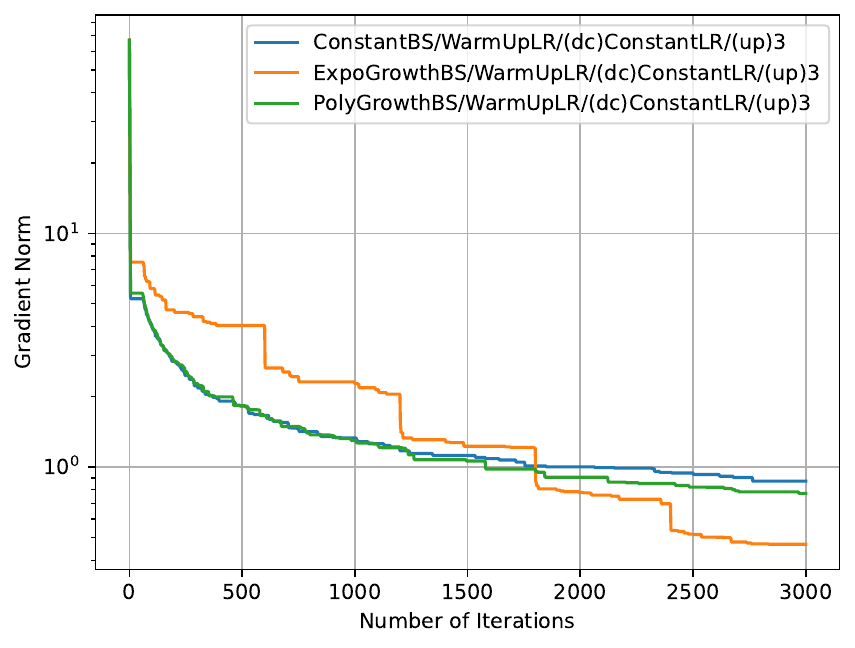}
\includegraphics[width=0.24\linewidth]{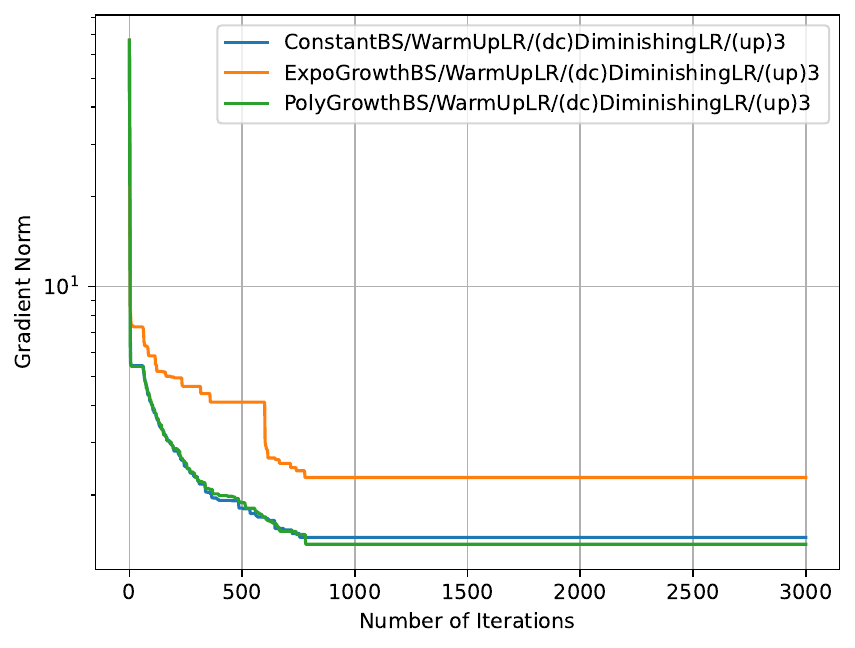}
\includegraphics[width=0.24\linewidth]{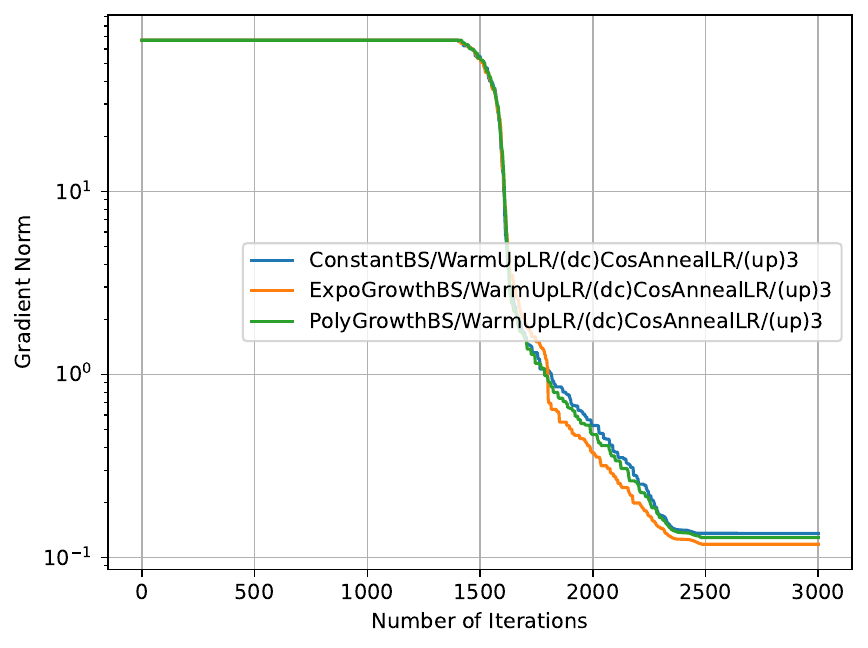}
\includegraphics[width=0.24\linewidth]{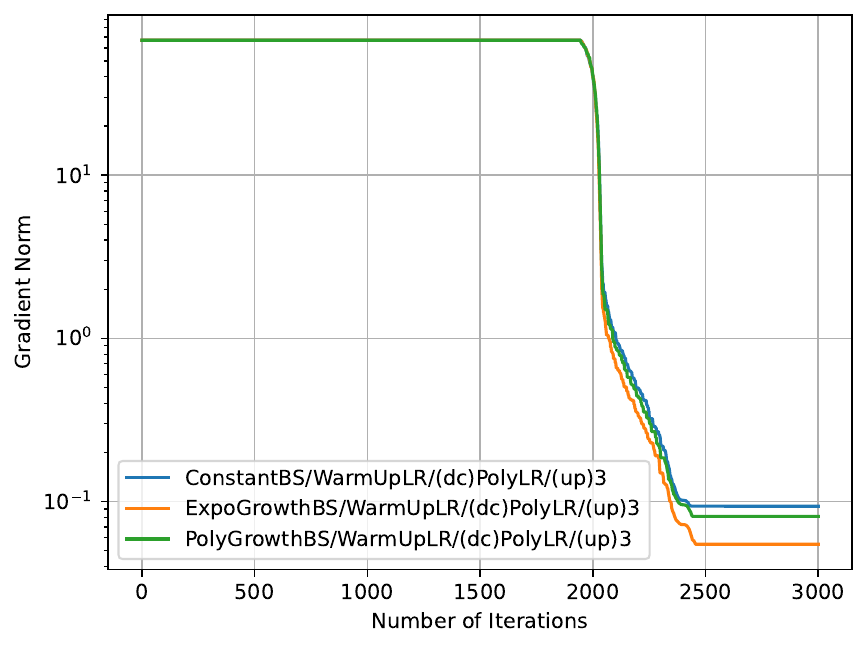}
\caption{Norm of the gradient of the objective function versus number of iterations for warm-up LRs that have an increasing part with three increments on COIL100 dataset (PCA).}
\label{fig:COIL100_grad_wup3}
\end{figure}

\begin{figure}[htbp]
\centering
\includegraphics[width=0.24\linewidth]{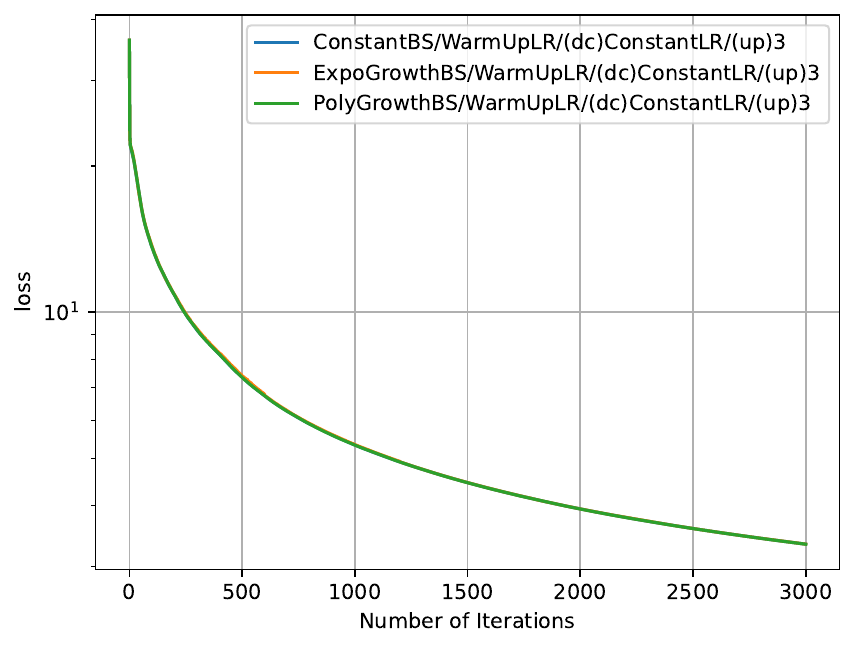}
\includegraphics[width=0.24\linewidth]{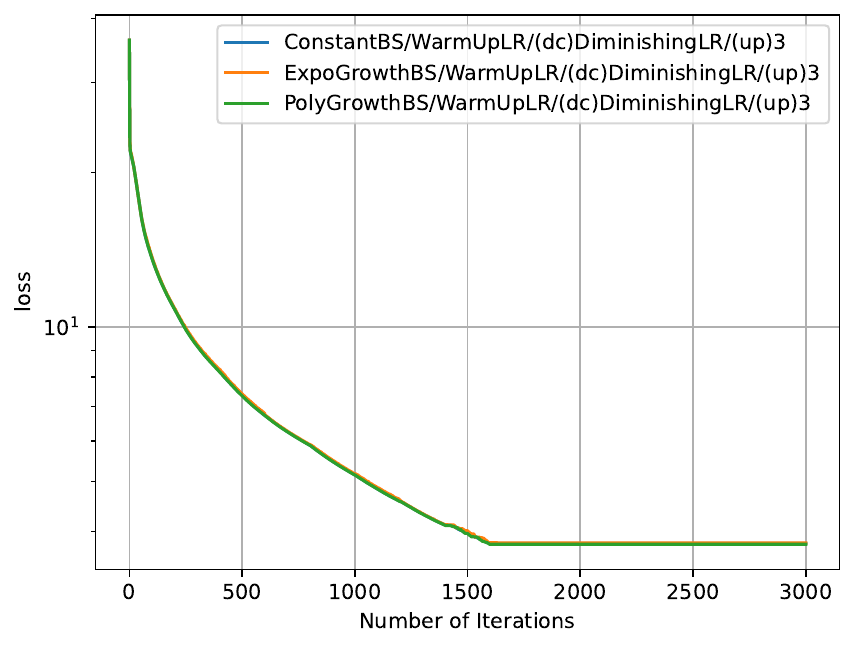}
\includegraphics[width=0.24\linewidth]{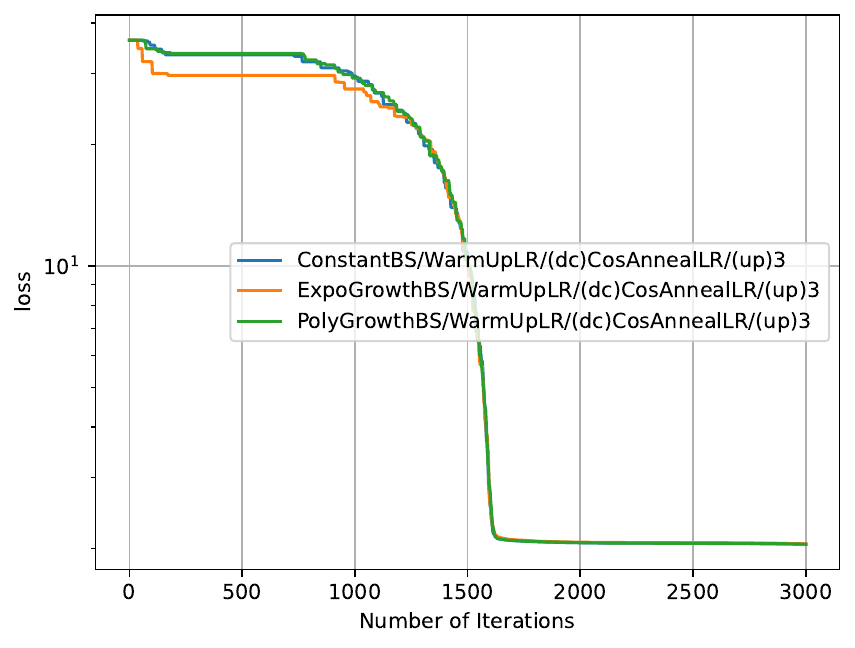}
\includegraphics[width=0.24\linewidth]{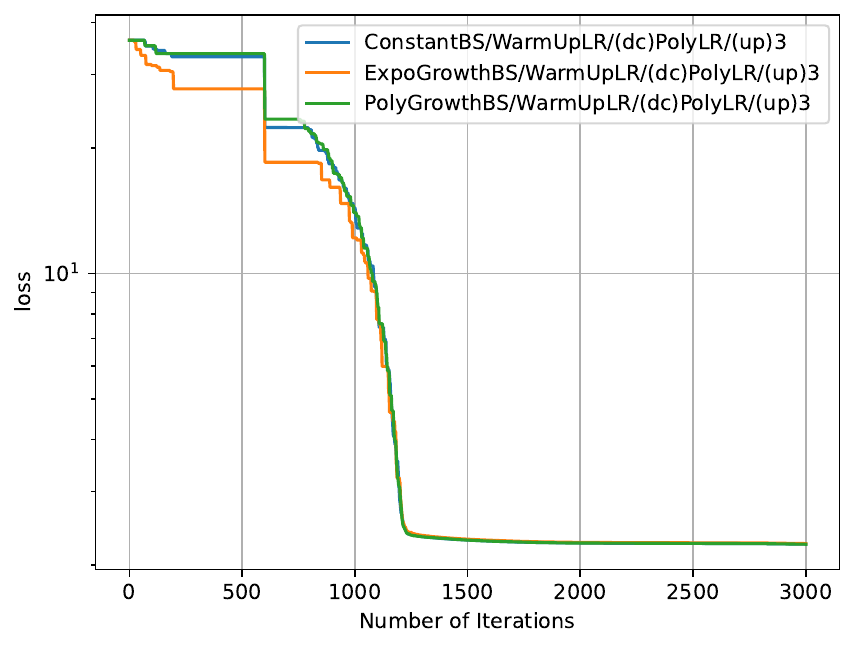}
\caption{Objective function value (loss) versus number of iterations for warm-up LRs that have an increasing part with three increments on COIL100 dataset (PCA).}
\label{fig:COIL100_loss_wup3}
\end{figure}

\begin{figure}[htbp]
\centering
\includegraphics[width=0.24\linewidth]{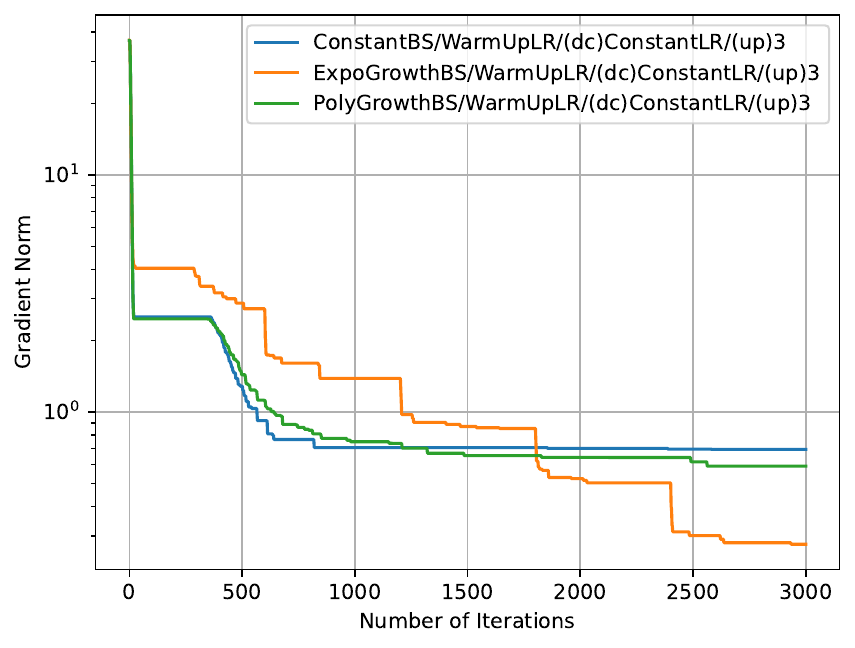}
\includegraphics[width=0.24\linewidth]{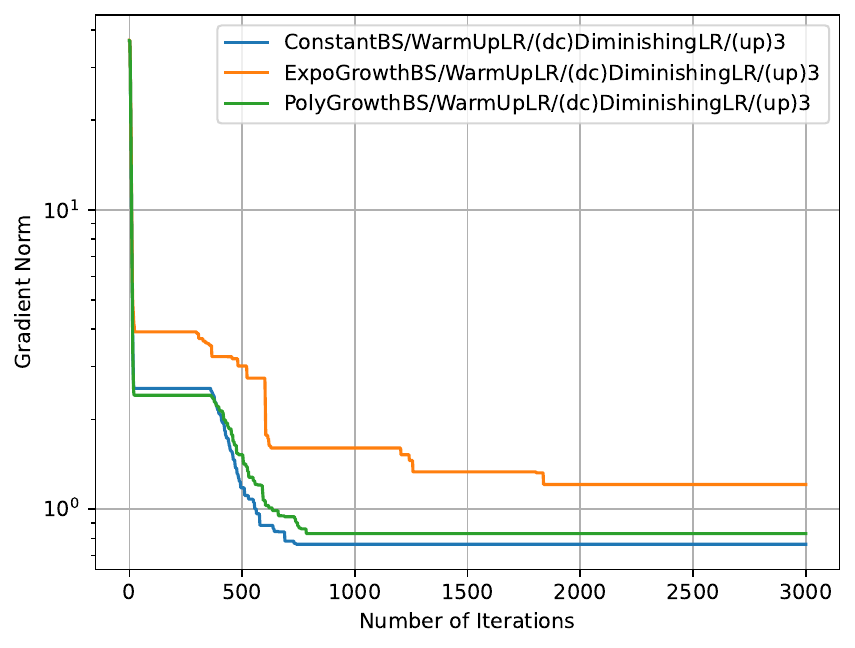}
\includegraphics[width=0.24\linewidth]{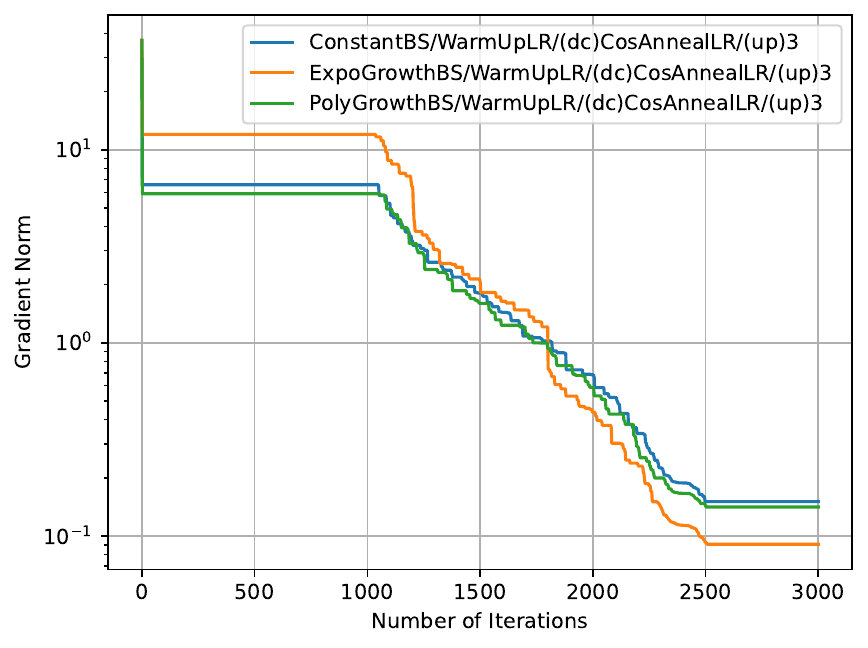}
\includegraphics[width=0.24\linewidth]{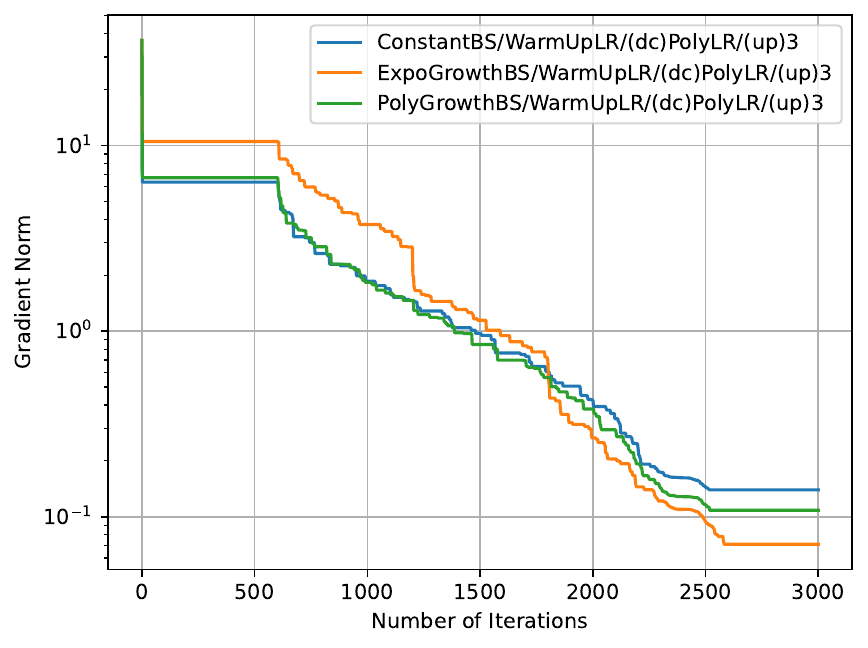}
\caption{Norm of the gradient of the objective function versus number of iterations for warm-up LRs that have an increasing part with three increments on MNIST dataset (PCA).}
\label{fig:MNIST_grad_wup3}
\end{figure}

\begin{figure}[htbp]
\centering
\includegraphics[width=0.24\linewidth]{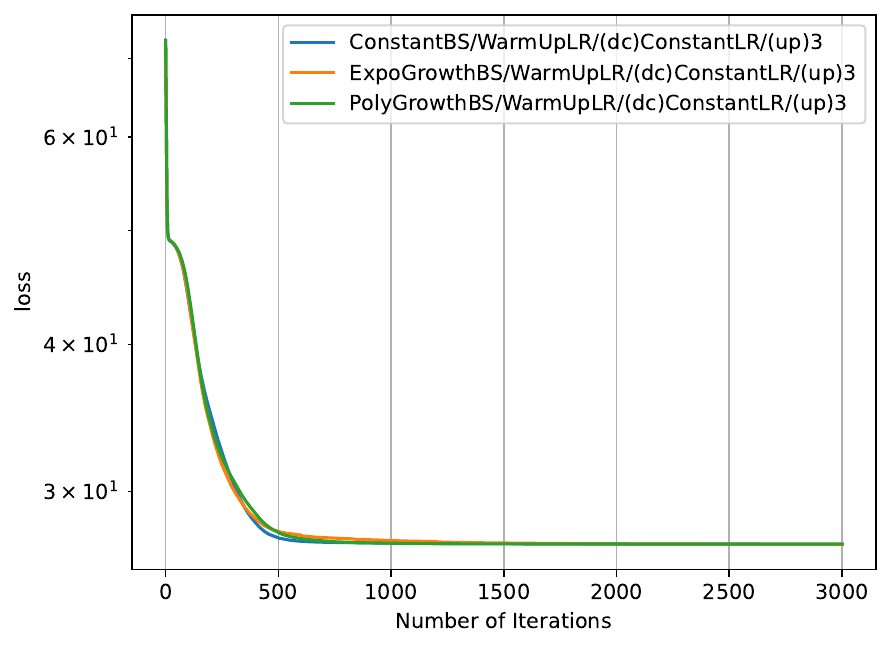}
\includegraphics[width=0.24\linewidth]{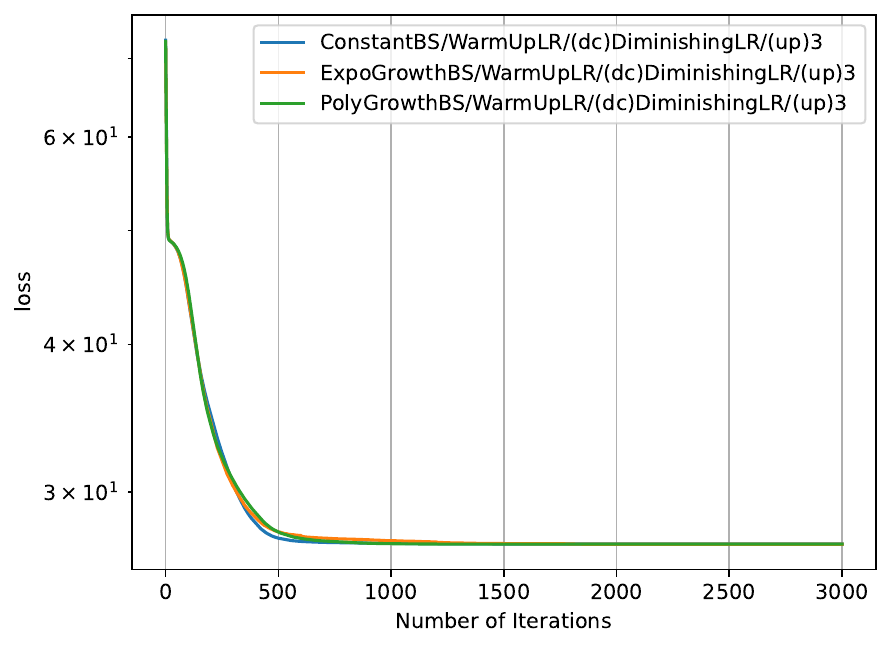}
\includegraphics[width=0.24\linewidth]{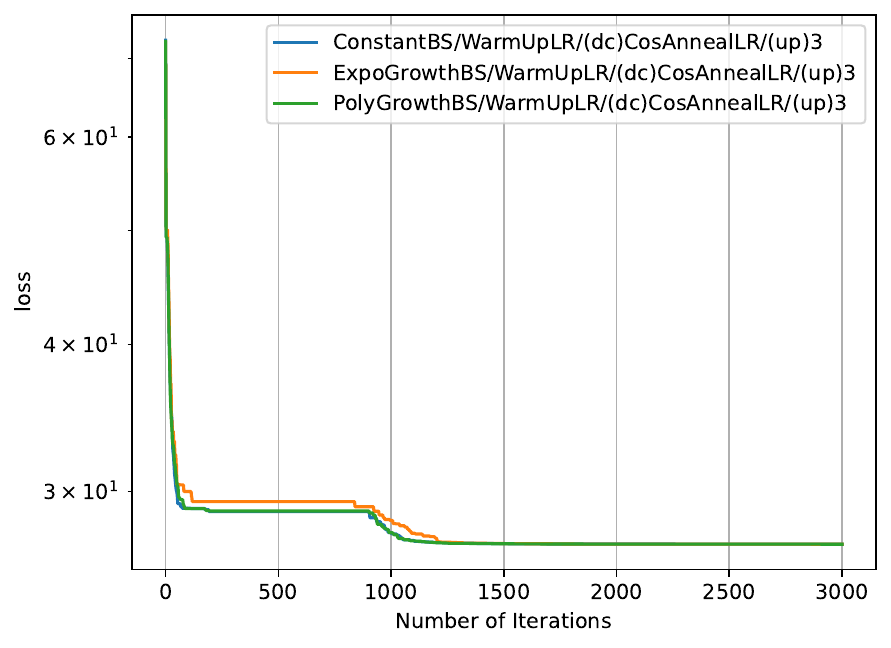}
\includegraphics[width=0.24\linewidth]{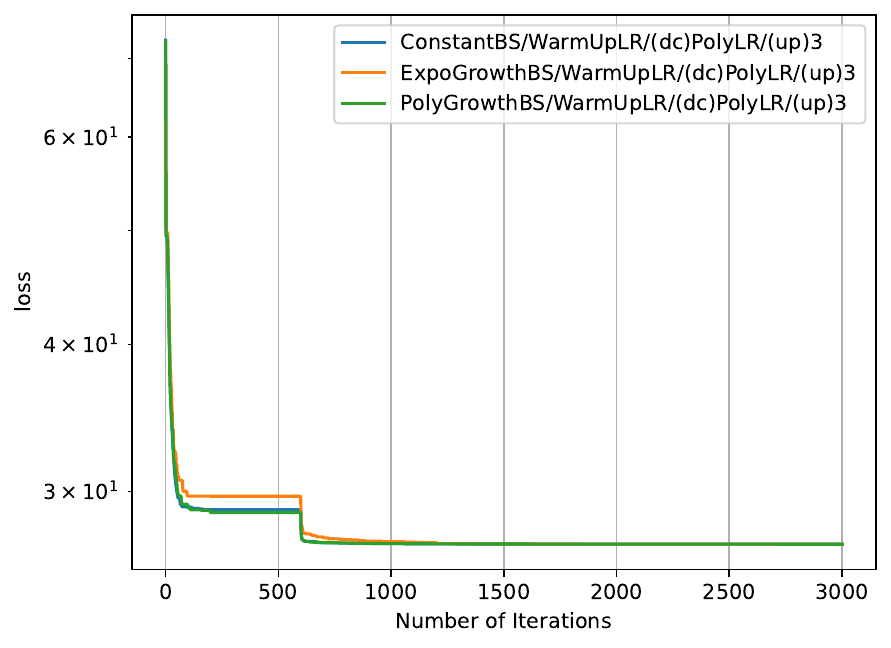}
\caption{Objective function value (loss) versus number of iterations for warm-up LRs that have an increasing part with three increments on MNIST dataset (PCA).}
\label{fig:MNIST_loss_wup3}
\end{figure}

\noindent$\textbf{[Case B]} \quad l=3, l_w=8$\\
\noindent
Figures \ref{fig:COIL100_grad_wup8} and \ref{fig:MNIST_grad_wup8} plot performance in terms of the gradient norm of the objective function versus the number of iterations for a warm-up LR with decay parts given by \eqref{eq:const_lr}, \eqref{eq:dim_lr}, \eqref{eq:cosan_lr}, and \eqref{eq:poly_dec_lr} for the COIL100 and MNIST datasets, respectively. Figures \ref{fig:COIL100_loss_wup8} and \ref{fig:MNIST_loss_wup8} plot performance in terms of the objective function value versus the number of iterations for a warm-up LR with decay parts given by \eqref{eq:const_lr}, \eqref{eq:dim_lr}, \eqref{eq:cosan_lr}, and \eqref{eq:poly_dec_lr} for the COIL100 and MNIST datasets, respectively.
\begin{figure}[htbp]
\centering
\includegraphics[width=0.24\linewidth]{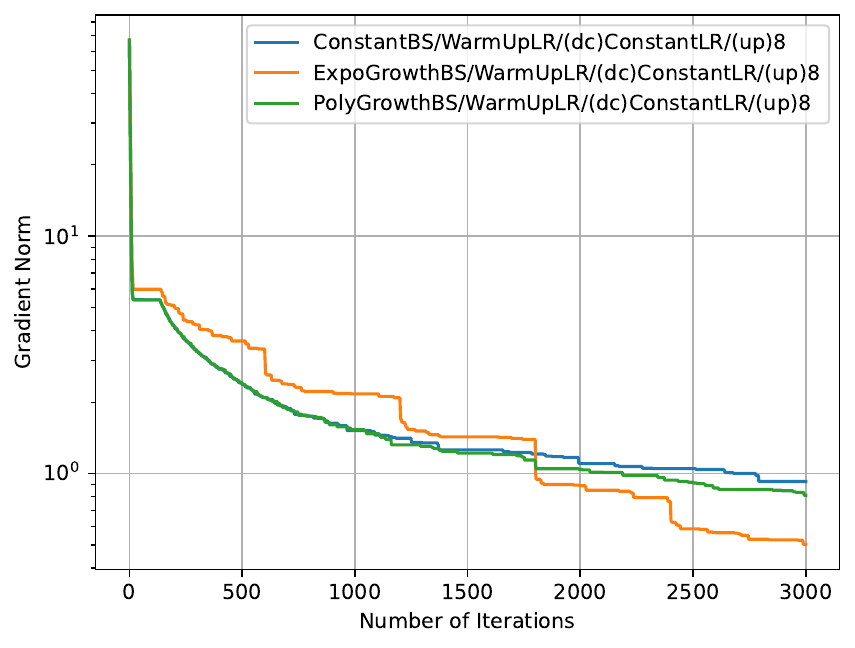}
\includegraphics[width=0.24\linewidth]{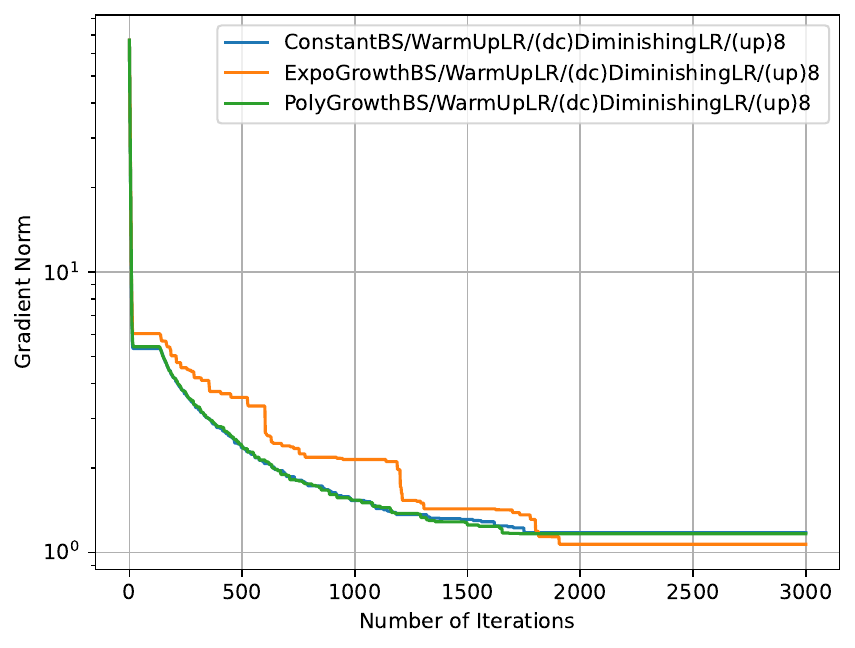}
\includegraphics[width=0.24\linewidth]{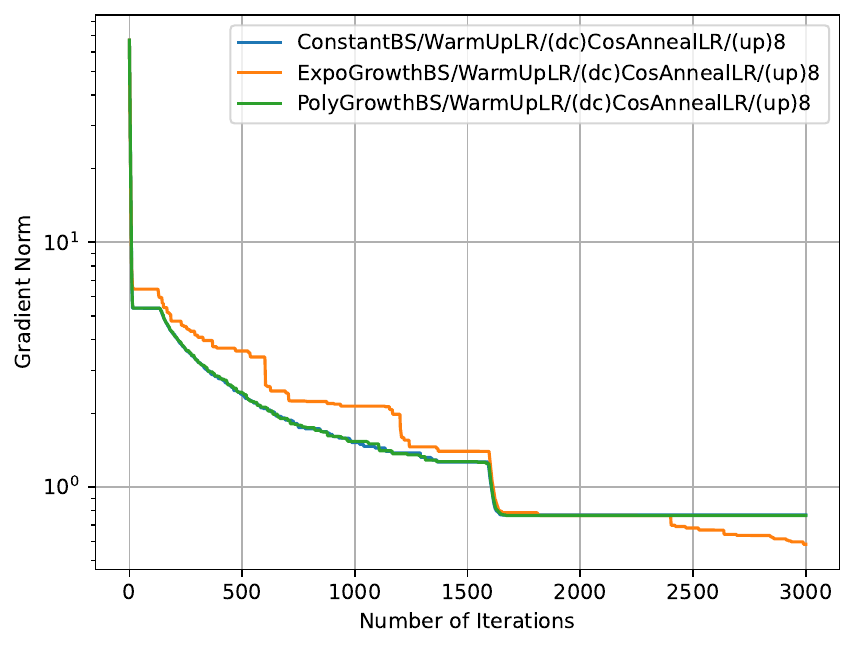}
\includegraphics[width=0.24\linewidth]{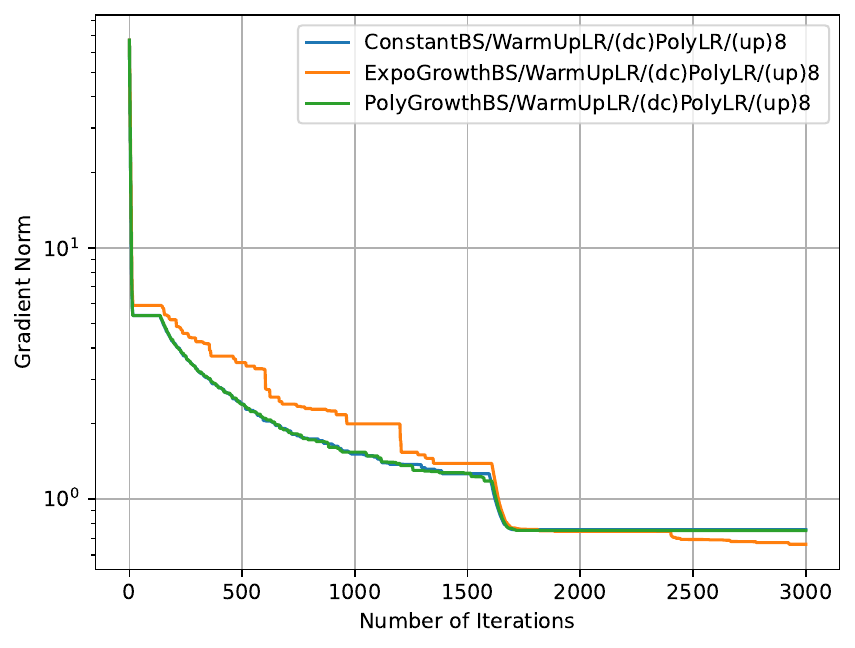}
\caption{Norm of the gradient of the objective function versus number of iterations for warm-up LRs that have an increasing part with eight increments on COIL100 dataset (PCA).}
\label{fig:COIL100_grad_wup8}
\end{figure}

\begin{figure}[htbp]
\centering
\includegraphics[width=0.24\linewidth]{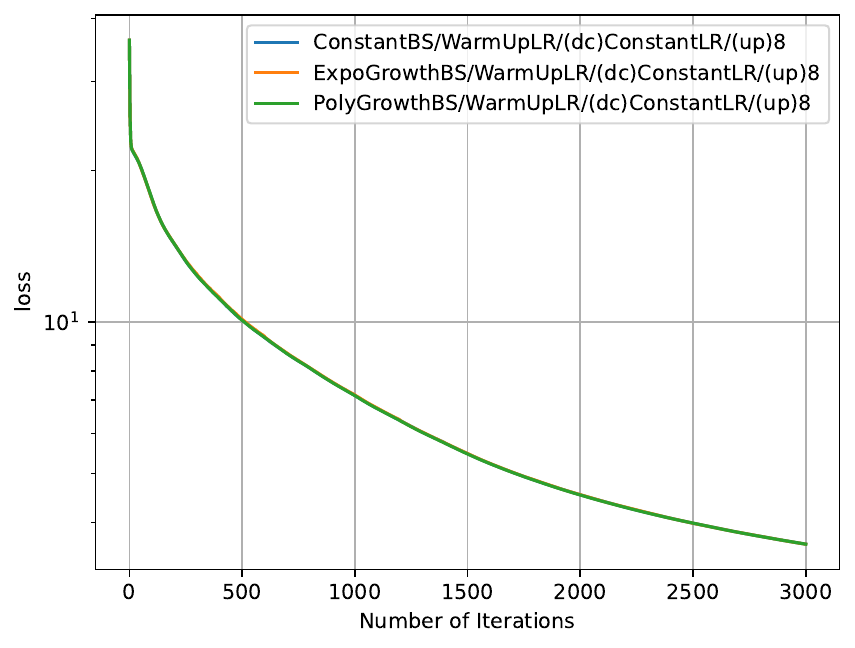}
\includegraphics[width=0.24\linewidth]{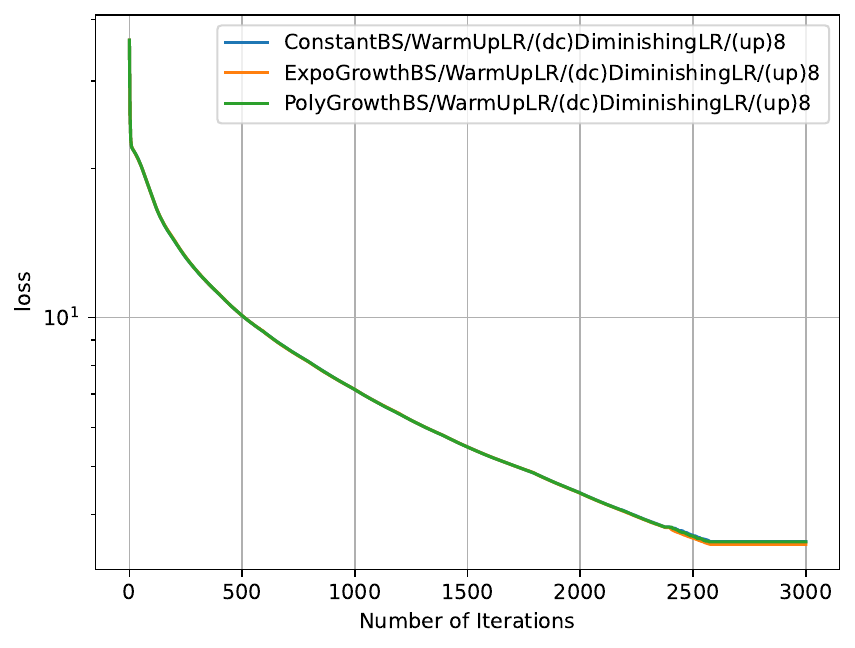}
\includegraphics[width=0.24\linewidth]{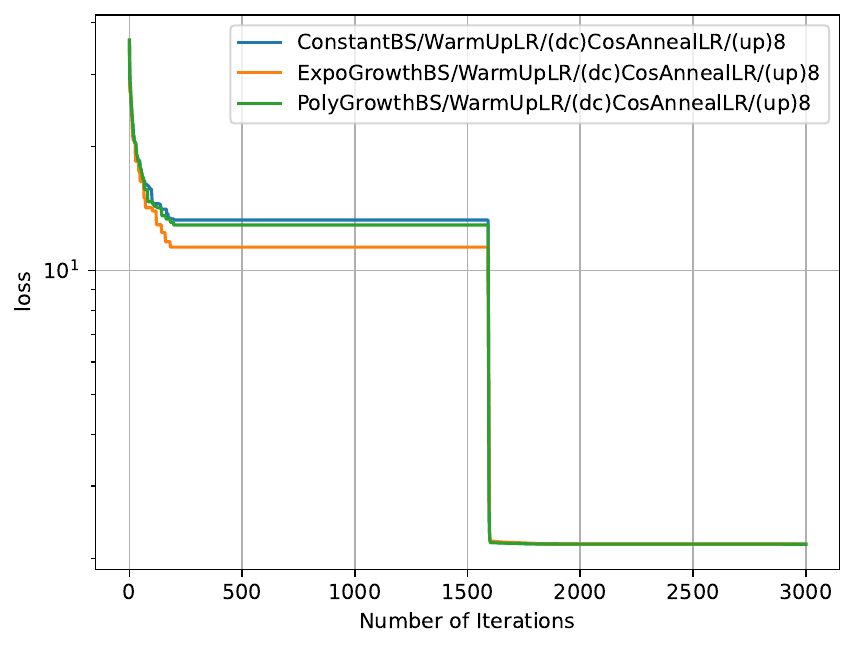}
\includegraphics[width=0.24\linewidth]{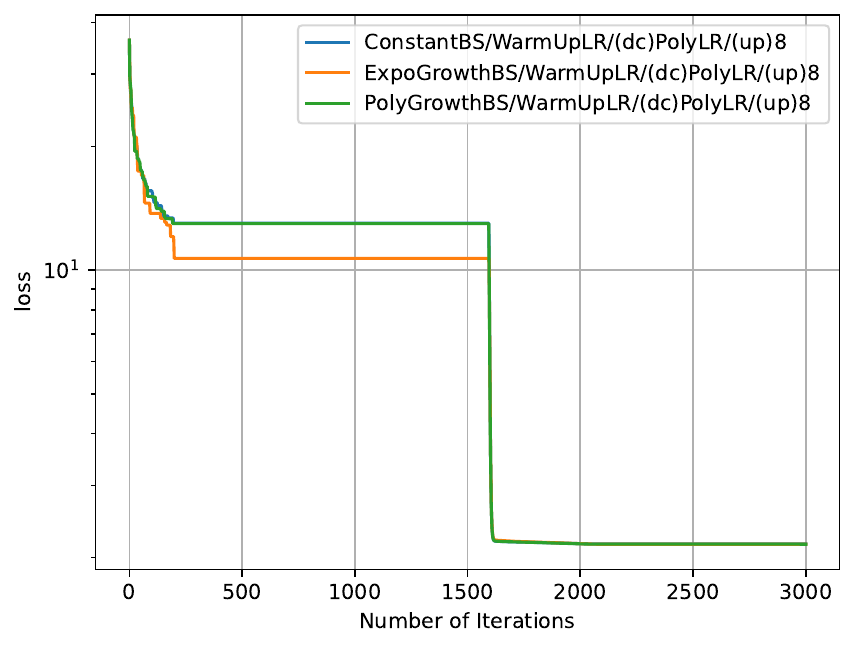}
\caption{Objective function value (loss) versus number of iterations for warm-up LRs that have an increasing part with eight increments on COIL100 dataset (PCA).}
\label{fig:COIL100_loss_wup8}
\end{figure}

\begin{figure}[htbp]
\centering
\includegraphics[width=0.24\linewidth]{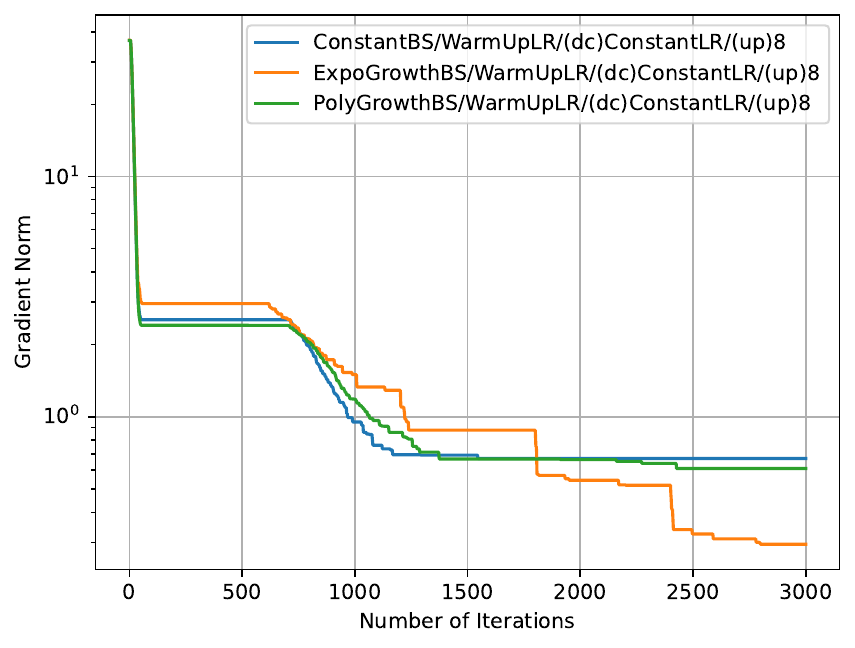}
\includegraphics[width=0.24\linewidth]{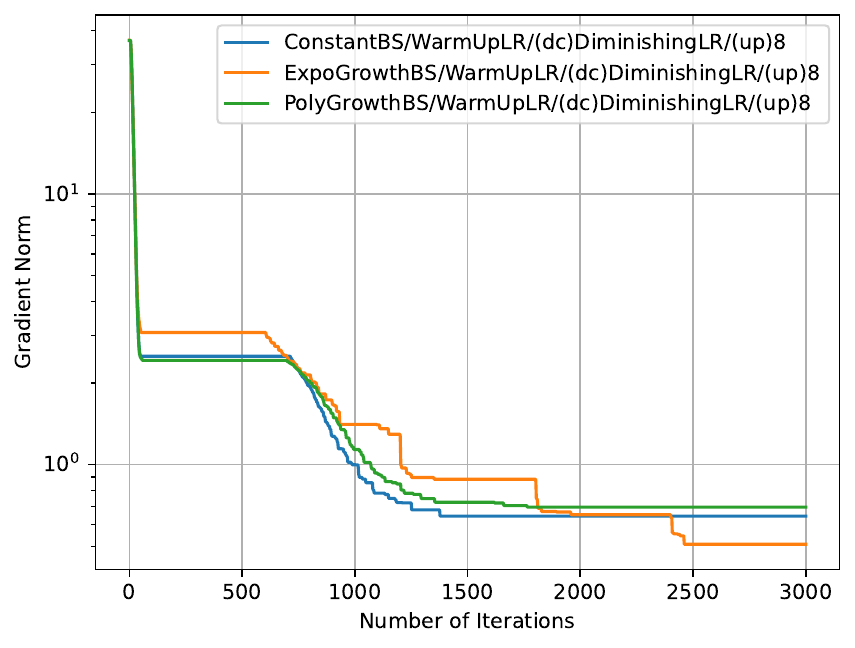}
\includegraphics[width=0.24\linewidth]{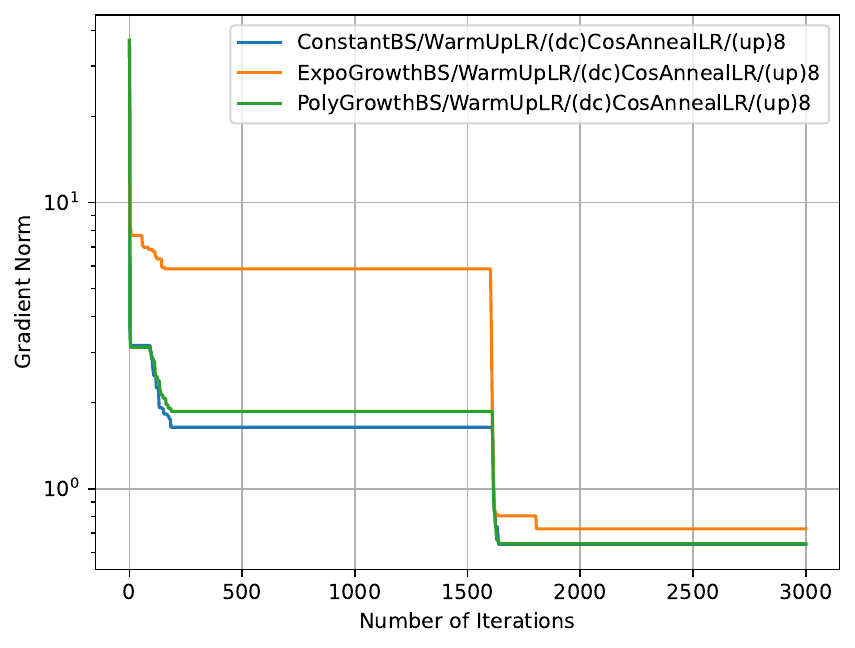}
\includegraphics[width=0.24\linewidth]{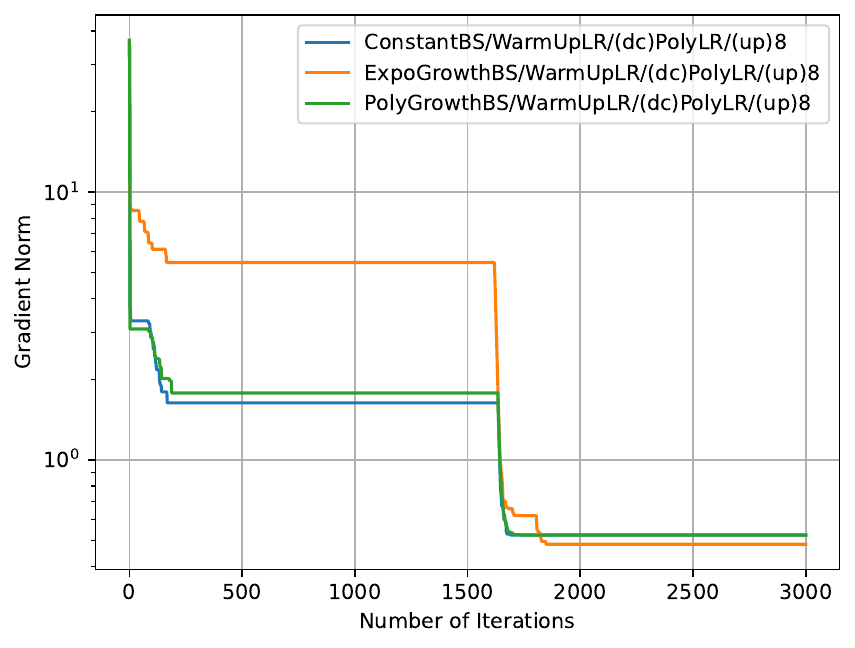}
\caption{Norm of the gradient of the objective function versus number of iterations for warm-up LRs that have an increasing part with eight increments on MNIST dataset (PCA).}
\label{fig:MNIST_grad_wup8}
\end{figure}

\begin{figure}[htbp]
\centering
\includegraphics[width=0.24\linewidth]{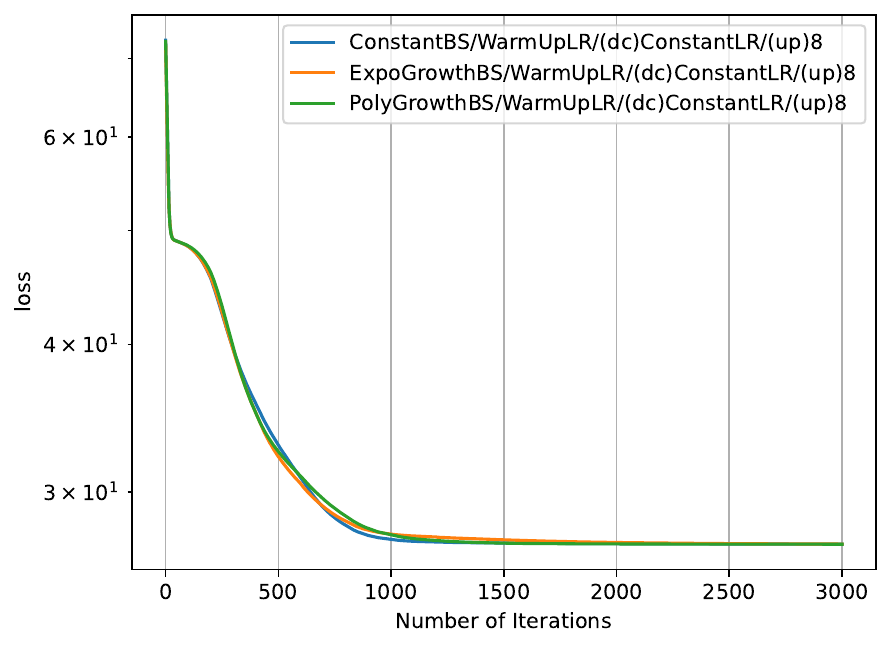}
\includegraphics[width=0.24\linewidth]{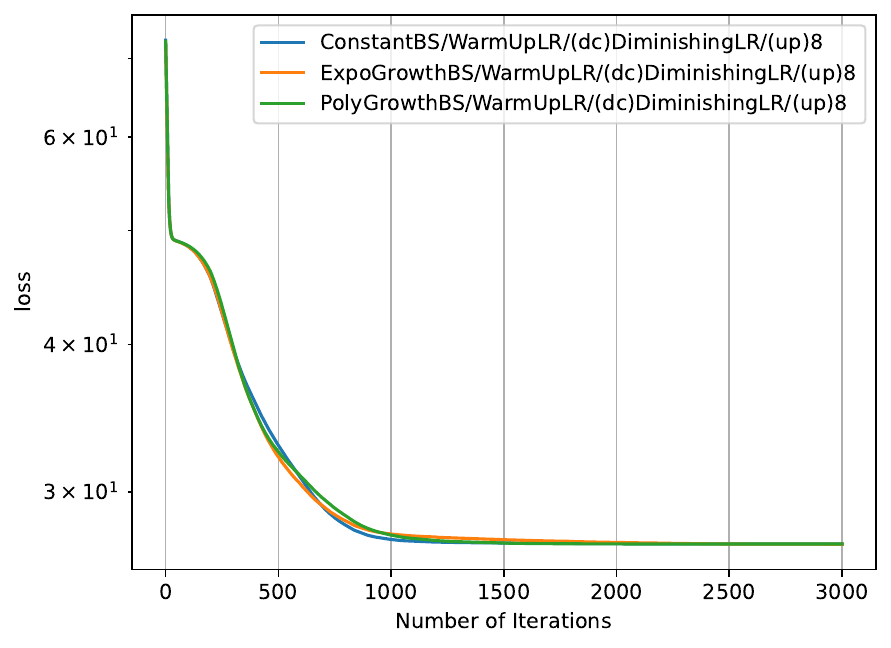}
\includegraphics[width=0.24\linewidth]{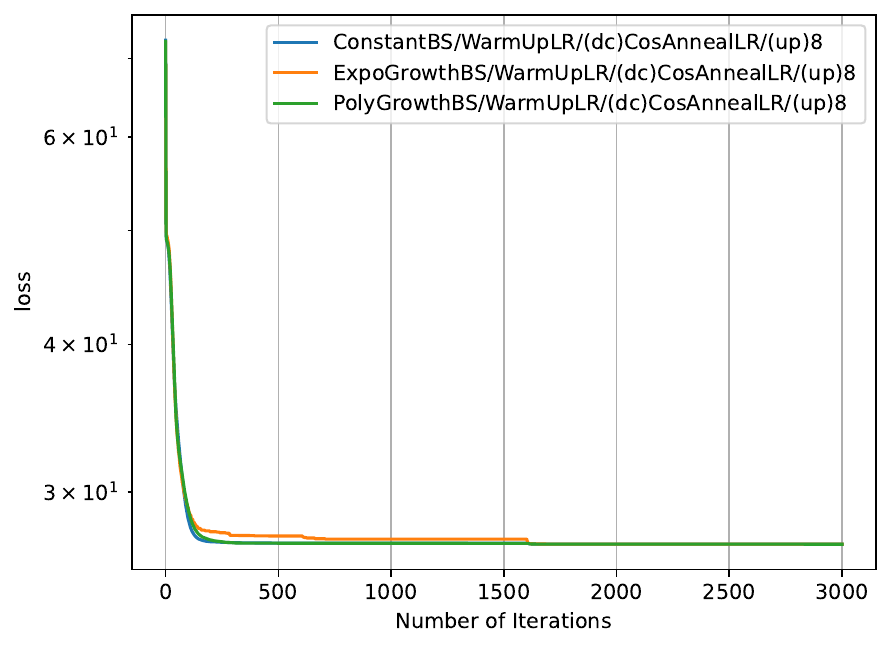}
\includegraphics[width=0.24\linewidth]{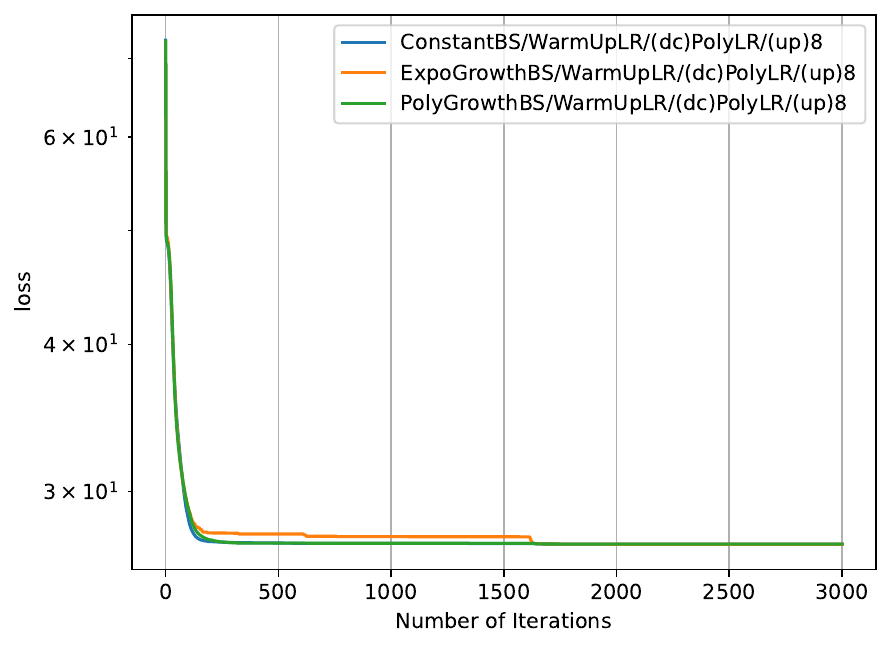}
\caption{Objective function value (loss) versus number of iterations for warm-up LRs that have an increasing part with eight increments on MNIST dataset (PCA).}
\label{fig:MNIST_loss_wup8}
\end{figure}

\subsection{Low-rank Matrix Completion}
\noindent$\textbf{[Case C]} \quad l=l_w=3$\\
\noindent
Figures \ref{fig:ml-1m_grad_wup3} and \ref{fig:jester_grad_wup3} plot performance in terms of the gradient norm of the objective function versus the number of iterations for a warm-up LR with decay parts given by \eqref{eq:const_lr}, \eqref{eq:dim_lr}, \eqref{eq:cosan_lr}, and \eqref{eq:poly_dec_lr} on the MovieLens-1M and Jester datasets, respectively. Figures \ref{fig:ml-1m_grad_wup3} and \ref{fig:jester_loss_wup3} plot performance in terms of the objective function value versus the number of iterations for a warm-up LR with decay parts given by \eqref{eq:const_lr}, \eqref{eq:dim_lr}, \eqref{eq:cosan_lr}, and \eqref{eq:poly_dec_lr} on the MovieLens-1M and Jester datasets, respectively.
\begin{figure}[htbp]
\centering
\includegraphics[width=0.24\linewidth]{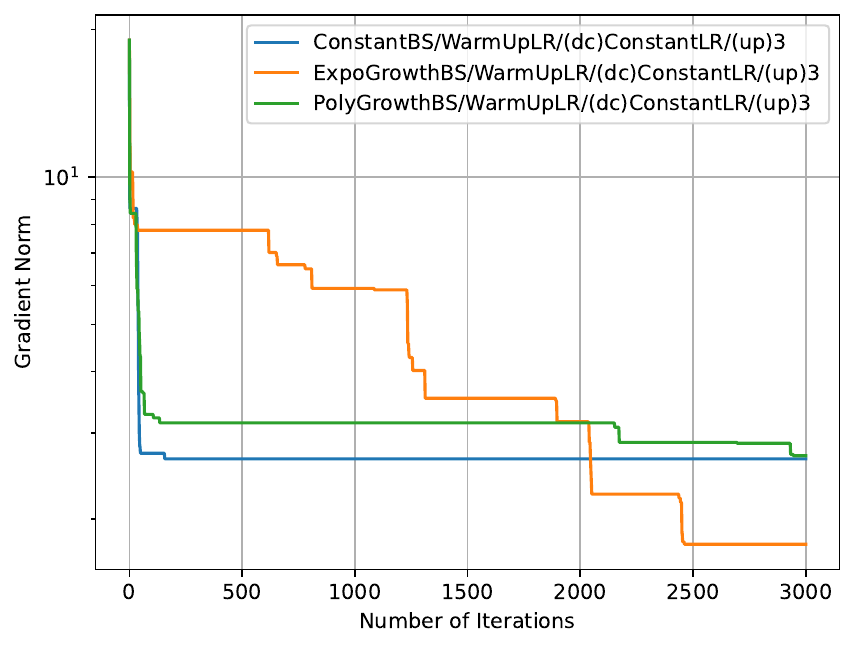}
\includegraphics[width=0.24\linewidth]{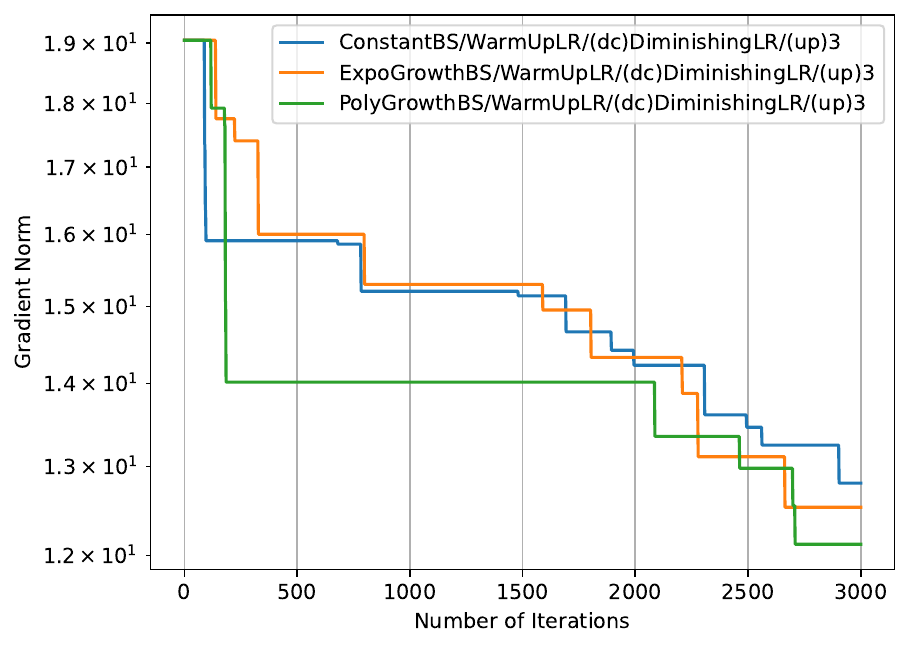}
\includegraphics[width=0.24\linewidth]{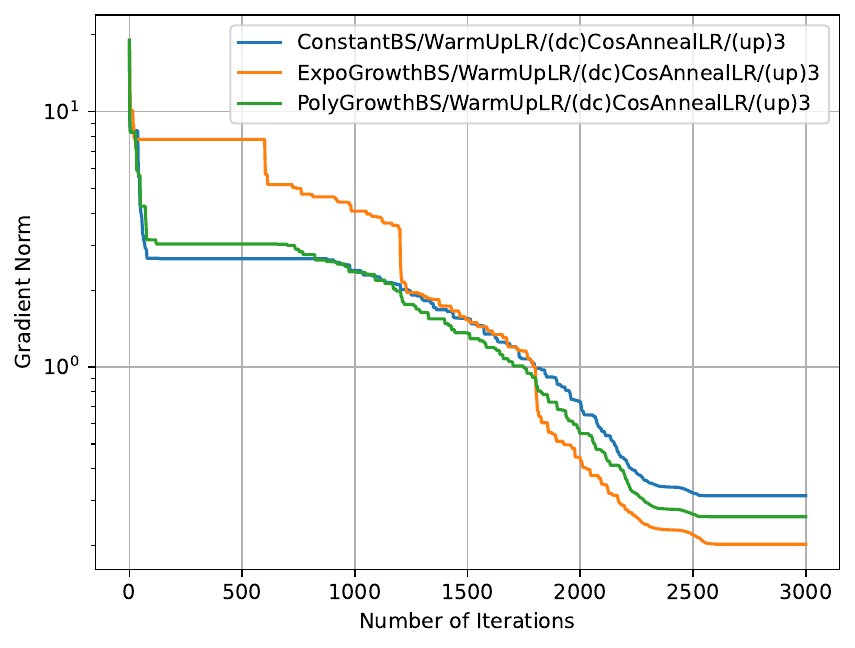}
\includegraphics[width=0.24\linewidth]{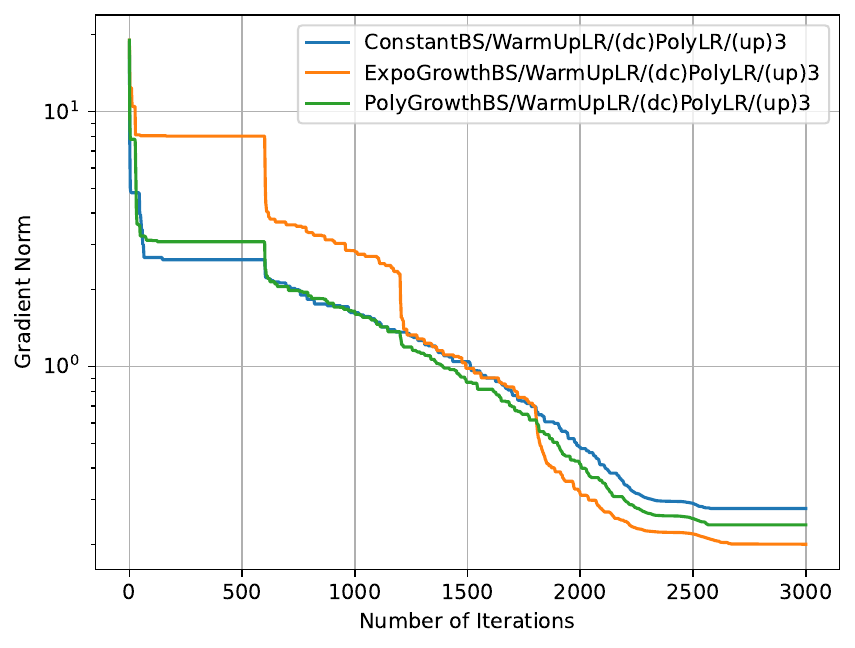}
\caption{Norm of the gradient of the objective function versus number of iterations for warm-up LRs that have an increasing part with three increments on MovieLens-1M dataset (LRMC).}
\label{fig:ml-1m_grad_wup3}
\end{figure}

\begin{figure}[htbp]
\centering
\includegraphics[width=0.24\linewidth]{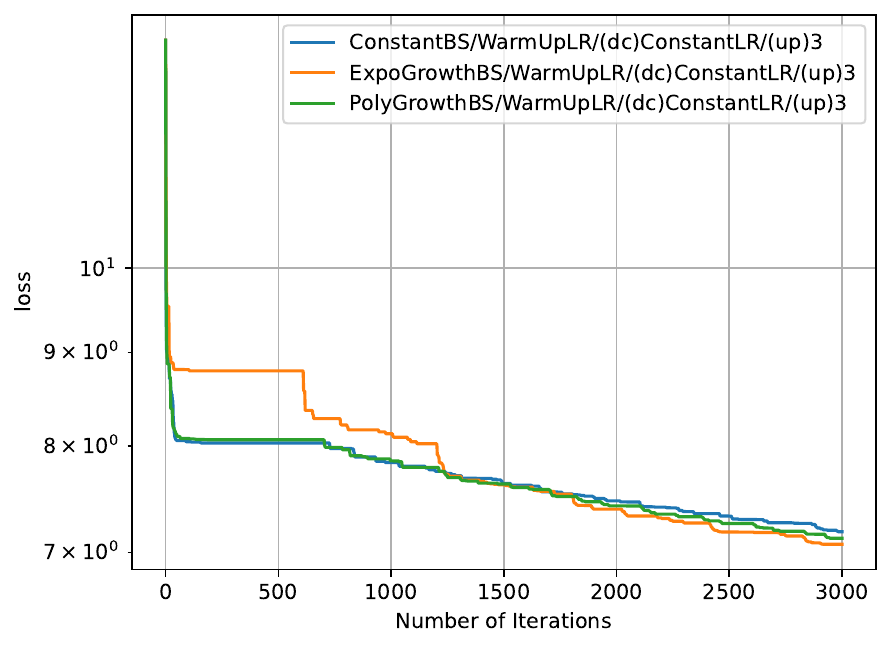}
\includegraphics[width=0.24\linewidth]{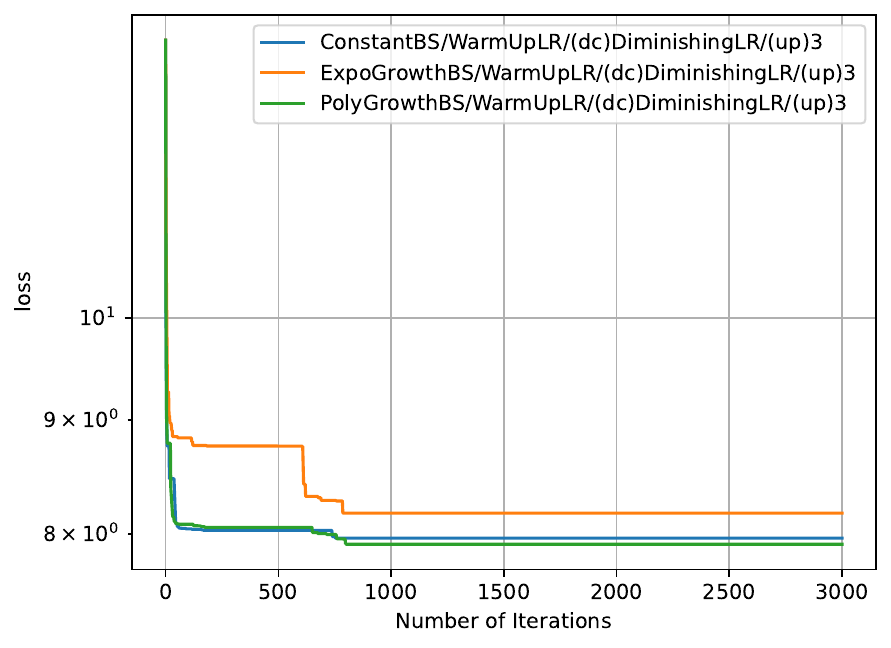}
\includegraphics[width=0.24\linewidth]{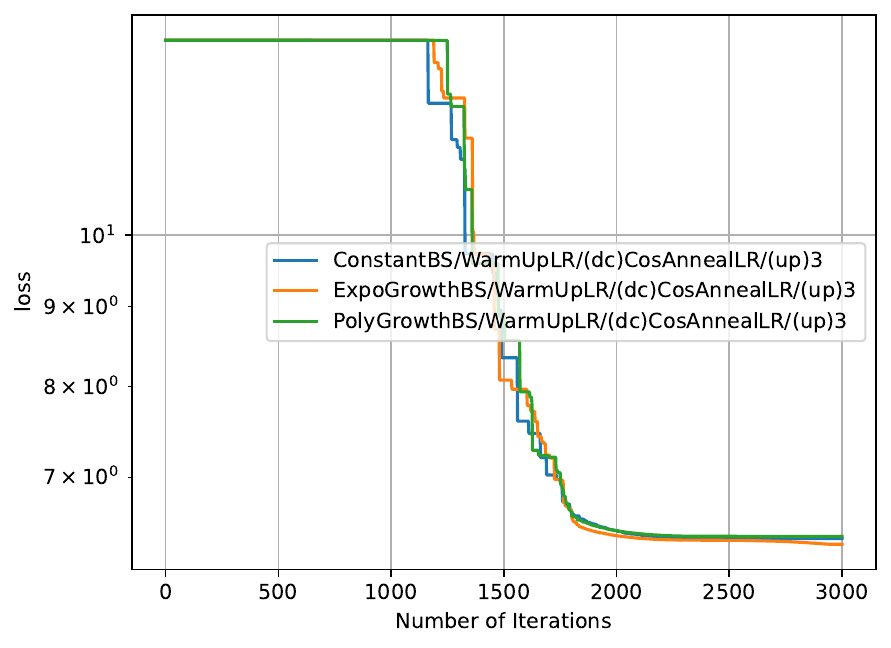}
\includegraphics[width=0.24\linewidth]{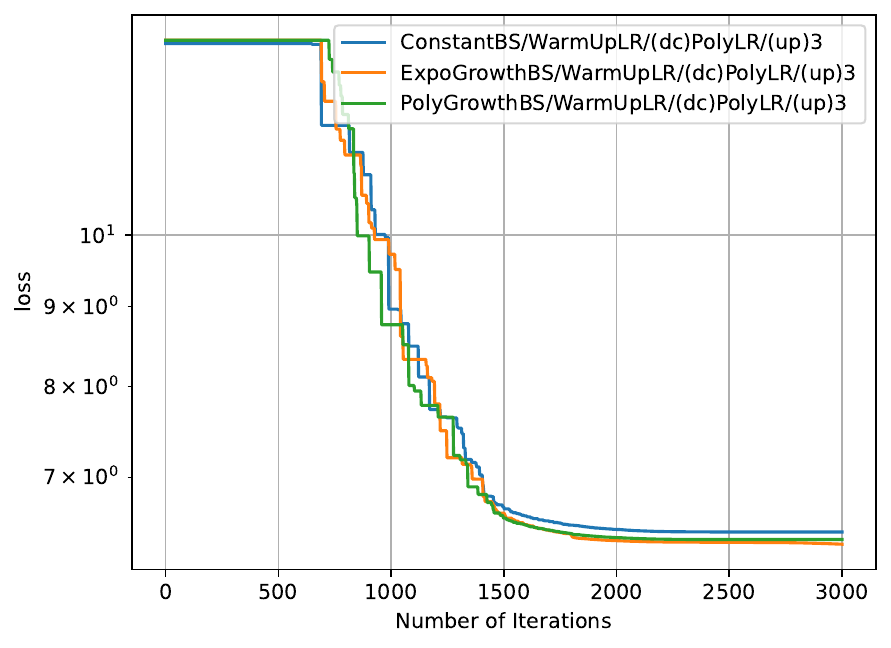}
\caption{Objective function value (loss) versus number of iterations for warm-up LRs that have an increasing part with three increments on MovieLens-1M dataset (LRMC).}
\label{fig:ml-1m_loss_wup3}
\end{figure}

\begin{figure}[htbp]
\centering
\includegraphics[width=0.24\linewidth]{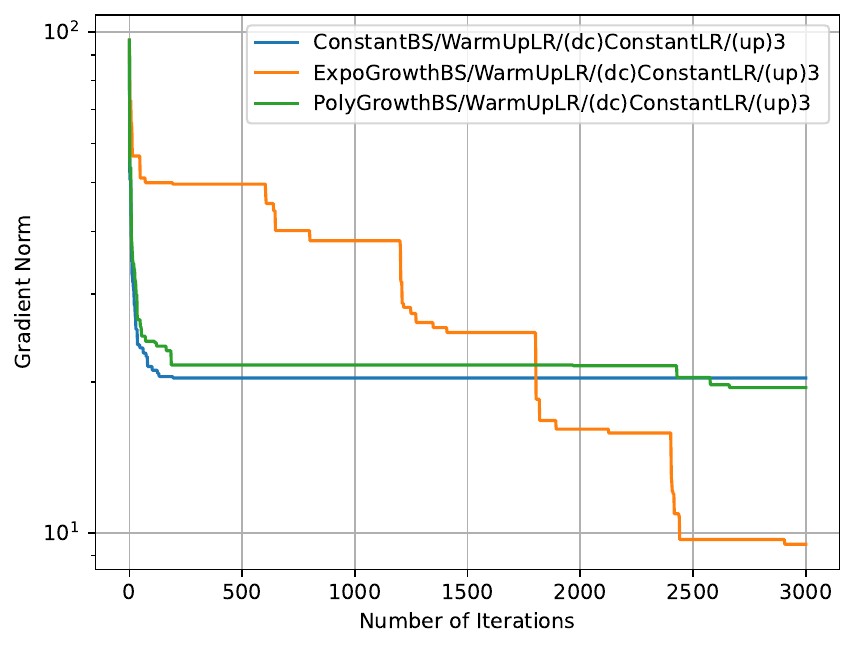}
\includegraphics[width=0.24\linewidth]{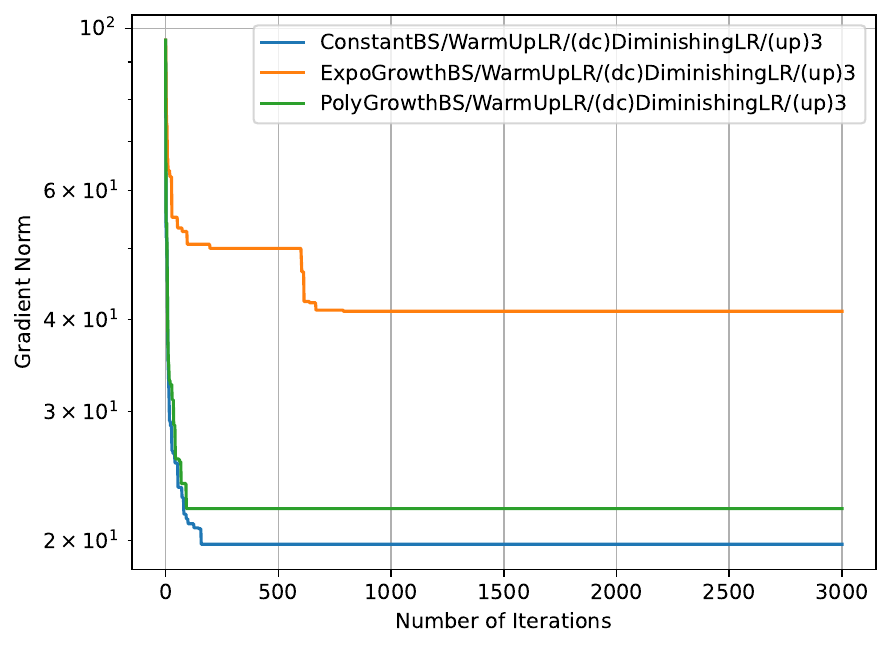}
\includegraphics[width=0.24\linewidth]{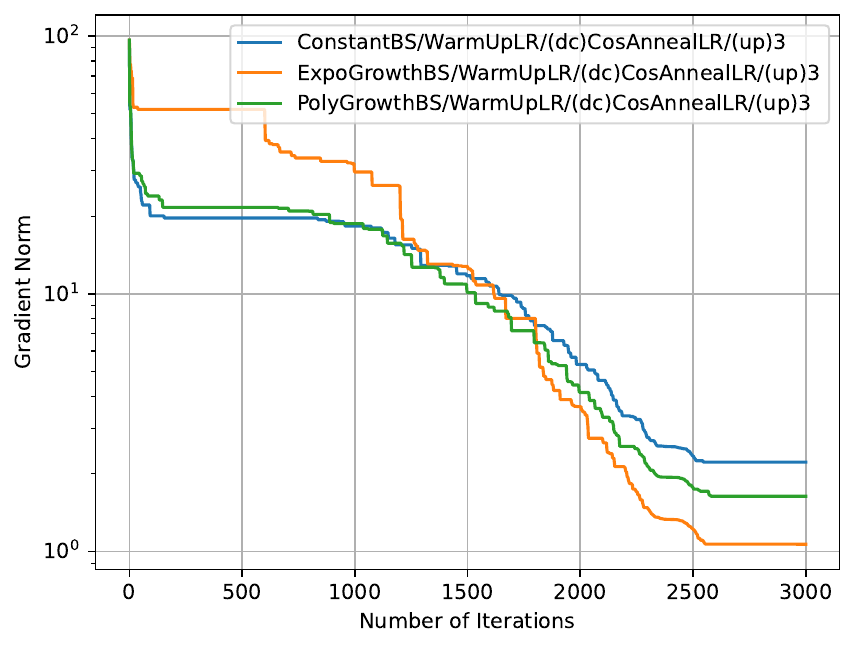}
\includegraphics[width=0.24\linewidth]{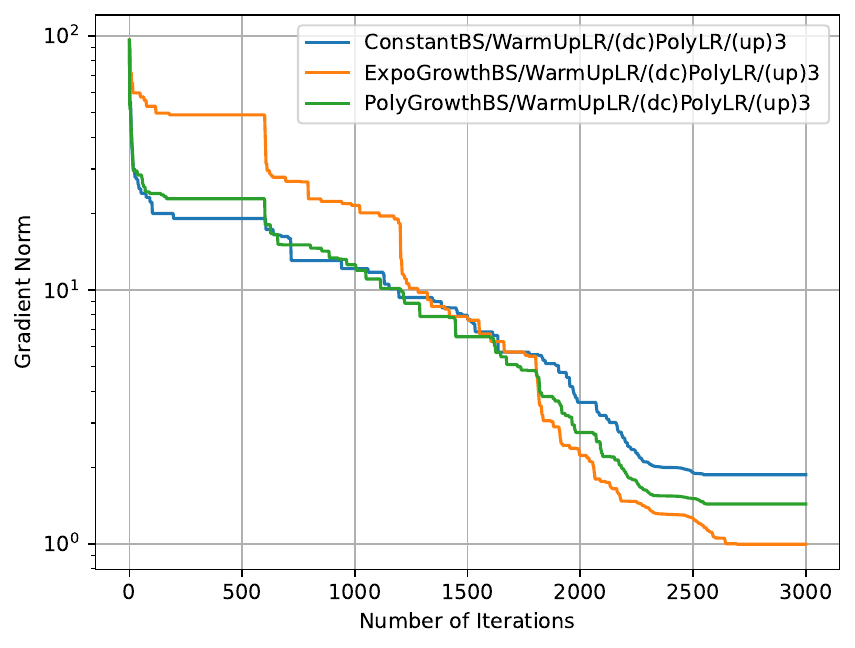}
\caption{Norm of the gradient of the objective function versus number of iterations for warm-up LRs that have an increasing part with three increments on Jester dataset (LRMC).}
\label{fig:jester_grad_wup3}
\end{figure}

\begin{figure}[htbp]
\centering
\includegraphics[width=0.24\linewidth]{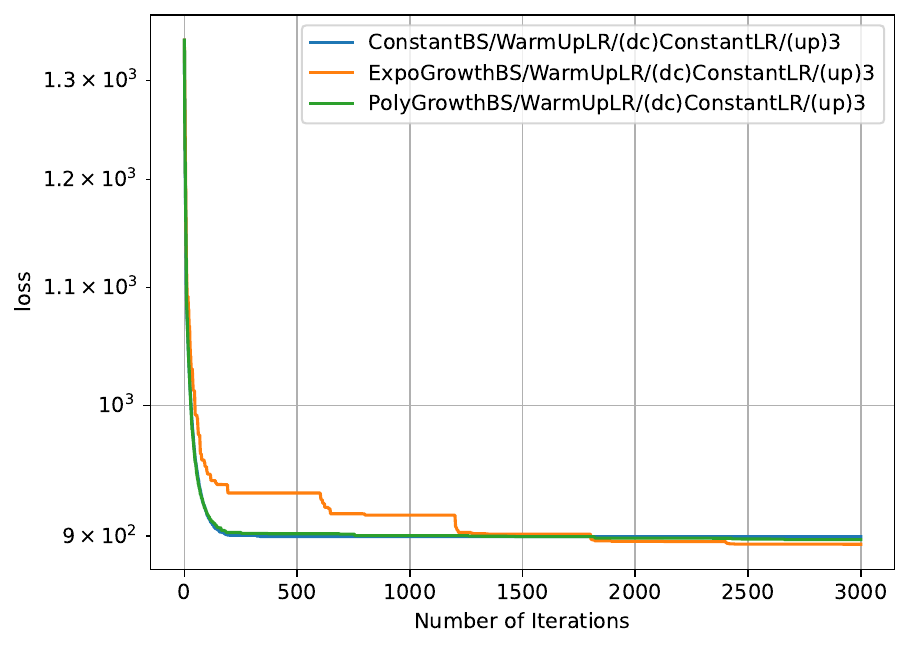}
\includegraphics[width=0.24\linewidth]{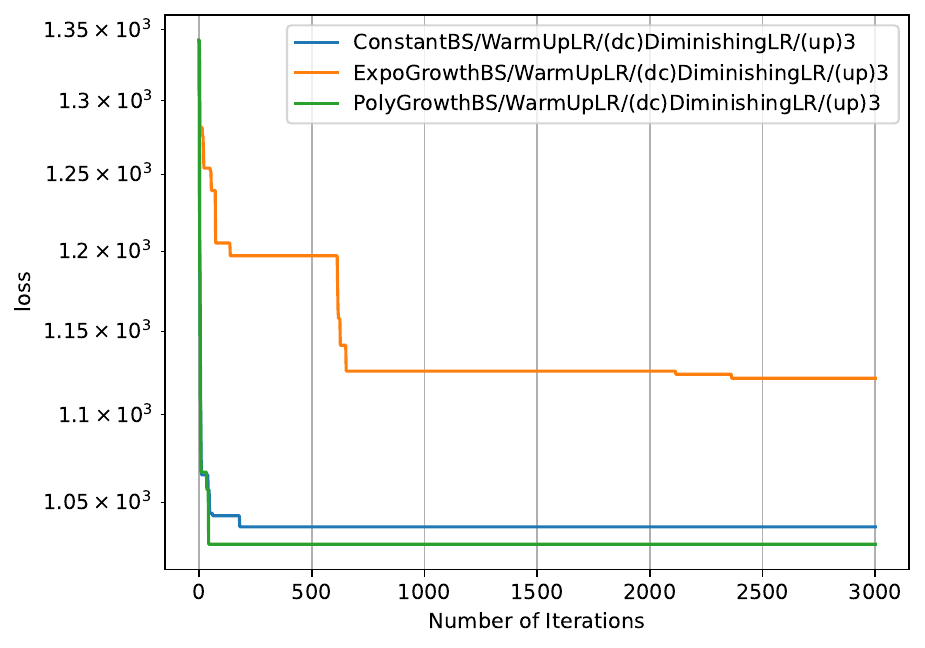}
\includegraphics[width=0.24\linewidth]{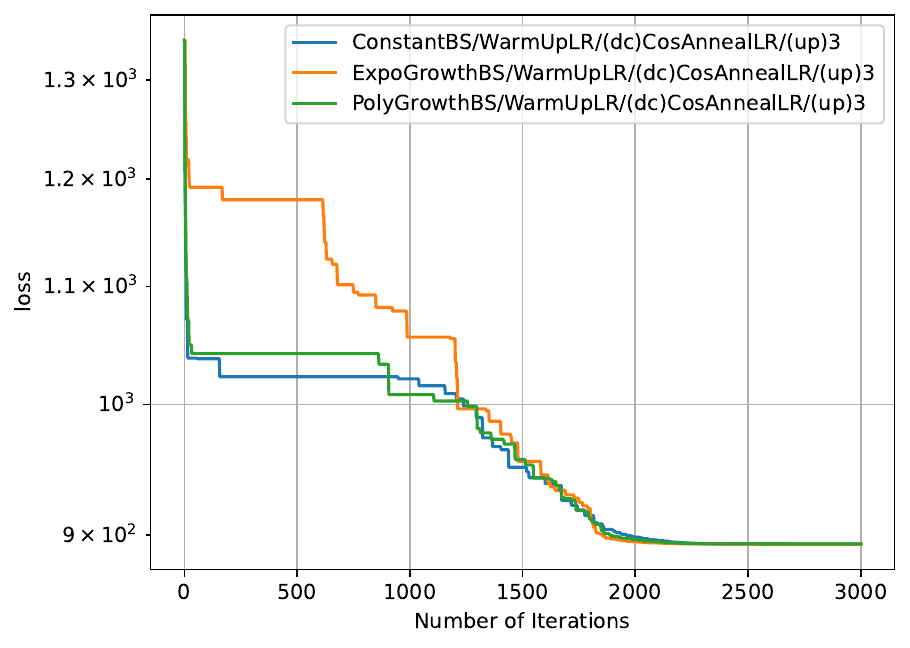}
\includegraphics[width=0.24\linewidth]{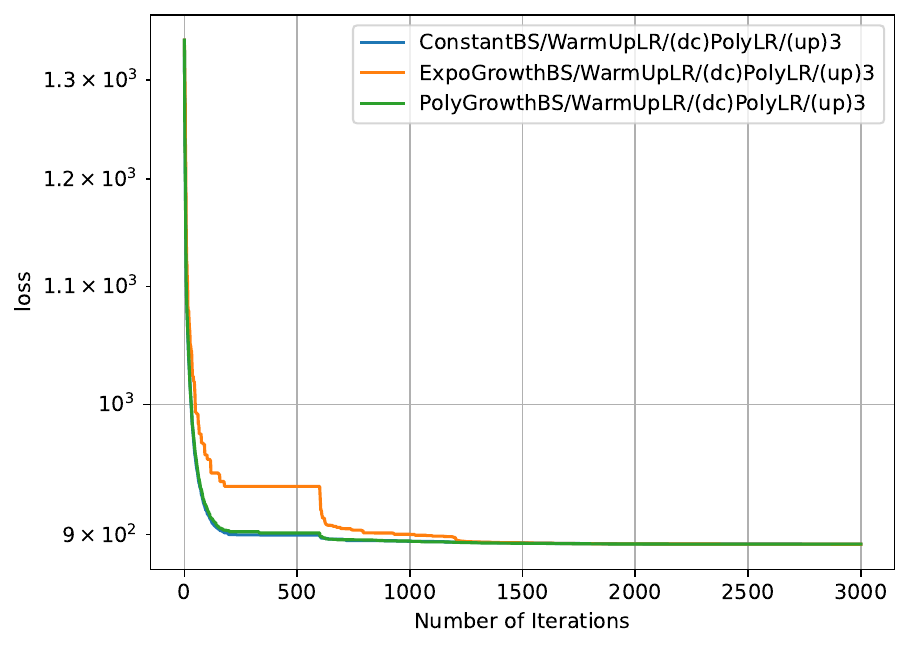}
\caption{Objective function value (loss) versus number of iterations for warm-up LRs that have an increasing part with three increments on Jester dataset (LRMC).}
\label{fig:jester_loss_wup3}
\end{figure}

\noindent$\textbf{[Case D]} \quad l=3, l_w=8$\\
\noindent
Figures \ref{fig:ml-1m_grad_wup8} and \ref{fig:jester_grad_wup8} plot performance in terms of the gradient norm of the objective function versus the number of iterations for a warm-up LR with decay parts given by \eqref{eq:const_lr}, \eqref{eq:dim_lr}, \eqref{eq:cosan_lr}, and \eqref{eq:poly_dec_lr} on the MovieLens-1M and Jester datasets, respectively. Figures \ref{fig:ml-1m_grad_wup8} and \ref{fig:jester_loss_wup8} plot performance in terms of the objective function value versus the number of iterations for a warm-up LR with decay parts given by \eqref{eq:const_lr}, \eqref{eq:dim_lr}, \eqref{eq:cosan_lr}, and \eqref{eq:poly_dec_lr} on the MovieLens-1M and Jester datasets, respectively.
\begin{figure}[htbp]
\centering
\includegraphics[width=0.24\linewidth]{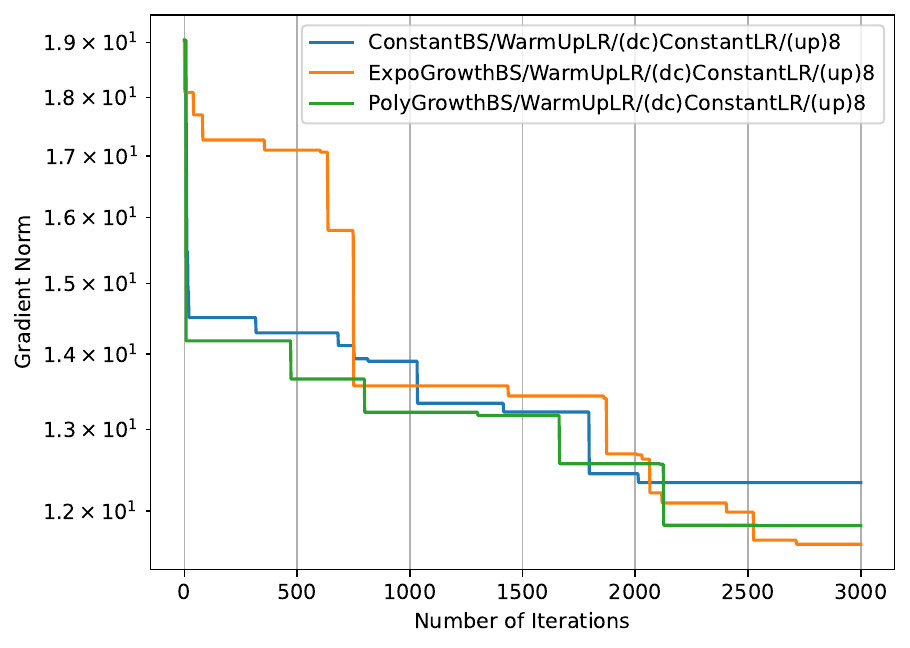}
\includegraphics[width=0.24\linewidth]{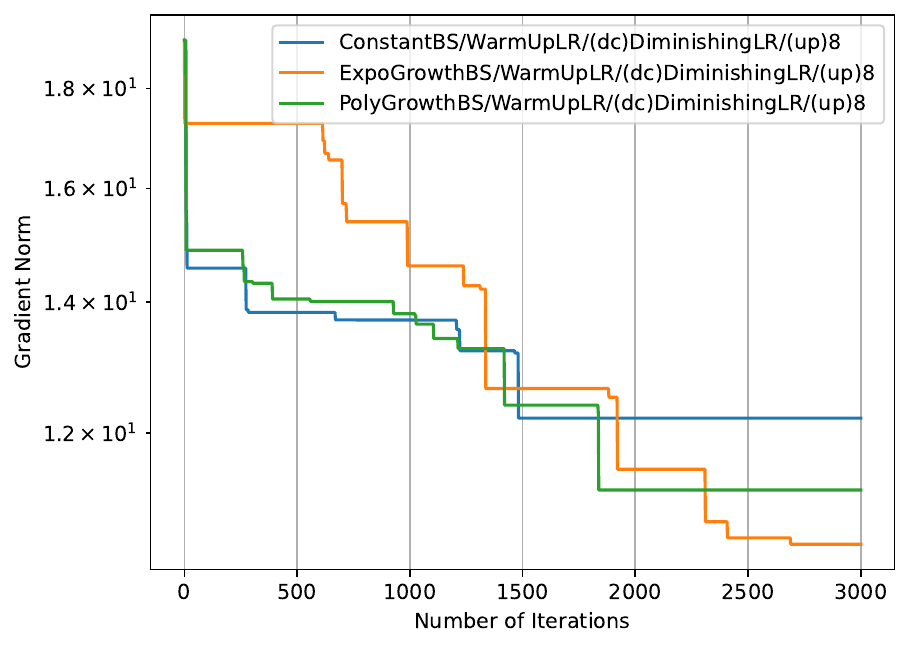}
\includegraphics[width=0.24\linewidth]{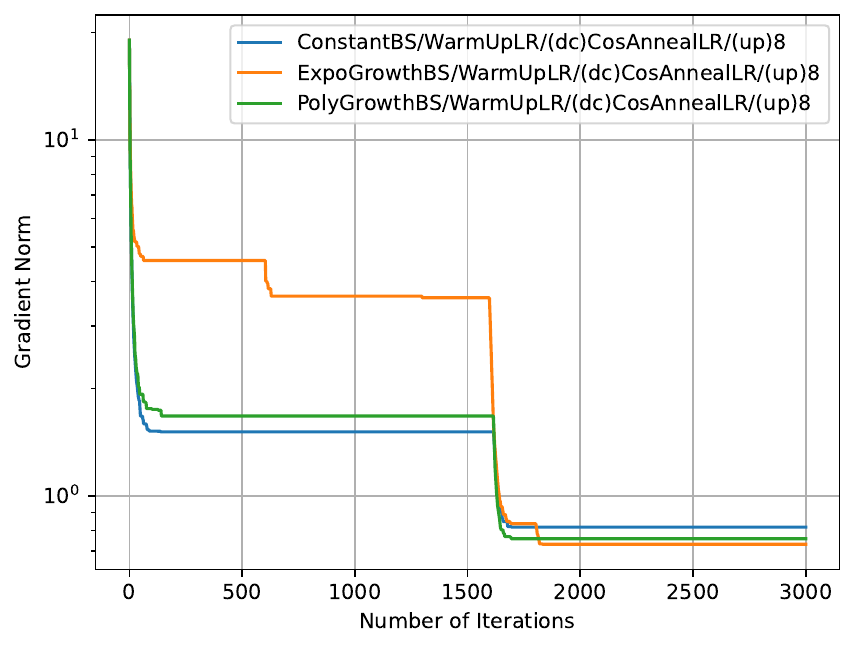}
\includegraphics[width=0.24\linewidth]{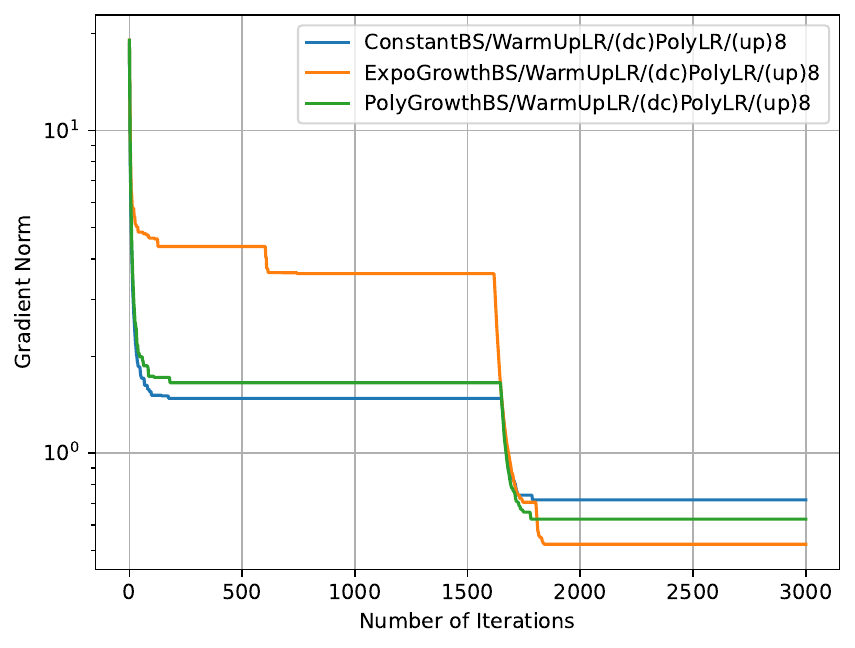}
\caption{Norm of the gradient of the objective function versus number of iterations for warm-up LRs that have an increasing part with eight increments on MovieLens-1M dataset (LRMC).}
\label{fig:ml-1m_grad_wup8}
\end{figure}

\begin{figure}[htbp]
\centering
\includegraphics[width=0.24\linewidth]{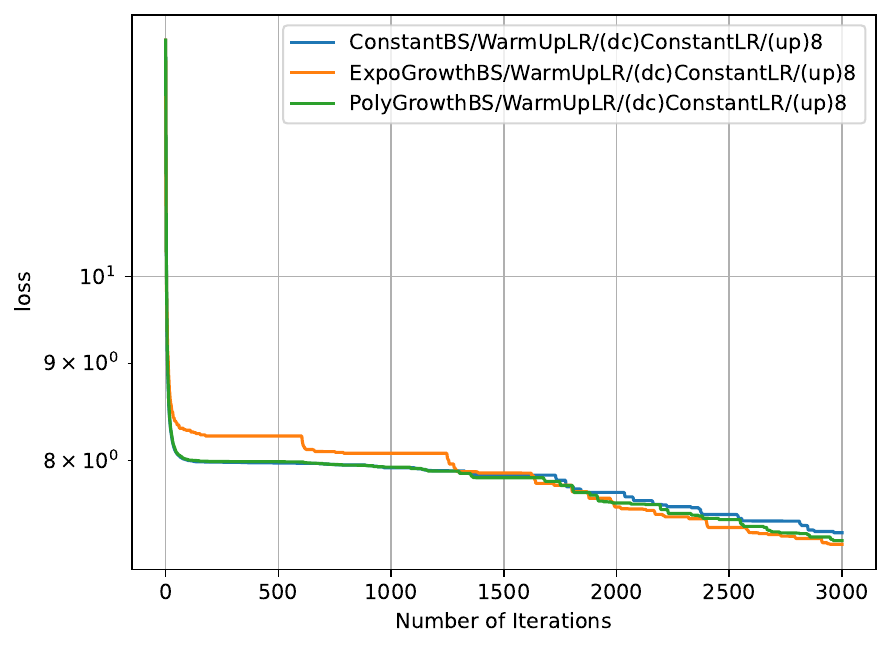}
\includegraphics[width=0.24\linewidth]{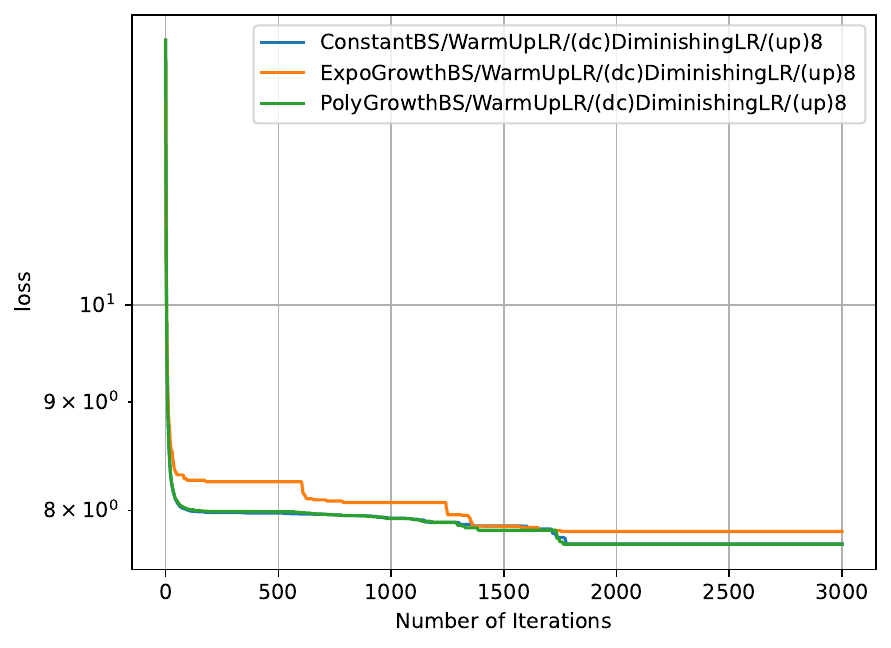}
\includegraphics[width=0.24\linewidth]{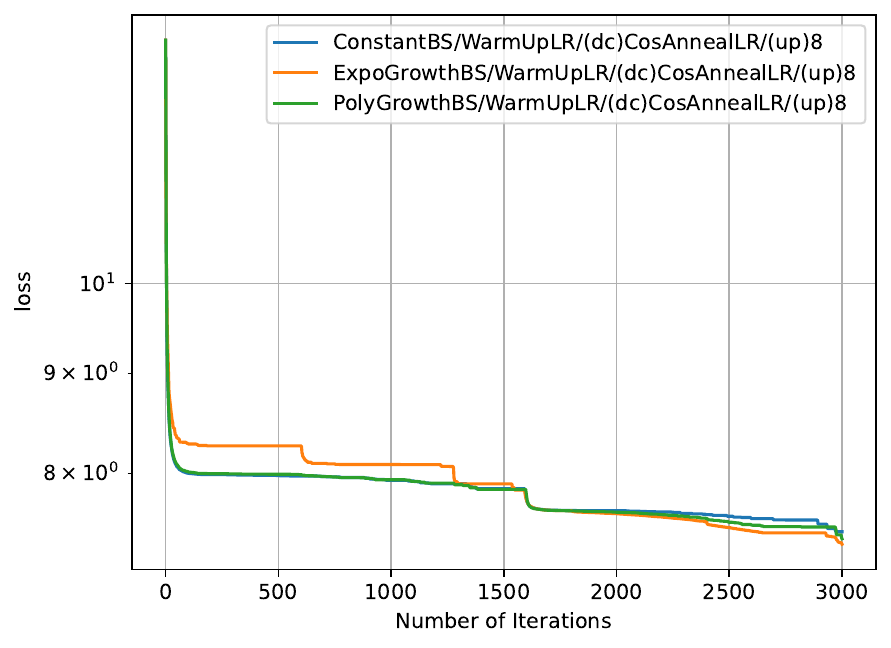}
\includegraphics[width=0.24\linewidth]{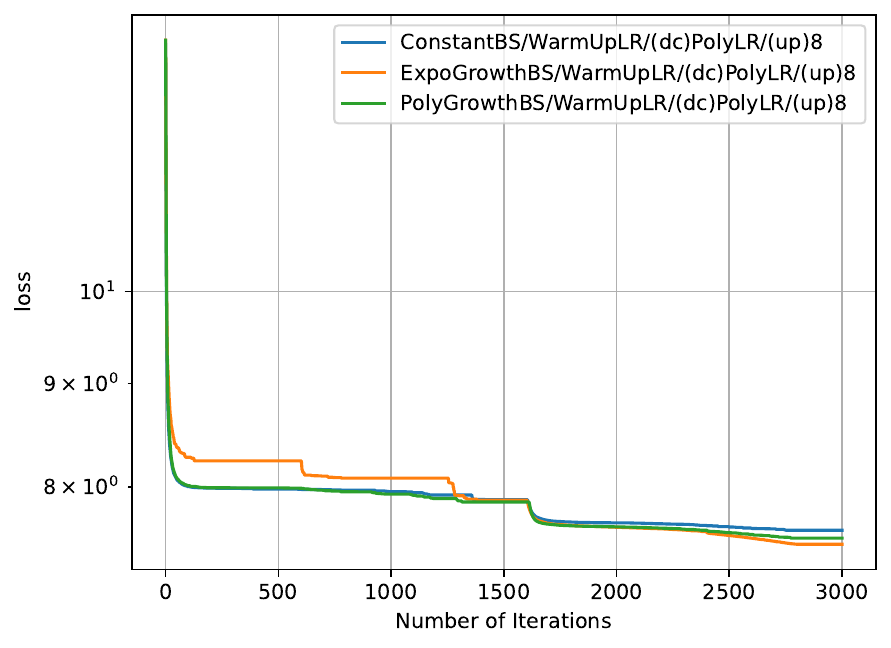}
\caption{Objective function value (loss) versus number of iterations for warm-up LRs that have an increasing part with eight increments on MovieLens-1M dataset (LRMC).}
\label{fig:ml-1m_loss_wup8}
\end{figure}

\begin{figure}[htbp]
\centering
\includegraphics[width=0.24\linewidth]{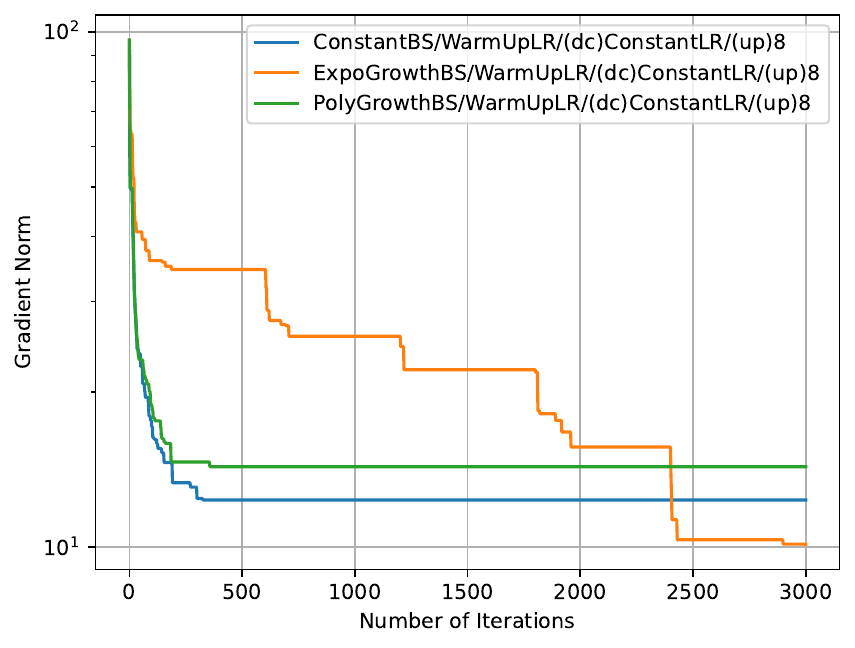}
\includegraphics[width=0.24\linewidth]{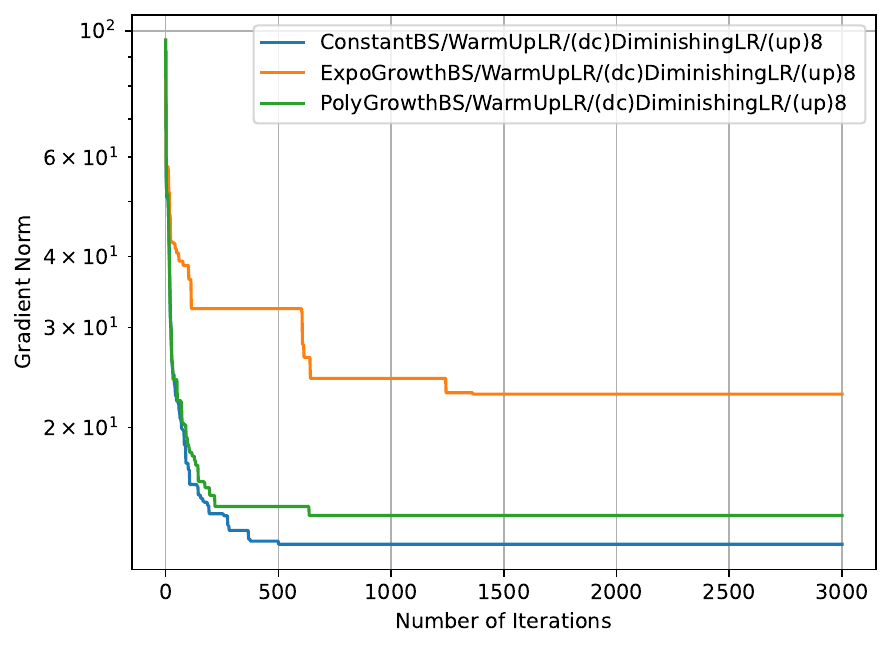}
\includegraphics[width=0.24\linewidth]{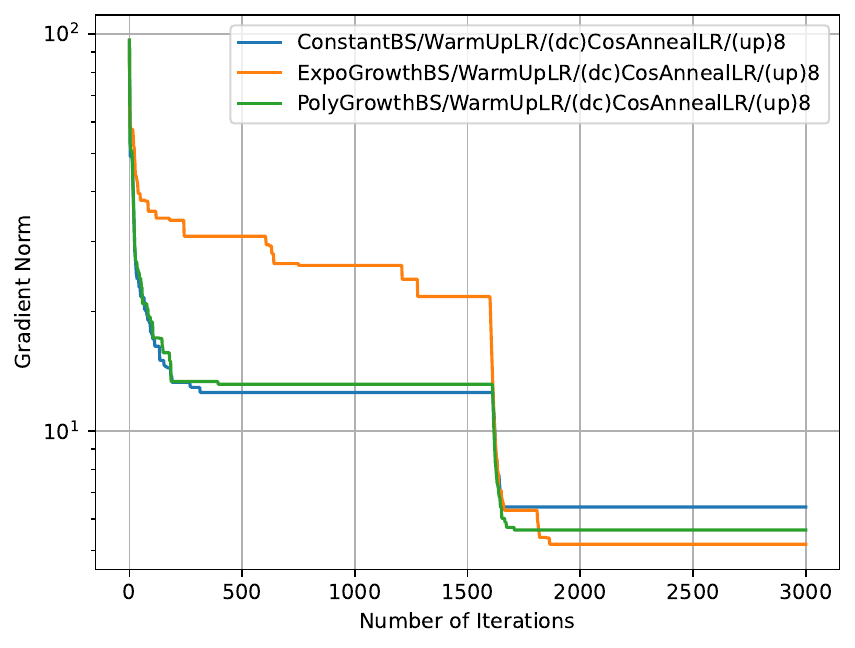}
\includegraphics[width=0.24\linewidth]{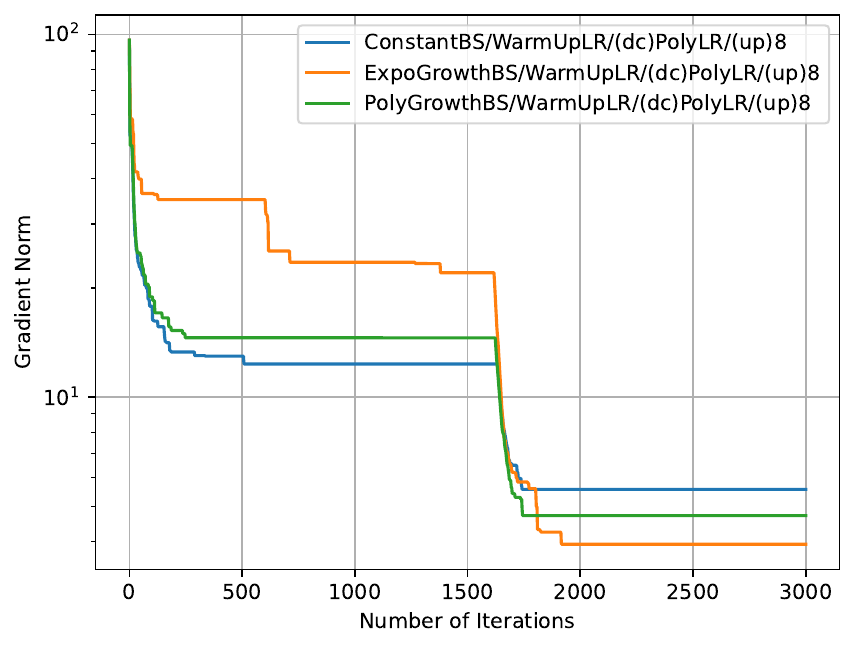}
\caption{Norm of the gradient of the objective function versus number of iterations for warm-up LRs that have an increasing part with eight increments on Jester dataset (LRMC).}
\label{fig:jester_grad_wup8}
\end{figure}

\begin{figure}[htbp]
\centering
\includegraphics[width=0.24\linewidth]{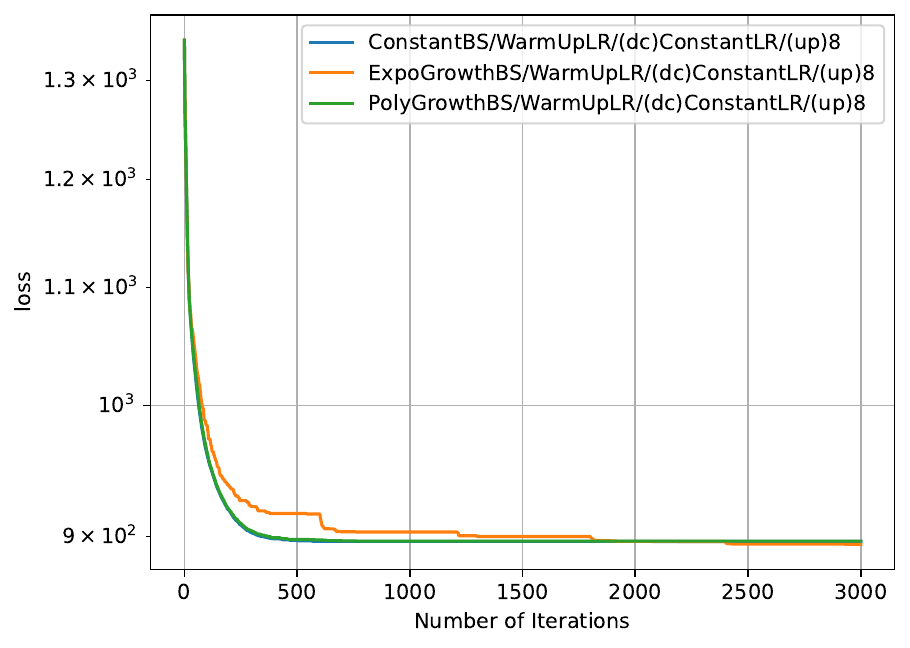}
\includegraphics[width=0.24\linewidth]{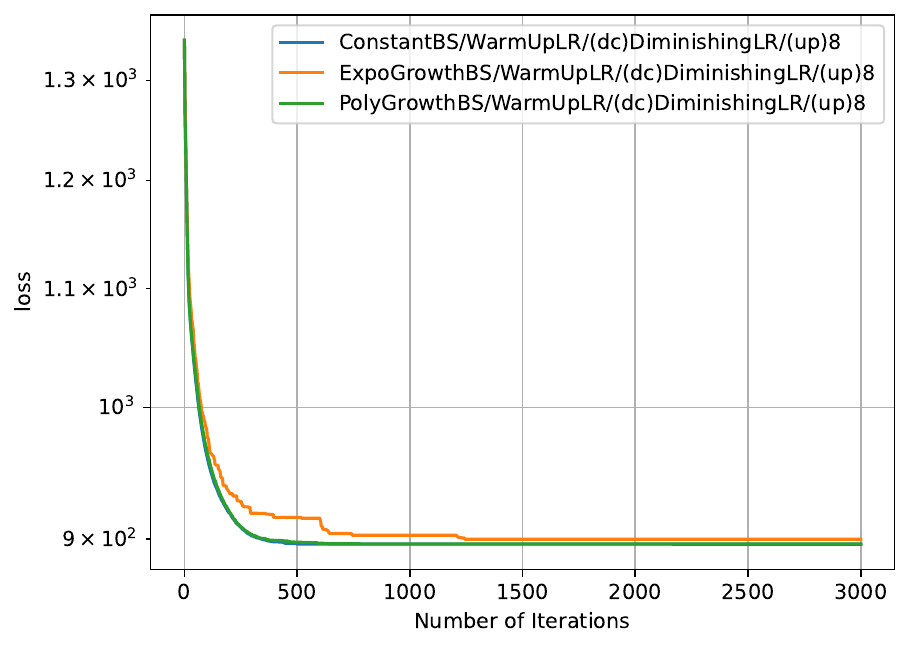}
\includegraphics[width=0.24\linewidth]{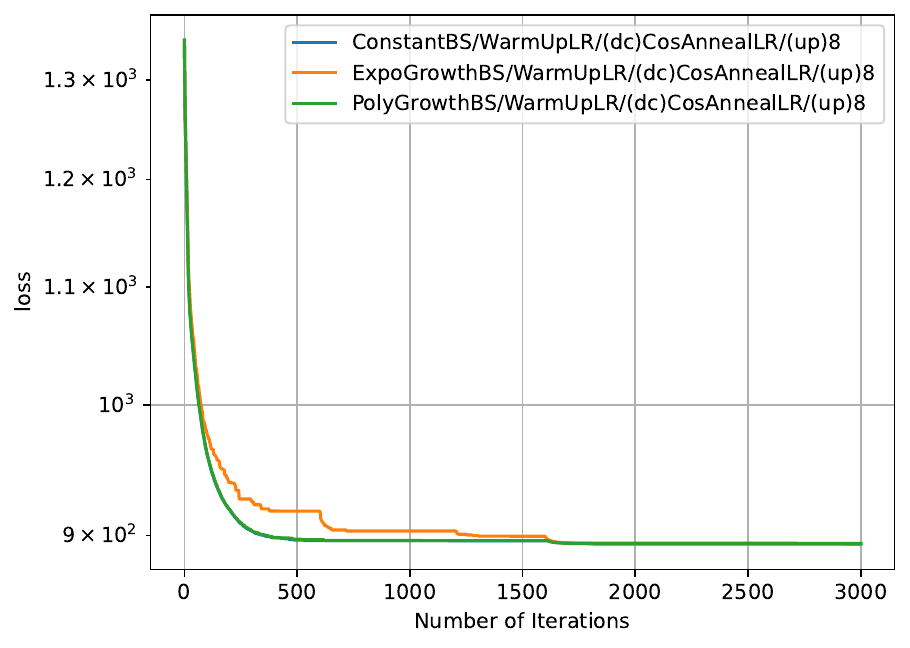}
\includegraphics[width=0.24\linewidth]{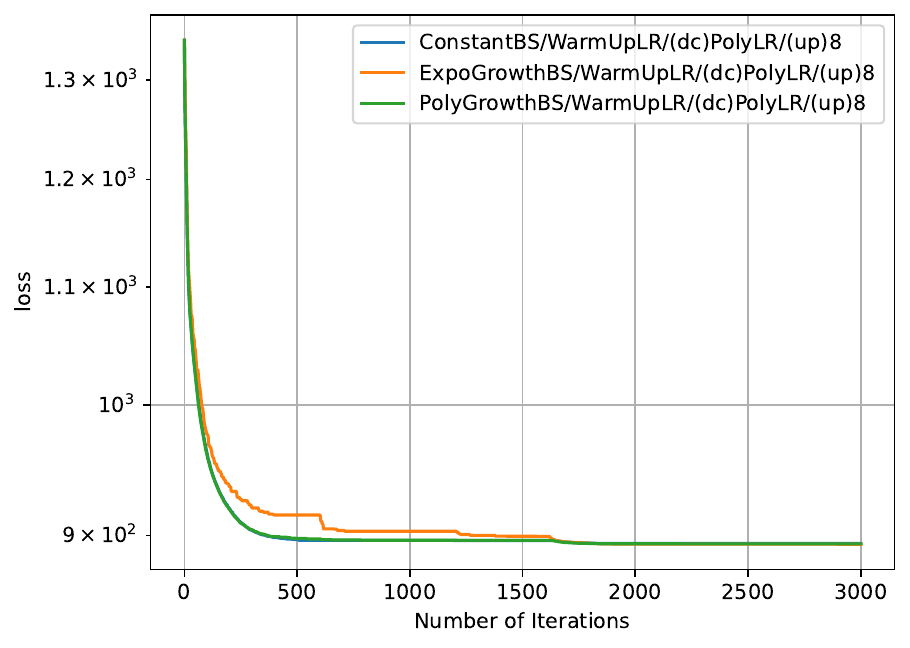}
\caption{Objective function value (loss) versus number of iterations for warm-up LRs that have an increasing part with eight increments on Jester dataset (LRMC).}
\label{fig:jester_loss_wup8}
\end{figure}

\end{document}